\documentclass[12pt, reqno]{amsart}

\author[Michele Caprio]{Michele Caprio}
\address{The University of Manchester, Oxford Road, Manchester, UK M13 9PL}
\email{michele.caprio@manchester.ac.uk}

\title{A Category-Theoretic Analysis of  Conformal Prediction} 

\usepackage[T1]{fontenc}
\usepackage{amsmath}
\usepackage{amssymb}
\usepackage{lscape}
\usepackage{centernot}
\usepackage{amsthm}
\usepackage[left=2.5cm, right=2.5cm, top=3cm]{geometry}
\usepackage[dvipsnames]{xcolor}  
\usepackage{hyperref}

\AtBeginDocument{%
  \hypersetup{
    citecolor  = OliveGreen,
    linkcolor  = blue,
    urlcolor   = cyan
  }%
}
\usepackage{algorithm}
\usepackage{bbm}
\usepackage{booktabs}
\usepackage{algorithm}
\usepackage{algpseudocode}
\usepackage{tikz}
\usepackage{tikz-cd}
\usepackage{caption}
\usepackage{subcaption}
\usepackage{fancyhdr}
\usepackage[inline]{enumitem}
\usepackage{comment}
\usepackage{nicefrac}
\usepackage{bm}
\usepackage{mathrsfs}
\usepackage{multirow}
\usepackage{graphicx}
\usepackage[utf8]{inputenc}
\usepackage{cancel}
\usepackage{mathtools}
\usepackage[numbers]{natbib}

\newcommand{\vertiii}[1]{{\left\vert\kern-0.25ex\left\vert\kern-0.25ex\left\vert #1 
    \right\vert\kern-0.25ex\right\vert\kern-0.25ex\right\vert}}

\AtBeginDocument{%
   \def\MR#1{}
}

\makeatletter
\def\algbackskip{\hskip-\ALG@thistlm}
\makeatother

\makeatletter
\def\namedlabel#1#2{\begingroup
    #2%
    \def\@currentlabel{#2}%
    \phantomsection\label{#1}\endgroup
}
\makeatother

\newtheorem{thm}{Theorem}[section]

\theoremstyle{definition} 
\newtheorem{defi}[thm]{Definition}
\let\olddefi\defi
\renewcommand{\defi}{\olddefi\normalfont}
\newtheorem{rmk}{Remark}
\let\oldrmk\rmk
\renewcommand{\rmk}{\oldrmk\normalfont}

\newenvironment{hproof}{%
  \proof}{\endproof}

\newtheorem{theorem}{Theorem}
\newtheorem{lemma}[theorem]{Lemma}
\newtheorem{proposition}[theorem]{Proposition}
\newtheorem{corollary}{Corollary}[theorem]

\newtheorem{definition}[theorem]{Definition}
\newtheorem{remark}{Remark}
\newtheorem{example}{Example}

\hypersetup{
    pdftitle={prove},
    pdfauthor={autori},
    pdfmenubar=false,
    pdffitwindow=true,
    pdfstartview=FitH,
    colorlinks=true,
    linkcolor=blue,
    citecolor=green,
    urlcolor=cyan
}

\providecommand{\MR}[1]{}

\providecommand{\MR}{\relax\ifhmode\unskip\space\fi MR }

\newcommand{\kappacorr}{\kappa}
\newcommand{\CRED}{\mathsf{CRED}}
\newcommand{\IHDR}{\mathsf{IHDR}}
\newcommand{\Y}{Y}
\newcommand{\F}{\mathscr{F}}
\newcommand{\Ccal}{\mathscr{C}}

\newcommand{\Prob}{\Delta_{\Y}}
\newcommand{\R}{\mathbb{R}}

\makeatletter
\newcommand{\skipitems}[1]{%
  \addtocounter{\@enumctr}{#1}%
}
\makeatother

\begin{document}

 \begin{abstract}
Conformal prediction (CP) produces prediction regions with finite-sample, distribution free coverage guarantees, but its interpretation as a quantitative uncertainty tool is often left implicit. We develop a category-theoretic approach that makes this structure explicit. We show that Full Conformal Prediction can be represented as a morphism in two categories capturing (i) stability of set-valued procedures and (ii) measurability of random regions. Under mild conditions, we prove a commuting diagram result that decomposes the construction of a conformal region into two steps: Extracting a set of predictive distributions from the data, and then deriving a prediction region from this set. This decomposition provides a principled route to numerical uncertainty summaries beyond region size. We further prove an asymptotic compatibility result showing that, for Bayesian predictive scores in regular regimes, conformal regions converge to Bayesian predictive density level sets; We also provide quantitative rates under local empirical process and boundary regularity assumptions. This highlights a bridge between Bayesian, frequentist, and imprecise probabilistic prediction. We additionally identify conditions under which upper posterior constructions are related to e-posteriors, clarifying when e-value-based and conformal-imprecise representations can coincide. Finally, we show that the region extractor is functorial; This yields a modular privacy-compatible perspective in which privacy-preserving outer approximations of shared summary objects lead to
conservative global prediction regions.
\end{abstract}
\maketitle
\thispagestyle{empty}

\section{Introduction}\label{sec:intro}
Category Theory (CT) can be thought of as a general theory of mathematical structures and their relations. Recently, there has been growing interest in studying Statistics and Machine Learning (ML) from a CT perspective, both to leverage its unifying power across apparently unrelated subfields and to obtain deeper theoretical insights.
For example, \citet{stehlik2012category} and \citet{song2024describe} survey category-theoretic formulations of statistical learning models, while \citet{shiebler2021categorytheorymachinelearning} do the same for gradient-based and probabilistic methods, and for invariant and equivariant learning. More generally, CT approaches to probability and statistics are active fields of study \citep{cornish2025stochasticneuralnetworksymmetrisation,cornish2026categoricalaccountmetropolishastingsalgorithm,fritz2025empiricalmeasuresstronglaws,bohinen2025categoricalalgebraconditionalprobability,ensarguet2024categoricalprobabilityspacesergodic,Ackerman_2024,Matache_2022}, and are widely discussed in the \href{https://ncatlab.org/nlab/show/category-theoretic+approaches+to+probability+theory}{nLab} and in the \href{https://golem.ph.utexas.edu/category/2020/06/statistics_for_category_theori.html}{n-Category Café}.

Conformal Prediction (CP, introduced formally in Section \ref{full-cp}) is a popular statistical/ML prediction tool. Given exchangeable training data $y^n=(y_1,\ldots,y_n)$, a confidence level $\alpha \in [0,1]$, and a non-conformity measure $\psi$ (also called non-conformity {\em score}), CP returns a prediction set, called the Conformal Prediction Region (CPR), that contains the ``true value'' $y_{n+1}$ with high probability $1-\alpha$ \citep{vovk2005algorithmic,shafer2008tutorial,vovk-trans,angelopoulos2024theoreticalfoundationsconformalprediction}; For extensions and more recent references, see e.g. \citep{caprio-conformalized,alireza,gibbs2023conformal,barber2023conformal,papadopoulos2008inductive}.
Throughout the paper we focus on \emph{Full} Conformal Prediction, introduced in Section \ref{full-cp}. This is partly for conceptual clarity---our main contributions concern the structural representation of CP as a set-valued statistical functional---and partly because Full CP is the cleanest setting in which our commuting diagram results are naturally expressed. Split/Inductive Conformal Prediction is central in applications, and we expect several aspects of our framework to extend to it; A systematic treatment is left for future work.

In \citet[Section 5.1]{caprio2025conformalpredictionregionsimprecise}, CP is shown to be a {\em correspondence} \citep[Chapter 17]{aliprantis} (i.e. a {\em set-valued function}, sometimes also called a {\em relation} \citep{perrone2024starting}), thus highlighting a connection between ML and functional analysis \citep{aliprantis}. Other such connections have been explored in the literature, see e.g. \citet{Korolev}. Building on this CP-as-a-correspondence intuition, we develop a category-theoretic approach to CP that clarifies the structural underpinnings of distribution-free uncertainty quantification.

\textbf{Contributions.} Our category-theoretic approach yields three main findings.

\begin{itemize}
    \item {\em Quantified uncertainty.} In Section \ref{main}, we show that, under mild regularity conditions, CP can be represented as a morphism in two natural categories of correspondences, capturing stability and measurability of the procedure. In both settings, the conformal region factorizes through a probability set correspondence and a prediction region (called Imprecise Highest Density Region, IHDR) correspondence, so that the indirect construction
    \(
    \text{data} \to \text{set of probabilities (called {\em credal set}, see Section \ref{ips-section})} \to \text{prediction region}
    \)
    is equivalent to the direct conformal construction at the level of procedures, not merely equal \emph{ex post} as sets. 
    
    This makes explicit that CP is not only a method for representing predictive uncertainty via regions, but also a tool for \emph{quantifying} predictive uncertainty in a decision-relevant way, through summaries of the induced credal set, such as lower/upper probabilities of assertions of interest and robust lower/upper expected losses for actions. As a byproduct, one can directly assess how the choice of non-conformity measure $\psi$ affects different notions of predictive uncertainty. While the coincidence between credal set-derived prediction regions and conformal regions was observed in \citet[Proposition 5]{caprio2025conformalpredictionregionsimprecise}, our contribution is to identify conditions under which this agreement is \emph{intrinsic} at the level of procedures.

    \item {\em Conceptual unification.} In Section \ref{unifying}, we show that CP bridges Bayesian, frequentist, and imprecise probabilistic approaches to predictive statistical reasoning, thus formalizing an intuition first put forth by \citet{martin2022valid}. 
    
On the one hand, the exact commuting diagram from Section \ref{main} makes precise when the indirect procedure (data \(\to\) credal set \(\to\) prediction region) is equivalent to the direct conformal procedure. On the other hand, when the non-conformity score \(\psi\) is derived from a Bayesian posterior predictive density and standard stability conditions hold, the conformal region converges, in Hausdorff distance, to the corresponding Bayesian posterior predictive density level set. We also provide quantitative convergence rates under local empirical process and boundary regularity assumptions. Thus, asymptotically, the Bayesian density level set, the direct conformal prediction region, and the IHDR obtained from the induced credal set become equivalent not only as sets, but as predictive procedures. This provides a principled sense in which CP acts as a finite-sample validity wrapper around Bayesian prediction, while remaining model-free and interfacing naturally with imprecise probabilistic semantics. Furthermore, we identify conditions under which upper posterior constructions are related to e-posteriors, clarifying when e-value-based and conformal-imprecise representations can coincide. 

    \item {\em Privacy-compatible aggregation.} In Section \ref{impr-prob-sec}, we formalize the IHDR extractor as a covariant functor from a category of credal sets to a category of subsets of $Y$. In turn, the CPR can be seen as the functor image of the credal set $\mathcal{M}$ induced by the CP procedure \citep{cella2022validity}, which coincides with the classical conformal prediction region. 
    
    This functorial viewpoint clarifies how CP behaves under inclusion-preserving transformations of the credal information being shared, and it yields a natural privacy-compatible interpretation: In federated settings \citep{LIU2024128019}, agents need only share summary objects---namely, possibly privatized credal sets---rather than raw data, while still obtaining conservative global prediction regions and preserving coverage guarantees \citep{fang2025trustworthyaisafetybias}.
\end{itemize}

These results hold under transparent and largely standard regularity conditions (on $Y$, on $\psi$, and on the credibility level $\alpha$), which are stated explicitly before each theorem.

Our work is organized as follows. To make the paper reasonably self-contained, Section \ref{background} reviews the necessary background on CP, Category Theory, and Imprecise Probability (IP). 
IPs \citep{walley1991statistical,augustin2014introduction,TroffaesDeCooman2014} are mathematical models for reasoning in the presence of both uncertainty (as in classical probability theory) and ambiguity (that is, full or partial ignorance about the stochastic model).
Their applications to statistics and ML form a lively subject field \citep{caprio2024credal,pmlr-v216-sale23a,second-order,credal-learning,inn,chau2025integralimpreciseprobabilitymetrics,novel_bayes,impr-MSG}. Functional Analysis and further Category Theory notions are reviewed in Appendices D and E in the Supplementary Material, respectively.

In Section \ref{main}, we introduce the category $\mathbf{UHCont}$, whose objects are topological spaces and whose morphisms (relationships between the objects) are upper hemicontinuous correspondences, i.e. the set-valued analogue of upper semicontinuous functions. Under minimal assumptions on the state space $Y$, on the credibility level $\alpha$, and on the class $\mathscr{F}$ of non-conformity measures $\psi$, we show that (Full) Conformal Prediction is a morphism of $\mathbf{UHCont}$. This is not merely a change of language: Upper hemicontinuity formalizes a stability property of the prediction correspondence, namely that small perturbations in the inputs (data and non-conformity measure) cannot produce pathological outward ``jumps'' of the output region. This makes CP compatible with modular, functional-style implementations where the numerical components (score evaluation, sorting/ranking, p-value computation) can benefit from standard tooling such as automatic batching and JIT compilation, while the finite-sample coverage guarantee remains intact.

CP is also a morphism of the category $\mathbf{WMeas}_\text{uc}$ (briefly introduced in Section \ref{main} and treated in detail in Appendix A in the Supplementary Material), whose objects are compact Polish spaces and whose morphisms are weakly measurable, uniformly compact-valued correspondences. We study both $\mathbf{UHCont}$ and $\mathbf{WMeas}_\text{uc}$ because continuity and measurability are fundamental---but generally non-equivalent---notions in functional analysis: the former captures stability of the set-valued functional, while the latter ensures that CP defines a well-posed random set and allows the study of probabilistic functionals of the region.

Not only is the CP correspondence a morphism in these categories, but it is also part of a commuting diagram in both. Such a factorization through a credal set correspondence and an IHDR correspondence endows CP with intrinsic (cardinal) uncertainty quantification capabilities beyond the mere size of the prediction set.

In Section \ref{unifying}, we show that, under transparent regularity conditions, a diagram encoding how Bayesian CP \citep{fong-bayes}, classical CP, and CP-informed imprecise probability theory each generate an \(\alpha\)-level prediction region commutes \emph{asymptotically in probability} (in Hausdorff distance). Concretely, when the non-conformity measure is derived from a Bayesian posterior predictive density and standard stability conditions hold, the CPR converges to the corresponding Bayesian predictive density level set, while remaining exactly representable as an IHDR extracted from an induced credal set. In this sense, the Bayesian density level set, the direct conformal region, and the credal-IHDR route become asymptotically equivalent as predictive procedures, not merely close as sets. This strengthens the view that CP bridges Bayesian, frequentist, and imprecise approaches to predictive statistical reasoning, first put forth by \citet{martin2022valid}. We further provide quantitative convergence rates under local empirical process and boundary regularity assumptions, and identify conditions under which upper posterior constructions are related to e-posteriors. 

Section \ref{impr-prob-sec} expands on category-theoretic Imprecise Probability theory \citep{liellcock2024compositionalimpreciseprobability}. There, we formalize IHDR \citep{coolen1992imprecise}---which can be read as a robust analogue of a highest density \emph{predictive} region, accounting for distributional ambiguity through a credal set---as a covariant functor $\mathsf{IHDR}_\alpha$, $\alpha\in[0,1]$, from the category of credal sets (with inclusion morphisms) to the category of subsets of $Y$ (again with inclusion morphisms). If multiple sources or agents produce credal sets $\mathcal{M}_j$, privacy-preserving outer approximations applied locally to these summary objects propagate monotonically through the functor. The resulting prediction regions can therefore retain the global coverage guarantee while requiring the exchange of summary objects rather than raw data; In this sense, privacy becomes a modular design choice at the level of credal summaries rather than raw samples.

Section \ref{concl} concludes our work, and discusses some open research questions. We provide full, rigorous proofs of our category-theoretic results in Appendix B of the Supplementary Material; In the main text, we include only proofs for the main theorems.
It is worth mentioning Appendix C in the Supplementary Material, where we study the properties of $\mathbf{UHCont}$ and $\mathbf{WMeas}_\text{uc}$. In particular, we show that (after slightly restricting the former) they are both monoidal, and find faithful functors between them. In addition, we find which of their subcategories admit a monad, give sufficient conditions for CP to be a morphism in such subcategories, and find a faithful functor between them.

\section{Background}\label{background}

In this section, we introduce (Full) Conformal Prediction (Section \ref{full-cp}), the concepts of category, commutative diagram, and functor (Section \ref{category-back}), and some basic notions of Imprecise Probability Theory (Section \ref{ips-section}). The reader familiar with any of these concepts can skip the associated part.

\subsection{Full Conformal Prediction}\label{full-cp}

This section is an expanded reproduction of \citet[Section 2.1]{caprio2025conformalpredictionregionsimprecise}. It is based on \citet[Section 4.1]{cella2022validity}, who summarize work by \citet{shafer2008tutorial} and \citet{vovk2005algorithmic} on Transductive (or Full) Conformal Prediction. It is worth mentioning that the version of CP that we consider does not directly involve covariates; Of course, our treatment can be extended to with-covariate problems, at the cost of more complicated notation.

Suppose that there is an exchangeable process $\mathcal{Y}_1,\mathcal{Y}_2,\ldots$ with distribution $\mathfrak{P}$, where each $\mathcal{Y}_i$ is a random element taking values in the measurable space $(Y,\Sigma_Y)$. 
Recall that a sequence is exchangeable if, for any $k\in\mathbb{N}$ and any permutation $\text{perm}(\cdot)$, the two random vectors $(\mathcal{Y}_1,\ldots,\mathcal{Y}_k) $ and  $(\mathcal{Y}_{\text{perm}(1)},\ldots,\mathcal{Y}_{\text{perm}(k)}) $ have the same joint distribution. This  implies that the marginal distributions of the $\mathcal{Y}_i$’s are the same.\footnote{Exchangeability does not imply independence, so the standard i.i.d. setup is in a sense more restrictive. Note also that we are not assuming any parametric form for the distribution $\mathfrak{P}$.} 

We want to solve the following statistical problem. Suppose we observe the first $n$ terms of the process, that is, $\mathcal{Y}^n = (\mathcal{Y}_1, \ldots , \mathcal{Y}_n) $. With
this data, and the assumption of exchangeability, the goal is to predict $\mathcal{Y}_{n+1}$ using a method that is valid or reliable in a certain sense. 

Let $\mathcal{Y}^{n+1}=(\mathcal{Y}^n,\mathcal{Y}_{n+1}) $ be an $(n+1)$-dimensional vector consisting of the observable $\mathcal{Y}^n$ and the yet-to-be-observed value $\mathcal{Y}_{n+1}$. Consider the transform 
$$
\mathcal{Y}^{n+1} \rightarrow T^{n+1}=(T_1,\ldots,T_{n+1}) 
$$ 
defined by the rule 
$$
T_i\coloneqq \psi_i \left(\mathcal{Y}^{n+1} \right) \equiv \psi \left(y^{n+1}_{-i},y_i \right) 
$$
for all $i\in\{1,\ldots,n+1\}$, 
where $y^{n+1}_{-i}=y^{n+1}\setminus\{y_i\}$ and $\psi:Y^n \times Y \rightarrow \mathbb{R}$ is a fixed function that is invariant to permutations in its first vector argument. Function $\psi$, which is called a \textit{non-conformity measure} (or {\em non-conformity score}), is constructed in such a way that $\psi_i(y^{n+1})$ is small if and only if $y_i$ agrees with---i.e. is ``close to''---a prediction based on the data $y^{n+1}_{-i}$. Or, the other way around, large values $\psi_i(y^{n+1})$ suggest that the observation $y_i$ is ``strange'' and does not conform to the rest of the data $y^{n+1}_{-i}$.
The key idea is to define $\psi_i(y^{n+1})$ in a way that allows us to compare $y_i$ to a suitable summary of $y^{n+1}_{-i}$, e.g.  $\psi_i(y^{n+1})=|\text{mean}(y^{n+1}_{-i})-y_i|$, for all $i\in\{1,\ldots,n+1\}$. Notice that transformation $\mathcal{Y}^{n+1} \rightarrow T^{n+1}$ preserves exchangeability. 
%

As the value $\mathcal{Y}_{n+1}$ has not yet been observed, and is actually the prediction target, the above calculations cannot be carried out exactly. Nevertheless, the exchangeability-preserving properties of the transformations described above provide a procedure to rank candidate values $\tilde{y}$ of $\mathcal{Y}_{n+1}$ based on the observed $\mathcal{Y}^n=y^n$, as shown in Algorithm \ref{algo1}, where $\mathbbm{1}[\cdot]$ denotes the indicator function.

\begin{algorithm}
\caption{Full Conformal prediction (CP)}\label{algo1}
\begin{algorithmic}
\State Initialize: data $y^n$, non-conformity measure $\psi$, grid of $\tilde{y}$ values
\For{each $\tilde{y}$ value in the grid} 
\State set $y_{n+1}=\tilde{y}$ and write $y^{n+1}=y^n \cup \{y_{n+1}\}$;
\State define $T_i=\psi_i(y^{n+1})$, for all $i\in\{1,\ldots,n+1\}$;
\State evaluate $\pi(\tilde{y},y^n)=(n+1)^{-1}\sum_{i=1}^{n+1} \mathbbm{1}[T_i \geq T_{n+1}]$;
\EndFor
\State return $\pi(\tilde{y},y^n)$ for each $\tilde{y}$ on the grid.
\end{algorithmic}
\end{algorithm}

The output of Algorithm \ref{algo1} is a data-dependent function,
$$\pi(\cdot,y^n): Y \rightarrow [0,1], \qquad \tilde{y} \mapsto \pi(\tilde{y},y^n) \coloneqq \frac{1}{n+1}\sum_{i=1}^{n+1} \mathbbm{1}[T_i \geq T_{n+1}],$$ 
that can be interpreted as a measure of plausibility of the assertion that $\mathcal{Y}_{n+1} = \tilde{y}$, given data $y^n$; \citet{vovk2005algorithmic} refer to the function $\pi$ as \textit{conformal transducer}. 
Conformal transducer $\pi$ plays a key role in the construction of \textit{conformal prediction regions} (CPRs). For any $\alpha\in [0,1]$, the $\alpha$-level CPR is defined as \citep[Equation (2)]{vovk-trans}
\begin{equation}\label{eq_imp5}
    \mathscr{R}_\alpha^\psi(y^n)\coloneqq\{y_{n+1}\in Y : \pi(y_{n+1},y^n)> \alpha\},
\end{equation}
and it satisfies
\begin{equation}\label{eq_imp6}
    P\left[\mathcal{Y}_{n+1}\in\mathscr{R}_\alpha^\psi(y^n)\right] \geq 1-\alpha,
\end{equation}
uniformly in $n$ and in $P$ \citep{vovk2005algorithmic}. That is, \eqref{eq_imp6} is satisfied for all $n\in\mathbb{N}$ and all exchangeable distributions $P$. 
Note that \eqref{eq_imp6} holds regardless of the choice of the non-conformity measure $\psi$. However, the choice of this function is crucial in terms of the {\em efficiency} of conformal prediction, that is, the {\em size} of the prediction regions. A simple visual representation of Full CP is given in Figure \ref{fig2}, bottom route.


\citet[Section 5]{caprio2025conformalpredictionregionsimprecise} point out how Full CP can be written as a correspondence

\begin{equation}\label{conf-pred-correspond-def}
\kappa: [0,1] \times Y^n \times \mathscr{F} \rightrightarrows Y, \quad (\alpha,y^n,\psi) \mapsto \kappa(\alpha,y^n,\psi),
\end{equation}

where 
\begin{align}\label{funct-coll}
    \mathscr{F}\coloneqq \{\psi: Y^n \times Y \rightarrow \mathbb{R} \text{, } \psi \text{ is invariant to permutations in its first argument}\},
\end{align}
and the image of $\kappa$ is defined as 
\begin{align}\label{kappa-definition}
    \kappa(\alpha,y^n,\psi) \coloneqq \Bigg\{y_{n+1}\in Y : \underbrace{\frac{1}{n+1}\sum_{i=1}^{n+1} \mathbbm{1}\left[\psi \left(y^{n+1}_{-i},y_i \right)  \geq \psi \left(y^{n+1}_{-(n+1)},y_{n+1} \right)\right]}_{\eqqcolon \pi(y_{n+1},y^n)} > \alpha\Bigg\}.
\end{align}
It is easy to see, then, that  $\kappa(\alpha,y^n,\psi) \equiv \mathscr{R}_\alpha^\psi(y^n)$. In this paper, we refer to the construction defined by \eqref{conf-pred-correspond-def}-\eqref{kappa-definition} as the {\em direct conformal construction} of prediction regions.

Full Conformal Prediction (FCP) is not the only way of carrying out conformal prediction. There exists another type called Inductive (or Split) Conformal Prediction (ICP) \citep{papadopoulos2008inductive,angelopoulos2024theoreticalfoundationsconformalprediction}. It too assumes exchangeability, and it is used to build the same region $\mathscr{R}_\alpha^\psi(y^n)$ in \eqref{eq_imp5} having the same guarantee \eqref{eq_imp6}, while being computationally less expensive than FCP;
In turn, ICP is central in applications.
However, it does not coincide in general  with FCP as a set-valued map at finite sample sizes, even though both enjoy finite-sample coverage guarantees under exchangeability.
Since our focus in this paper is primarily conceptual and structural, we develop our results for FCP, which provides the cleanest setting for the categorical factorizations and commuting diagram statements.
We expect several aspects of our framework to extend to ICP; A systematic treatment is left for future work.

\subsection{Category Theory}\label{category-back}

We begin by introducing an important notion from Category Theory (CT), namely that of a category; It is needed to understand the results in Sections \ref{main}, \ref{unifying}, and \ref{impr-prob-sec}. The formal definitions we provide here and in Appendix E in the Supplementary Material all come from \citet{perrone2024starting}.

In layman terms, a category is a setup comprising objects and relationships between such objects that can be chained.
The formal definition follows.

\begin{definition}[Category]%
\label{def:category-back}
A \emph{category} $\mathbf C$ consists of
\begin{itemize}
  \item a collection $\mathbf C_0$ of \emph{objects} (written $X,Y,Z,\dots$);
  \item a collection $\mathbf C_1$ of \emph{morphisms} (written $f,g,h,\dots$);
\end{itemize}
together with
\begin{enumerate}
  \item for every morphism $f$ there are two distinguished objects called its \emph{source} and \emph{target}, written $s(f)$ and $t(f)$ (we abbreviate $f\colon X\rightarrow Y$ when $s(f)=X$, $t(f)=Y$);
  \item for every object $X$ an \emph{identity} morphism $\mathrm{id}_X\colon X\rightarrow X$;
  \item for every composable pair $X \xrightarrow{f} Y \xrightarrow{g} Z$ (that is, for which $t(f)=s(g)$) a \emph{composite} morphism $g\circ f\colon X\rightarrow Z$;
\end{enumerate}
satisfying the \emph{unitality} laws $f\circ\mathrm{id}_X = f = \mathrm{id}_Y\circ f$ and the \emph{associativity} law
$(h\circ g)\circ f = h\circ(g\circ f)$, whenever the expressions are well defined.
\end{definition}

Beyond objects and morphisms, Category Theory provides a language for comparing different \emph{compositions} of morphisms.
This is captured by the notion of a \emph{commutative diagram}.

\begin{definition}[Commutative Diagram]\label{def:comm-diag}
A \emph{diagram} in a category $\mathbf C$ is a collection of objects and morphisms arranged in a graph-like shape.
A diagram is said to \emph{commute} if, whenever there are two (or more) directed paths with the same source and target objects,
the corresponding composites are equal as morphisms in $\mathbf C$.
\end{definition}

Commutativity is the categorical way of expressing that a construction is \emph{intrinsically} a composition of simpler ones.
For instance, consider a triangle diagram
\[
\begin{tikzcd}
X \ar[r,"g"] \ar[dr,swap,"h"] & Y \ar[d,"f"] \\
& Z
\end{tikzcd}
\]
in $\mathbf C$. The diagram commutes precisely when the two paths from $X$ to $Z$ agree, that is,
\begin{equation}\label{eq:triangle-comm}
h = f\circ g.
\end{equation}
Equation \eqref{eq:triangle-comm} should be read as more than an equality of outputs: It is an equality of \emph{procedures}.
It states that the morphism $h$ \emph{factors} through $Y$ via $g$ followed by $f$, so $Y$ can be interpreted as an intermediate
representation and $g$ and $f$ as two successive processing steps.
In particular, any property that is preserved under composition (e.g. continuity or measurability, when $\mathbf C$ is a category of such morphisms)
can be studied modularly by analyzing $g$ and $f$ separately, and then composing the results.

As we shall see in Sections \ref{main} and \ref{unifying}, in the present paper we use commuting diagrams to formalize when an apparently ``direct''
prediction procedure (Conformal Prediction) is equivalent, as a morphism, to an ``indirect'' one obtained by first extracting an intermediate object
(a set of probabilities) and then mapping it to a prediction region. This is the sense in which commutativity encodes an intrinsic equivalence between
$h$ and the composite $f\circ g$.

Categories are the building blocks in CT. It is natural, then, to ask how to relate them. This is captured by the notion of a {\em functor}.

\begin{definition}[Functor]%
\label{def:functor-back}
Let $\mathbf C,\mathbf D$ be categories.  
A \emph{functor} $F\colon\mathbf C\rightarrow\mathbf D$ assigns
\begin{itemize}
  \item to every object $X$ of $\mathbf C$ an object $FX$ of $\mathbf D$;
  \item to every morphism $f\colon X\rightarrow Y$ in $\mathbf C$ a morphism $Ff\colon FX\rightarrow FY$ in $\mathbf D$,
\end{itemize}
such that
\[
F(\mathrm{id}_X)=\mathrm{id}_{FX}, \qquad 
F(g\circ f)=Fg\circ Ff
\]
for all composable $f,g$ in $\mathbf C$. 
\end{definition}

\subsection{Imprecise Probabilities}\label{ips-section}
We now introduce a few imprecise probabilistic concepts that will be used when studying CP in the Category Theory framework.

Let $(Y,\Sigma_Y)$ be a measurable space, and call $\Delta_Y$ the space of finitely additive probabilities on $Y$, endowed with the weak$^\star$ topology; The reason we work with finitely additive probabilities is explained below. A {\em credal set} $\mathcal{M}$ is a weak$^\star$-closed and convex element of $\Delta_Y$ \citep{levi2}. Throughout the paper, we will denote by $\mathscr{C} \subset 2^{\Delta_Y}$ the space of nonempty weak$^\star$-closed and convex subsets of $\Delta_Y$.
For any nonempty credal set $\mathcal{M}\in\Ccal$, define the {\em upper probability} as the upper envelope of $\mathcal{M}$,
\[
  \overline P_\mathcal{M}(A) \equiv \overline P(A) = \sup_{P\in\mathcal{M}} P(A),
  \quad A\in\Sigma_Y.
\]

The {\em conjugate lower probability} $\underline{P}_\mathcal{M} \equiv \underline{P}$ of $\overline{P}$ satisfies
$\underline{P}(A)=1-\overline{P}(A^c)$, for all $A\in\Sigma_Y$,
thus knowing one immediately gives us the other; In a statistical and ML framework, these concepts allow us to crystallize model ambiguity.

Even in the presence of such ambiguity, we are able to define (a version of) a credible interval. Fix a credibility level $\alpha\in[0,1]$. Then, the $\alpha$-level Imprecise Highest Density Region (IHDR) of credal set $\mathcal{M}$ is the set $B_\mathcal{M}^\alpha \subset Y$ such that $\underline P(B_\mathcal{M}^\alpha)=1-\alpha$, and its ``size is minimal'' \citep{coolen1992imprecise}. More formally \citep[Equation (11)]{caprio2025conformalpredictionregionsimprecise}, 
$$B_\mathcal{M}^\alpha = \bigcap \left\lbrace{A \in \Sigma_Y : \underline{P}(A) \geq 1-\alpha}\right\rbrace.$$
Notice that $B_\mathcal{M}^\alpha$ can be written as the image of the following correspondence
\begin{align}\label{ihdr-corresp}
    \IHDR:[0,1] \times \Ccal \rightrightarrows Y, \quad (\alpha,\mathcal{M}) \mapsto \IHDR(\alpha,\mathcal{M}) \coloneqq \bigcap \left\lbrace{A \in \Sigma_Y : \underline P(A) \geq 1-\alpha}\right\rbrace.
\end{align}

Conformal Prediction is associated with a special case of upper probability $\overline{P}$,  called a {\em possibility function} \citep{DuboisPrade1988PossibilityTheory,deCooman1997Possibility1,deCooman1997Possibility2,deCooman1997Possibility3,cuzzolin2020geometry}. This representation is introduced because conformal predictors naturally produce nested plausibility levels across candidate outputs $\tilde y \in Y$, and possibility functions are exactly the upper probabilities that encode such nested uncertainty. A possibility is an upper probability for which there exists a function $\pi:Y\rightarrow [0,1]$ such that $\sup_{y\in Y}\pi(y)=1$ and $\overline{P}(A)=\sup_{y\in A} \pi(y)$, $A\in \Sigma_Y$.

\citet{cella2022validity} show that, under the so-called {\em consonance assumption}, that is, if $\sup_{y_{n+1} \in Y} \pi(y_{n+1},y^n)=1$, then the conformal transducer $\pi$ induces a possibility function $\overline{\Pi}(A)\coloneqq \sup_{y_{n+1} \in A} \pi(y_{n+1},y^n)$, for all $A\in \Sigma_Y$, and in turn a credal set $\mathcal{M}(\overline{\Pi}) \coloneqq \{P : P(A) \leq \overline{\Pi}(A), \forall A\in \Sigma_Y\}$. Such a credal set can be seen as the image of the correspondence $\CRED:Y^n \times \F \rightrightarrows \Prob$, where $\F$ is the collection of functions in \eqref{funct-coll}, and $(y^n,\psi) \mapsto \CRED(y^n,\psi) \coloneqq $

\begin{align}\label{cred-corresp}
    \Bigg\{P \in \Prob : P(A) \leq \underbrace{\sup_{y_{n+1}\in A} \left(\frac{1}{n+1}\sum_{i=1}^{n+1} \mathbbm{1}\left[\psi \left(y^{n+1}_{-i},y_i \right)  \geq \psi \left(y^{n+1}_{-(n+1)},y_{n+1} \right)\right]\right)}_{\equiv \sup_{y_{n+1} \in A} \pi(y_{n+1},y^n) \eqqcolon \overline{\Pi}(A)} \text{, } \forall A \in \Sigma_Y\Bigg\}.
\end{align}

The consonance assumption is needed to ensure that the conformal transducer admits the nested uncertainty interpretation and that, since the probabilities we work with are finitely additive, $\CRED(y^n,\psi)$ is a nonempty credal set \citep[Proposition 7.14]{TroffaesDeCooman2014}.\footnote{In particular, this implies that $\overline{\Pi}$ is coherent in the sense of \citet{walley1991statistical}.} The elements of $\CRED(y^n,\psi)$ are characterized in \citet[Theorem 1]{martin2025efficientmontecarlomethod}. In the remainder of the paper, we sometimes write $\pi(y,y^n,\psi)\equiv \pi_\psi(y,y^n)$ for notational convenience (we drop the $n+1$ subscript for $y_{n+1}$), and to explicitly acknowledge that the conformal transducer $\pi$ depends on the choice of the non-conformity measure $\psi$. 
A visual representation of the data $\to$ credal set $\to$ IHDR route---which in this paper we refer to as the {\em indirect procedure} for constructing a prediction region---is given in Figure \ref{fig2}, top route.

\begin{figure*}[h!]
\centering
\includegraphics[width=.8\textwidth]{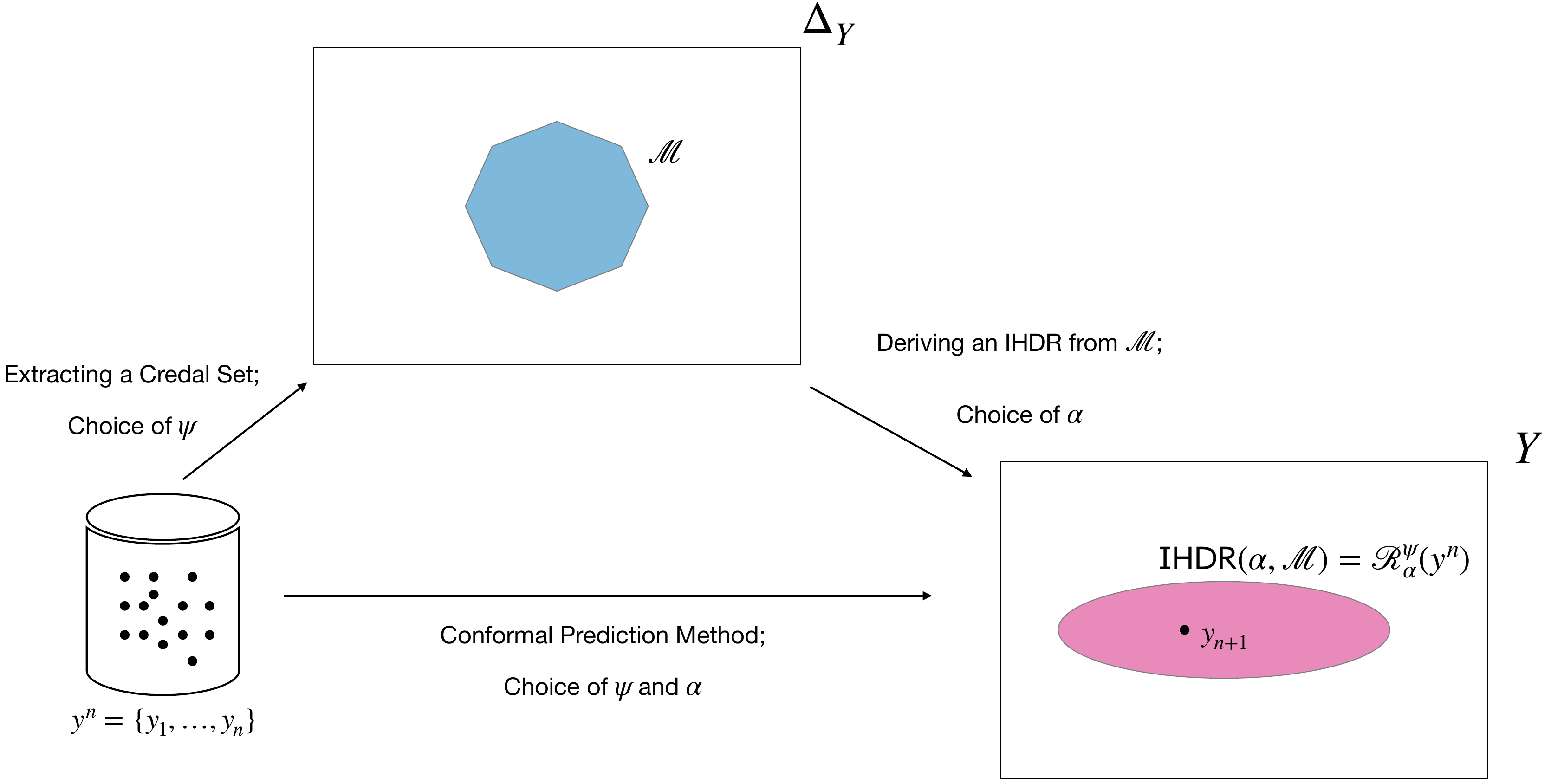}
\caption{A visual representation of both the full CP methodology, and the two-step procedure that first extracts credal set $\mathcal{M}(\overline{\Pi})$ from the training set $y^n$ as in \eqref{cred-corresp}, and then computes the IHDR as in \eqref{ihdr-corresp}.}\label{fig2}
\end{figure*}


Throughout the paper, we tacitly consider a slight modification of $\CRED$, that  significantly simplifies our proving endeavors; We denote it by $\CRED_\text{cl}(y^n,\psi)\coloneqq \{P \in \Delta_Y: P(A) \leq \overline{\Pi}(A) \text{, } \forall A \in \Sigma_Y \text{ closed}\}$. As we can see, $\CRED_\text{cl}(y^n,\psi) \supseteq \CRED(y^n,\psi)$, because it imposes less constraints on the $P$'s: They only have to be setwise dominated on closed sets. The following is an extension of \citet[Proposition 5]{caprio2025conformalpredictionregionsimprecise} showing that, under reasonable assumptions, applying the
IHDR intersection rule to the closed set relaxation
\(\CRED_{\mathrm{cl}}(y^n,\psi)\) recovers the conformal prediction region
\(\mathcal R_\alpha^\psi(y^n)\). Since the resulting set is exactly the
conformal prediction region, it inherits the usual conformal coverage
guarantee \eqref{eq_imp6}.

\begin{proposition}\label{prop-ihdr-cpr}
Fix $n\in\mathbb N$, $y^n\in Y^n$, $\psi\in\mathscr{F}$, and $\alpha\in[0,1]$.
Assume that $Y$ is compact Hausdorff and $\Sigma_Y=\mathcal B(Y)$, that $\psi$ is jointly continuous on $Y^n\times Y$, and that the conformal transducer $\pi( \cdot ,y^n)$ is consonant.
Let
\[
\overline P_{\mathrm{cl}}(A)\coloneqq 
\sup_{P\in\mathrm{CRED}_{\mathrm{cl}}(y^n,\psi)}P(A),
\qquad
\underline P_{\mathrm{cl}}(A)\coloneqq 
1-\overline P_{\mathrm{cl}}(A^c),
\qquad A\in\Sigma_Y.
\]
Then
\[
\bigcap\{A\in\Sigma_Y:\underline P_{\mathrm{cl}}(A)\ge 1-\alpha\}
=
\mathcal R^\psi_\alpha(y^n).
\]
In other words, the set obtained from \(\mathrm{CRED}_{\mathrm{cl}}(y^n,\psi)\)
by the IHDR intersection formula coincides with the conformal prediction region.
\end{proposition}

\begin{remark}
The compactness assumption on \(Y\) is not strictly necessary for the conclusion of the
previous proposition. In the proof, compactness is used only to ensure that the
upper semicontinuous conformal transducer attains its consonant supremum; That is,
that there exists \(y^\star\in Y\) such that
\(
\pi(y^\star,y^n)=1.
\)
Accordingly, the same proof goes through on any Hausdorff topological space \(Y\) with
\(\Sigma_Y=\mathcal B(Y)\), provided that \(y\mapsto \pi(y,y^n)\) is upper
semicontinuous and the supremum of the conformal transducer is attained. Compactness
is therefore a convenient sufficient condition, but the result remains valid under the
weaker requirement that the consonant supremum be achieved.
\end{remark}

\begin{proof}[Proof of Proposition \ref{prop-ihdr-cpr}]
Write
\[
B^{\mathrm{cl}}_\alpha
\coloneqq 
\bigcap\{A\in\Sigma_Y:\underline P_{\mathrm{cl}}(A)\ge 1-\alpha\}.
\]

We first prove the key identity
\[
\overline P_{\mathrm{cl}}(\{y\})=\pi(y,y^n),
\qquad\text{for every }y\in Y.
\]

\emph{Upper bound.}
Since $Y$ is Hausdorff, every singleton $\{y\}$ is closed. Hence, for every
$P\in\mathrm{CRED}_{\mathrm{cl}}(y^n,\psi)$,
\[
P(\{y\})\le \Pi(\{y\})=\pi(y,y^n).
\]
Taking the supremum over $P\in\mathrm{CRED}_{\mathrm{cl}}(y^n,\psi)$ gives
\[
\overline P_{\mathrm{cl}}(\{y\})\le \pi(y,y^n).
\]

\emph{Lower bound.}
We claim that $y\mapsto \pi(y,y^n)$ is upper semicontinuous.
Indeed, for each $i\in\{1,\dots,n+1\}$, the map
\[
y\longmapsto
\psi(y^{n+1}_{-i},y_i)-\psi(y^n,y)
\]
is continuous, because $\psi$ is jointly continuous and $y$ enters continuously in both
arguments. Therefore,
\[
y\longmapsto
\mathbbm{1} \left[
\psi(y^{n+1}_{-i},y_i)\ge \psi(y^n,y)
\right]
\]
is the indicator of a closed set, hence is upper semicontinuous. Since $\pi( \cdot ,y^n)$ is
the average of finitely many such terms, it is upper semicontinuous.

Because $Y$ is compact and $\pi( \cdot ,y^n)$ is upper semicontinuous, it attains its
supremum on $Y$. By consonance, there exists $y^\star\in Y$ such that
\[
\pi(y^\star,y^n)=1.
\]

Now fix $y\in Y$ and define
\[
P_y\coloneqq \pi(y,y^n) \delta_y+\bigl(1-\pi(y,y^n)\bigr)\delta_{y^\star}.
\]
Then $P_y\in\Delta_Y$. We show that $P_y\in\mathrm{CRED}_{\mathrm{cl}}(y^n,\psi)$.
Let $F\subseteq Y$ be closed.

If $y^\star\in F$, then
\[
P_y(F)\le 1=\pi(y^\star,y^n)\le \Pi(F),
\]
because $\Pi(F)=\sup_{z\in F}\pi(z,y^n)\ge \pi(y^\star,y^n)=1$.

If $y^\star\notin F$ and $y\notin F$, then $P_y(F)=0\le \Pi(F)$.

If $y^\star\notin F$ and $y\in F$, then
\[
P_y(F)=\pi(y,y^n)\le \sup_{z\in F}\pi(z,y^n)=\Pi(F).
\]

Thus $P_y(F)\le \Pi(F)$ for every closed $F$, so $P_y\in\mathrm{CRED}_{\mathrm{cl}}(y^n,\psi)$.
Consequently,
\[
\overline P_{\mathrm{cl}}(\{y\})
\ge
P_y(\{y\})
\ge
\pi(y,y^n).
\]
If $y\neq y^\star$, then in fact $P_y(\{y\})=\pi(y,y^n)$; if $y=y^\star$, then also
$P_y(\{y\})=1=\pi(y,y^n)$. Hence
\[
\overline P_{\mathrm{cl}}(\{y\})\ge \pi(y,y^n).
\]
Combining with the upper bound yields
\[
\overline P_{\mathrm{cl}}(\{y\})=\pi(y,y^n)
\qquad(\forall y\in Y).
\]

We now prove the equality $B^{\mathrm{cl}}_\alpha=\mathcal R^\psi_\alpha(y^n)$.

\emph{First inclusion:} $\mathcal R^\psi_\alpha(y^n)\subseteq B^{\mathrm{cl}}_\alpha$.

Let $y\in \mathcal R^\psi_\alpha(y^n)$. Then $\pi(y,y^n)>\alpha$.
Suppose, by contradiction, that there exists
$A\in\Sigma_Y$ such that $\underline P_{\mathrm{cl}}(A)\ge 1-\alpha$ and $y\notin A$.
Since $\{y\}\subseteq A^c$, monotonicity of $\overline P_{\mathrm{cl}}$ gives
\[
\overline P_{\mathrm{cl}}(A^c)\ge \overline P_{\mathrm{cl}}(\{y\})=\pi(y,y^n)>\alpha.
\]
Therefore,
\[
\underline P_{\mathrm{cl}}(A)
=
1-\overline P_{\mathrm{cl}}(A^c)
<
1-\alpha,
\]
a contradiction. Hence every set $A$ with $\underline P_{\mathrm{cl}}(A)\ge 1-\alpha$
must contain $y$, so $y\in B^{\mathrm{cl}}_\alpha$.

\emph{Second inclusion:} $B^{\mathrm{cl}}_\alpha\subseteq \mathcal R^\psi_\alpha(y^n)$.

Let $y\notin \mathcal R^\psi_\alpha(y^n)$. Then $\pi(y,y^n)\le \alpha$.
Set
\[
A_y\coloneqq Y\setminus\{y\}.
\]
Since $\{y\}$ is closed, $A_y\in\Sigma_Y$. Using the singleton identity proved above,
\[
\overline P_{\mathrm{cl}}(A_y^c)
=
\overline P_{\mathrm{cl}}(\{y\})
=
\pi(y,y^n)
\le \alpha.
\]
Hence
\[
\underline P_{\mathrm{cl}}(A_y)
=
1-\overline P_{\mathrm{cl}}(A_y^c)
\ge 1-\alpha.
\]
So $A_y$ is one of the sets appearing in the defining intersection of
$B^{\mathrm{cl}}_\alpha$. But $y\notin A_y$, therefore $y\notin B^{\mathrm{cl}}_\alpha$.

Thus
$B^{\mathrm{cl}}_\alpha\subseteq \mathcal R^\psi_\alpha(y^n)$.
Combining the two inclusions, we obtain
$B^{\mathrm{cl}}_\alpha=\mathcal R^\psi_\alpha(y^n)$.
\end{proof}

To avoid notation complications, in the remainder of this work we do not distinguish between $\CRED(y^n,\psi)$ and $\CRED_\text{cl}(y^n,\psi)$.

\section{Conformal Prediction as a Morphism}\label{main}



This section introduces categories \textbf{UHCont} and $\mathbf{WMeas}_{\text{uc}}$, and shows that, under minimal assumptions and a fixed choice of $\alpha$, the Conformal Prediction correspondence $\kappa$ is a morphism of both, embedded in a commuting diagram; Let us briefly go over again our category-theoretic motivations.
Conformal prediction maps inputs (data and a non-conformity measure) to an output that is inherently set-valued (a prediction region).
Rather than treating CP as a monolithic algorithm, we study it as a morphism in categories of correspondences chosen to encode two fundamental notions of well-posedness:
Stability (upper hemicontinuity) and measurability (so that CP defines a random compact set).
This perspective makes structural statements precise: Commuting diagrams encode exact factorizations, allowing us to view CP as the composition of simpler procedures,
and to interpret intermediate objects (credal sets) as carrying additional uncertainty information, while the finite-sample coverage guarantee remains intact. In particular, the main diagram expresses that the classical conformal route and the indirect route through a credal set followed by the IHDR extractor are intrinsically equivalent: They produce the same prediction region while exposing different structural information.

One could of course state our main equalities purely as identities between set-valued maps.
The categorical viewpoint is not invoked to make the equalities true; Rather, it is used to make them \emph{structurally meaningful}.
By working in categories whose morphisms encode stability and measurability, the commuting diagrams become statements about well-posedness and compositionality of the underlying statistical functionals.
This yields practical consequences: It justifies modular implementations, and enables uncertainty summaries on intermediate credal objects.

\subsection{Setup}

The following is inspired by \citet[Example 1.1.14]{perrone2024starting}.

\begin{definition}[Structure \textbf{UHCont}]\label{structure-uhcont}
Define a structure $\mathbf{UHCont}$ as follows,
\begin{itemize}
  \item \textbf{Objects:} Topological spaces $(X, \tau_X)$.
  \item \textbf{Morphisms:} Upper hemicontinuous correspondences $\Phi : X \rightrightarrows Y$ (see Definition 13 in Appendix D in the Supplementary Material).
\end{itemize}
\end{definition}

\begin{theorem}[\textbf{UHCont} is a Category]\label{cat-uhc}
The structure $\mathbf{UHCont}$ defines a category.
\end{theorem}

\begin{hproof}
We verify that $\mathbf{UHCont}$ satisfies the axioms of a category of Definition \ref{def:category-back}. 
For every topological space $(X,\tau_X)$, the identity correspondence 
$\mathrm{id}_X : X \rightrightarrows X$, given by $\mathrm{id}_X(x)=\{x\}$, is upper hemicontinuous, and hence is a morphism in $\mathbf{UHCont}$. 
Moreover, if $\Phi : X \rightrightarrows Y$ and $\Psi : Y \rightrightarrows Z$ are upper hemicontinuous correspondences, then their composition
\[
  \Psi \circ \Phi : X \rightrightarrows Z,
  \qquad
  (\Psi \circ \Phi)(x) = \bigcup_{y \in \Phi(x)} \Psi(y),
\]
is again upper hemicontinuous. Thus composition is well-defined in $\mathbf{UHCont}$. 
Associativity of composition follows from associativity of relational composition, and the identity correspondences act as left and right identities. 
Therefore, $\mathbf{UHCont}$ defines a category.
\end{hproof}

\subsection{Main Result}

We now show that, under minimal assumptions and a fixed choice of $\alpha$, the (Full) Conformal Prediction correspondence $\kappa$ in  \eqref{conf-pred-correspond-def} is a morphism of $\mathbf{UHCont}$, and that it is part of a commutative diagram involving the IHDR and the credal set correspondences. 
Fix any $n \in \mathbb{N}$, and consider the following conditions,

\newcommand{\kappaCorr}{\kappa}
\newcommand{\SigY}{\Sigma_Y}
\newcommand{\piVal}{\pi}
\newcommand{\credset}{\mathcal M}

\begin{enumerate}
\renewcommand{\labelenumi}{(\roman{enumi})}
  \renewcommand{\theenumi}{(\roman{enumi})}
  \item\label{H1}
    $\Y$ is compact Hausdorff and $\SigY=\mathcal{B}(Y)$ is its Borel $\sigma$-field.
  \item\label{H2}
    Each $\psi\in\F$ is jointly continuous on $\Y^{n}\times\Y$; We can then write $\F \subset C(Y^n \times Y)$, the space of jointly continuous real-valued functions on $Y^n \times Y$. $\F$ is endowed with the uniform ($\sup$-norm) topology.
  \item\label{H3}
    \emph{No-tie parameter:}
    The credibility level $\alpha$ satisfies
    $\alpha\notin S_{n+1}\coloneqq \{0, \frac{1}{n+1}, \frac{2}{n+1}, \ldots, \frac{n}{n+1}, 1\}$.
  \item\label{H5}
    $\Prob$ is endowed with the weak$^\star$ topology,
    and $\mathscr C\subset 2^{\Prob}$ is endowed with the Vietoris topology (see Definition 16 in Appendix D in the Supplementary Material).
\end{enumerate}



We are now ready for the main result of this section.

\begin{theorem}[The Full Conformal Prediction Diagram Commutes in \(\mathbf{UHCont}\)]
\label{conf-pred-diagr}
Fix \(n\in\mathbb N\), and let
$X\coloneqq Y^n\times\mathscr F$
with the product topology. Fix \(\alpha\in[0,1]\) satisfying \ref{H3}, and assume
\ref{H1}, \ref{H2}, and \ref{H5}. Assume moreover that the conformal transducer is
consonant on \(X\).
Define the direct conformal correspondence
\[
\kappa_\alpha:X\rightrightarrows Y,
\qquad
\kappa_\alpha(y^n,\psi)
\coloneqq
\{y\in Y:\pi_\psi(y,y^n)>\alpha\},
\]
and the pulled-back indirect correspondence
\[
\eta_\alpha:X\rightrightarrows Y,
\qquad
\eta_\alpha(y^n,\psi)
\coloneqq
\IHDR_\alpha\bigl(\CRED(y^n,\psi)\bigr).
\]
Then \(\kappa_\alpha\) and \(\eta_\alpha\) are morphisms in \(\mathbf{UHCont}\), and
$\eta_\alpha=\kappa_\alpha$.
Moreover, the credal correspondence
\[
\Gamma:X\rightrightarrows \Prob,
\qquad
\Gamma(y^n,\psi)\coloneqq \CRED(y^n,\psi),
\]
is a morphism in \(\mathbf{UHCont}\).
Finally, let
\[
\mathscr M_{\mathrm{CP}}
\coloneqq
\{\CRED(y^n,\psi):(y^n,\psi)\in X\}\subseteq \mathscr C
\]
be the image of the credal-object map. Although \(\mathscr M_{\mathrm{CP}}\) is a
subset of \(\mathscr C\), in the categorical factorization below it is endowed with
the final topology induced by the map
\[
c:X\to\mathscr M_{\mathrm{CP}},
\qquad
c(y^n,\psi)\coloneqq \CRED(y^n,\psi),
\]
not necessarily with the subspace topology inherited from the Vietoris topology on
\(\mathscr C\).
Define
\[
\widehat{\IHDR}_\alpha:\mathscr M_{\mathrm{CP}}\rightrightarrows Y,
\qquad
\widehat{\IHDR}_\alpha(\mathcal M)\coloneqq \IHDR_\alpha(\mathcal M).
\]
Then \(c\), viewed as a singleton-valued correspondence, and
\(\widehat{\IHDR}_\alpha\) are morphisms in \(\mathbf{UHCont}\), and the diagram
\[
\begin{tikzcd}
X \ar{r}{c} \ar{dr}[swap]{\kappa_\alpha} &
\quad \mathscr M_{\mathrm{CP}} \ar{d}{\widehat{\IHDR}_\alpha} \\
&  Y
\end{tikzcd}
\]
commutes in \(\mathbf{UHCont}\); We call it the Full CP Diagram.
\end{theorem}



To prove Theorem \ref{conf-pred-diagr}, we need first to introduce two ancillary lemmas. The first one is a famous classical result in set-valued functional analysis \citep[Theorem 17.11]{aliprantis}.

\begin{lemma}[Closed Graph Theorem]
\label{lem:CG}
A correspondence with compact Hausdorﬀ
range space has closed graph if and only if it is upper hemicontinuous and closed-valued.
\end{lemma}

The second one is the following. 

\begin{lemma}[Upper Semicontinuity of the Upper Envelope $\overline{\Pi}$]\label{lem-H4}
    If \ref{H1} and \ref{H2} hold, then the map $(y^{n},\psi)\mapsto \overline\Pi_{(y^{n},\psi)}(A)=\sup_{y\in A}\piVal(y,y^{n},\psi)$
     is {upper semicontinuous}, for all closed sets $A\in\Sigma_Y$, with the convention \(\sup_{\emptyset}\pi=0\).
\end{lemma}

\begin{proof}[Proof]
    First, we show that, if a generic set $K$ is compact and a generic map \(f:K\to\mathbb R\) is upper semicontinuous (u.s.c.), then \(f\) attains its maximum on \(K\). Choose \(x_m\) with \(f(x_m)\ge\sup_K f-\tfrac1m\).
A compactness subsequence \(x_{m_k}\to x_\star\in K\) exists.
Upper semicontinuity gives
\(
f(x_\star)\ge\limsup f(x_{m_k})=\sup_K f.
\)

Now, we introduce auxiliary maps, and show their continuity.
For \(i\in\{1,\dots ,n+1\}\), put
\[
g_i(y_{n+1},y^{n},\psi)
   \coloneqq\psi\bigl(y^{n+1}_{-i},y_i\bigr)-\psi\bigl(y^{n},y_{n+1}\bigr).
\]
Because (a) the evaluation \(\psi\mapsto\psi(y^n,y)\) is continuous in the sup-norm, and (b) coordinate deletion \(y^{n+1}\mapsto(y^{n+1}_{-i},y_i)\) is continuous, each \(g_i\) is continuous on \(Y^{n+1}\times \F\).

Let us now show the upper semicontinuity of \(\pi_\psi\).
Fix \((y^{n+1},\psi)\).
The set \(\{g_i\ge0\}\) is closed;
For \(k\in\{0,\dots ,n+1\}\), define
\[
R_k\coloneqq\Bigl\{(y^{n+1},\psi) : 
          \pi_\psi(y_{n+1},y^{n})\ge\tfrac{k}{n+1}\Bigr\}.
\]
At least \(k\) of the inequalities
\(g_i\ge0\) must hold, so
\(R_k=\bigcup_{\lvert I\rvert=k}
      \bigcap_{i\in I}\{g_i\ge0\}\),
a finite union of finite intersections of closed sets;
Hence \(R_k\) is closed and \(\pi_\psi\) is u.s.c.

We are now ready to prove the main statement. Fix any closed set $A\in\Sigma_Y$. By \ref{H1}, it is also compact. 
Let \(\bigl(y^{n}_m,\psi_m\bigr)_{m\ge1}\)
converge to \((y^{n},\psi)\) and set
\(
L\coloneqq\limsup_{m\to\infty}\overline\Pi_{(y^{n}_m,\psi_m)}(A)
\).\footnote{If $Y$ is not first countable, we can switch from a sequence to a net argument.}
We must show \(L\le \overline\Pi_{(y^{n},\psi)}(A)\).

\emph{Step 1 (Discretization).}
Because \(\pi_\psi\) takes values in the finite set
\(\{0,\tfrac1{n+1},\dots ,1\}\),
there is an integer \(k\in\{0,\dots ,n+1\}\)
and a subsequence (still indexed by \(m\)) with
\(\overline\Pi_{(y^{n}_m,\psi_m)}(A)=k/(n+1)\), for all \(m\).

\emph{Step 2 (Select Maximizers).}
By the attained maximum result that we presented earlier, pick
\(y_{n+1}^{(m)}\in A\) such that
\(\pi_{\psi_m}\bigl(y_{n+1}^{(m)},y^{n}_m\bigr)=k/(n+1)\).
Compactness of \(A\) yields a further subsequence and a limit
\(y_{n+1}\in A\) with \(y_{n+1}^{(m)}\to y_{n+1}\).

\emph{Step 3 (Pass to the Limit).}
For every \(i\), the continuity of \(g_i\) gives
\(
g_i\bigl(y_{n+1}^{(m)},y^{n}_m,\psi_m\bigr)
\to g_i\bigl(y_{n+1},y^{n},\psi\bigr).
\)
Since at least \(k\) of the \(g_i\)’s are non-negative for each \(m\),
suppose, for contradiction, that {\em strictly fewer} than \(k\)
are non-negative at the limit.
Then there exists \(\delta>0\) such that,
for all sufficiently large \(m\),
fewer than \(k\) indices satisfy
\(g_i\bigl(y_{n+1}^{(m)},y^{n}_m,\psi_m\bigr)\ge-\delta\),
contradicting the fact that {\em every} \(m\) has
\(k\) indices with the same expression \(\ge0\).
Hence at least \(k\) of the limiting \(g_i\)’s are \(\ge0\), and
\[
\pi_\psi(y_{n+1},y^{n}) \ge \frac{k}{n+1}.
\]

\emph{Step 4 (Take the Supremum).}
Because \(y_{n+1}\in A\),
\[
\overline\Pi_{(y^{n},\psi)}(A)
       \ge \pi_\psi(y_{n+1},y^{n})
       \ge \frac{k}{n+1}
       = L.
\]
Thus \(\limsup_{m\to\infty}\overline\Pi_{(y^{n}_m,\psi_m)}(A)\le \overline\Pi_{(y^{n},\psi)}(A)\),
establishing upper semicontinuity.
\end{proof}

We are finally ready for the proof of Theorem \ref{conf-pred-diagr}.

\begin{proof}[Proof of Theorem \ref{conf-pred-diagr}]
Fix \(n\in\mathbb N\), and write \(x=(y^n,\psi)\in X\).
First observe that, by the argument of Lemma \ref{lem-H4}, the map
\[
(y,y^n,\psi)\longmapsto \pi_\psi(y,y^n)
\]
is upper semicontinuous on \(Y\times X\). Indeed, for each
\(i\in\{1,\ldots,n+1\}\), the map
\[
(y,y^n,\psi)\longmapsto
\psi(y^{n+1}_{-i},y_i)-\psi(y^n,y)
\]
is continuous, and therefore the sets on which at least \(k\) of these
\(n+1\) inequalities are nonnegative are closed.
Since \(\alpha\notin S_{n+1}\), define
\[
\beta(\alpha)
\coloneqq
\min\left\{\frac{k}{n+1}\in S_{n+1}:\frac{k}{n+1}>\alpha\right\}.
\]
Then, because \(\pi_\psi(y,y^n)\) only takes values in \(S_{n+1}\),
\[
\pi_\psi(y,y^n)>\alpha
\quad\Longleftrightarrow\quad
\pi_\psi(y,y^n)\ge \beta(\alpha).
\]
Hence,
\[
\operatorname{Gr}(\kappa_\alpha)
=
\left\{
(y^n,\psi,y)\in X\times Y:
\pi_\psi(y,y^n)\ge \beta(\alpha)
\right\}
\]
is a closed subset of \(X\times Y\), since it is a superlevel set of an
upper semicontinuous function. Because \(Y\) is compact Hausdorff by \ref{H1},
Lemma \ref{lem:CG} implies that \(\kappa_\alpha\) is upper hemicontinuous.
Thus \(\kappa_\alpha\) is a morphism in \(\mathbf{UHCont}\).

We next show that the credal correspondence
\[
\Gamma:X\rightrightarrows \Prob,
\qquad
\Gamma(y^n,\psi)=\CRED(y^n,\psi),
\]
is upper hemicontinuous. For every closed \(A\in\Sigma_Y\), define
\[
H_A
\coloneqq
\left\{
(y^n,\psi,P)\in X\times\Prob:
P(A)\le \overline\Pi_{(y^n,\psi)}(A)
\right\}.
\]
The map \(P\mapsto P(A)\) is weak\(^\star\)-continuous on \(\Prob\), since
\(\mathbbm 1_A\) is a bounded measurable function. By Lemma \ref{lem-H4}, the map
\[
(y^n,\psi)\mapsto \overline\Pi_{(y^n,\psi)}(A)
\]
is upper semicontinuous. Therefore \(H_A\) is closed.
Since we work with the closed set version of the credal construction,
\[
\operatorname{Gr}(\Gamma)
=
\bigcap_{\substack{A\in\Sigma_Y\\ A\text{ closed}}} H_A .
\]
Thus \(\operatorname{Gr}(\Gamma)\) is closed in \(X\times\Prob\). By \ref{H5}
and Banach-Alaoglu, \(\Prob\) is compact Hausdorff in the weak\(^\star\)
topology. Hence Lemma \ref{lem:CG} implies that \(\Gamma\) is upper
hemicontinuous. Therefore \(\Gamma\) is a morphism in \(\mathbf{UHCont}\).

Now define
\[
\eta_\alpha(y^n,\psi)
\coloneqq
\IHDR_\alpha\bigl(\CRED(y^n,\psi)\bigr).
\]
By Proposition \ref{prop-ihdr-cpr}, for every \((y^n,\psi)\in X\),
\[
\IHDR_\alpha\bigl(\CRED(y^n,\psi)\bigr)
=
\{y\in Y:\pi_\psi(y,y^n)>\alpha\}
=
\kappa_\alpha(y^n,\psi).
\]
Hence
$\eta_\alpha=\kappa_\alpha$.
Since \(\kappa_\alpha\) is upper hemicontinuous, so is \(\eta_\alpha\). Thus
\(\eta_\alpha\) is also a morphism in \(\mathbf{UHCont}\).

It remains to justify the displayed categorical factorization. Let
\[
\mathscr M_{\mathrm{CP}}
=
\{\CRED(y^n,\psi):(y^n,\psi)\in X\}
\]
and endow \(\mathscr M_{\mathrm{CP}}\) with the final topology induced by
\[
c:X\to\mathscr M_{\mathrm{CP}},
\qquad
c(y^n,\psi)=\CRED(y^n,\psi).
\]
By construction, \(c\) is continuous. Therefore, viewed as the singleton-valued
correspondence
$x\mapsto\{c(x)\}$,
it is upper hemicontinuous, hence a morphism in \(\mathbf{UHCont}\).
Define
\[
\widehat{\IHDR}_\alpha:\mathscr M_{\mathrm{CP}}\rightrightarrows Y,
\qquad
\widehat{\IHDR}_\alpha(\mathcal M)=\IHDR_\alpha(\mathcal M).
\]
This correspondence is well defined on \(\mathscr M_{\mathrm{CP}}\). Indeed,
each \(\mathcal M\in\mathscr M_{\mathrm{CP}}\) is of the form \(\mathcal M=\CRED(y^n,\psi)\)
for some \((y^n,\psi)\in X\), and the value
\(\IHDR_\alpha(\mathcal M)\) depends only on the credal set \(\mathcal M\), not on the particular
representation.

We show that \(\widehat{\IHDR}_\alpha\) is upper hemicontinuous. Let
\(V\subseteq Y\) be open and set
\[
O_V
\coloneqq
\{\mathcal M\in\mathscr M_{\mathrm{CP}}:
\widehat{\IHDR}_\alpha(\mathcal M)\subseteq V\}.
\]
By the definition of the final topology, it is enough to show that
\(c^{-1}(O_V)\) is open in \(X\). But
\[
c^{-1}(O_V)
=
\left\{
(y^n,\psi)\in X:
\IHDR_\alpha(\CRED(y^n,\psi))\subseteq V
\right\}.
\]
Using \(\IHDR_\alpha(\CRED(y^n,\psi))=\kappa_\alpha(y^n,\psi)\), this becomes
\[
c^{-1}(O_V)
=
\left\{
(y^n,\psi)\in X:
\kappa_\alpha(y^n,\psi)\subseteq V
\right\},
\]
which is open because \(\kappa_\alpha\) is upper hemicontinuous. Therefore
\(O_V\) is open in \(\mathscr M_{\mathrm{CP}}\), and
\(\widehat{\IHDR}_\alpha\) is upper hemicontinuous.

Finally, for every \((y^n,\psi)\in X\),
\[
(\widehat{\IHDR}_\alpha\circ c)(y^n,\psi)
=
\widehat{\IHDR}_\alpha(\CRED(y^n,\psi))
=
\IHDR_\alpha(\CRED(y^n,\psi))
=
\kappa_\alpha(y^n,\psi).
\]
Thus the diagram commutes in \(\mathbf{UHCont}\).
\end{proof}

Four remarks are in order.

\begin{remark}[Correspondence form versus functional form]
\label{rem:correspondence-functional-form}
There are two related, but distinct, ways to view the credal construction.
First, one may view it as a correspondence
\[
\Gamma:X\rightrightarrows \Prob,
\qquad
\Gamma(y^n,\psi)=\CRED(y^n,\psi).
\]
This is the form in which \(\CRED(y^n,\psi)\) is treated as a subset of
\(\Prob\), and Theorem \ref{conf-pred-diagr} shows that this correspondence is
upper hemicontinuous.

Second, one may view it as a credal-object-valued map
\[
c:X\to\mathscr M_{\mathrm{CP}},
\qquad
c(y^n,\psi)=\CRED(y^n,\psi),
\]
where \(\mathscr M_{\mathrm{CP}}\) is the image of the construction. The
commuting diagram in Theorem \ref{conf-pred-diagr} uses this second
interpretation, with \(\mathscr M_{\mathrm{CP}}\) endowed with the final
topology induced by \(c\).

Thus the diagram should not be read as asserting that the credal-object-valued
map \(X\to\mathscr C\) is continuous for the ambient Vietoris topology on
\(\mathscr C\). That is a different, stronger statement and may require
additional assumptions.
\end{remark}

\begin{remark}[Why the target is \(Y\), not \(\Sigma_Y\)]
\label{rem:targetY-not-SigmaY}
Although \(\kappa_\alpha(y^n,\psi)\) and
\(\IHDR_\alpha(\CRED(y^n,\psi))\) are measurable subsets of \(Y\), as
morphisms in \(\mathbf{UHCont}\) they are correspondences with target object
\(Y\):
\[
\kappa_\alpha:X\rightrightarrows Y,
\qquad
\widehat{\IHDR}_\alpha:\mathscr M_{\mathrm{CP}}\rightrightarrows Y.
\]
Writing the output as an element of \(\Sigma_Y\) would instead amount to
working with single-valued maps into a hyperspace of measurable sets, equipped
with some additional topology. That is not the categorical structure used here.
\end{remark}

\begin{remark}[Role of the no-tie condition]
\label{rem:no-tie-role}
The restriction \(\alpha\notin S_{n+1}\) is used to rewrite the strict
conformal region
$\{y\in Y:\pi_\psi(y,y^n)>\alpha\}$
as the closed superlevel set
$\{y\in Y:\pi_\psi(y,y^n)\ge \beta(\alpha)\}$.
This makes the closed-graph argument immediate. The assumption is convenient
rather than conceptually essential; with minor endpoint conventions, the same
argument can often be extended to fixed grid values of \(\alpha\).
\end{remark}

\begin{remark}[Compactness of \(Y\)]
\label{rem:compactnessY}
The compactness assumption on \(Y\) is a technical device that lets us work in
a clean topological setting. It ensures, among other things, that closed subsets
of \(Y\) are compact and that upper semicontinuous functions attain their
suprema on closed sets. This is used in Lemma \ref{lem-H4} to obtain upper
semicontinuity of the upper envelope
\[
(y^n,\psi)\mapsto
\overline\Pi_{(y^n,\psi)}(A)
=
\sup_{y\in A}\pi_\psi(y,y^n),
\]
for closed \(A\subseteq Y\).
On noncompact spaces, analogous conclusions may still hold under assumptions
that prevent loss of mass or nonattainment of suprema, for example suitable
compactness of upper level sets or other properness conditions. We do not
pursue those refinements here.
\end{remark}

We now give an example in which conditions \ref{H1} and \ref{H2} are satisfied. Condition \ref{H3} is enforced by choosing
\(\alpha\notin S_{n+1}\), and \ref{H5} belongs to the ambient
functional-analytic setup rather than to the choice of non-conformity measure.

\begin{example}[A Statistical ML-Inspired Setting]
\label{ex:neural-psi}
Let $Y=[0,1]^d$ with its Euclidean topology (compact and Hausdorff),
fix a parameter value \(\theta\), and let
\(\varphi_\theta\colon Y\to\mathbb{R}^m\) be a feed-forward neural
network whose activation functions are continuous, such as ReLU. Given that we fixed
\(\theta\), the map \(y\mapsto\varphi_\theta(y)\) is continuous.\footnote{If \(\theta\) is itself chosen as a function of the data, additional assumptions
on the training map \(y^n\mapsto\theta(y^n)\) are needed to preserve joint
continuity of the resulting non-conformity measure.}

\textbf{Non-Conformity Measure.}
For a training set $y^n=(y_1,\dots,y_n)\in Y^n$ and a candidate point
$y\in Y$, define the non-conformity measure
\begin{align}\label{neur-net}
    \psi_\theta(y^{n},y)
   = 
  \biggl\|
        \varphi_\theta(y)
         - 
        \frac1n\sum_{i=1}^n\varphi_\theta(y_i)
    \biggr\|_2^{2}.
\end{align}

It satisfies the following properties,

\begin{itemize}
  \item \emph{Permutation invariance:} The sample mean is symmetric in
        $y_1,\dots,y_n$.
  \item \emph{Joint continuity:} Composition, finite sums,
        Euclidean norm, and squaring preserve continuity;
        Thus $(y^{n},y)\mapsto\psi_\theta(y^{n},y)$ is jointly
        continuous on $Y^{n} \times Y$.
\end{itemize}



Hence, $\psi_\theta\in\mathscr{F}\subset C(Y^n \times Y)$. Consonance does not automatically follow from the continuity of \(\psi_\theta\). Under
compactness, upper semicontinuity ensures that the supremum of
\(y\mapsto\pi(y,y^n)\) is attained, but the attained value need not be equal
to \(1\); Thus, consonance should be verified separately when the
possibility/credal set interpretation is required. One possible way to enforce it is to normalize the transducer by setting
\[
\widetilde\pi(y,y^n)
=
\frac{\pi(y,y^n)}{\sup_{z\in Y}\pi(z,y^n)}
\]
whenever the denominator is positive \citep[Section 7]{cella2021valid}. This gives
$\sup_{y\in Y}\widetilde\pi(y,y^n)=1$.
However, the normalized region satisfies
\[
\{y\in Y:\widetilde\pi(y,y^n)>\alpha\}
=
\left\{
y\in Y:
\pi(y,y^n)>
\alpha\sup_{z\in Y}\pi(z,y^n)
\right\},
\]
so normalization generally changes the fixed-\(\alpha\) conformal region. We
therefore treat normalization as a separate device for obtaining a consonant
possibility representation, not as part of the basic conformal construction.

\textbf{Statistical Machine Learning Relevance.} This example bears the following implications for (Statistical) Machine Learning.
\begin{enumerate}
  \item \emph{Prototype learning:}
        \eqref{neur-net} is the squared Euclidean distance, in the representation
        space, from the candidate point to the empirical prototype
        \(\frac1n\sum_{i=1}^n\varphi_\theta(y_i)\).

  \item \emph{Out-of-Distribution (OoD) detection:}
        The same distance can be used as a prototype distance OoD score in a
        learned embedding space.

  \item \emph{Representation-based conformal prediction:}
        If \(\varphi_\theta\) is a fixed feature map, for instance the
        penultimate layer of a trained network, then \eqref{neur-net} yields a
        continuous non-conformity score suitable for the topological framework
        of Theorem \ref{conf-pred-diagr}.
\end{enumerate}
As we can see, this choice of $\psi_\theta$ simultaneously fulfills joint continuity
and practical usefulness in modern ML workflows. \hfill $\diamond$
\end{example}

More generally, to ensure that the hypotheses of Theorem
\ref{conf-pred-diagr} are satisfied, it is enough for the practitioner to
\begin{itemize}
    \item work on a compact Hausdorff outcome space \(Y\), or restrict attention
    to a compact prediction domain, so that \ref{H1} is met;

    \item choose a non-conformity measure \(\psi\) that is invariant to
    permutations in its first argument and jointly continuous in its arguments,
    so that \ref{H2} is met;

    \item choose \(\alpha\notin S_{n+1}\), so that \ref{H3} is met;

    \item verify consonance, or explicitly state the normalization or tie-breaking device
used to enforce it when the consonant possibility/credal set interpretation is required.
\end{itemize}

\subsection{Implications}

Theorem \ref{conf-pred-diagr} clarifies a subtle but important point about what is (and is not) intrinsic in Conformal Prediction.
Conformal Prediction Regions provide a \emph{calibrated representation} of predictive uncertainty (via set membership at level $\alpha$),
yet the region alone does not come equipped with a canonical \emph{cardinal} scale of ``amount of uncertainty''.
The reason is that the CP region depends only on the \emph{ordering} of the non-conformity measures: For any strictly increasing transform $f$ (applied pointwise) and any $y^n \in Y^n$,
\[
\kappa_\alpha(y^n,\psi)=\kappa_\alpha(y^n,f\circ\psi).
\]
Consequently, numerical summaries of a region---such as its diameter $\mathrm{diam}(\mathscr R)$ or its volume---are only meaningful \emph{relative} to
additional analyst choices (a metric or reference measure on $Y$).
They do not by themselves define an intrinsic unit in which uncertainty can be compared across problems or across modeling choices \citep{pmlr-v216-sale23a}.
In particular,
\[
\mathrm{diam}(\mathscr{R}_1)=3 \mathrm{diam}(\mathscr{R}_2) \not\Rightarrow\
\text{``$\mathscr{R}_1$ is three times more uncertain than $\mathscr{R}_2$''},
\]
unless one commits to a geometry on $Y$ and to an interpretation of this geometry as a utility-relevant scale.

The commuting factorization $\kappa_\alpha=\IHDR_\alpha\circ\CRED$ reveals that CP nevertheless carries a richer, implicitly probabilistic layer.
Indeed, computing a CP region is equivalent to (i) forming the credal set $\mathcal{M}(\overline{\Pi})$ and (ii) extracting its IHDR.
Because $\mathcal{M}(\overline{\Pi})$ lives in a space of probability measures, one can attach to it uncertainty functionals with clear operational
meaning, such as lower/upper probabilities of events, robust (lower/upper) expected losses for actions, and metric-type summaries of
imprecision  \citep{abellan,Hullermeier2021Aleatoric,alireza2,hofman2024quantifyingaleatoricepistemicuncertainty,chau2025integralimpreciseprobabilitymetrics}.
These are \emph{cardinal} in the sense that they support meaningful arithmetic once a decision problem (loss) or a probability metric is specified, see Example \ref{ex:uncertainty-beyond-size} inspired by \citet{martin2025decisionmakingpossibilisticinferentialmodels}.
In this paper the focus is on the marginal (exchangeability-based) setting, and the resulting summaries naturally reflect the marginal nature of
the CP guarantee---see equation \eqref{eq_imp6}.

\begin{example}[Uncertainty Beyond Size: Same CPR Set Size, Different Robust Conclusions]
\label{ex:uncertainty-beyond-size}
Let $Y=\{1,2,3\}$ be a finite label space. Suppose a (normalized, thus satisfying consonance) conformal transducer
$\pi(\cdot,y^n):Y\to[0,1]$ is available for a fixed training sample $y^n$ and choice of non-conformity measure $\psi$. For ease of exposition, the numerical values below should be read as an illustrative possibility profile; In a finite-sample Full CP construction they would lie on the grid \(S_{n+1}=\{0,1/(n+1),\ldots,1\}\), or be approximated by grid-compatible values.
Define the induced possibility function
\(
\overline \Pi(A)=\sup_{y\in A}\pi(y,y^n),\text{ } A\subseteq Y,
\)
and the associated credal set
\[
\mathcal M(\overline\Pi)=\bigl\{P: P(A)\le \overline\Pi(A), \quad   \forall A\subseteq Y\bigr\}.
\]
The conjugate lower probability is given by $\underline \Pi(A)=1-\overline\Pi(A^c)$, for all $A\subseteq Y$.
For a level $\alpha\in(0,1)$, the conformal prediction set is the $\alpha$-cut
\(
\kappa_\alpha(y^n,\psi)=\{y\in Y: \pi(y,y^n)>\alpha\}.
\)

\textbf{Case A (One Label Robustly Dominant).}
Let
\[
\pi(1,y^n)=1,\qquad \pi(2,y^n)=0.12,\qquad \pi(3,y^n)=0.08,
\]
and take $\alpha=0.10$. Then, $\kappa_\alpha(y^n,\psi)=\{1,2\}$, so $|\kappa_\alpha(y^n,\psi)|=2$.
For the event $A=\{1\}$,
\[
\underline \Pi(\{1\})
=1-\overline \Pi(\{2,3\})
=1-\max\{0.12,0.08\}
=0.88,
\qquad
\overline \Pi(\{1\})=1,
\]
hence $P(Y=1)\in[0.88,1]$ for all $P\in \mathcal M(\overline \Pi)$.
If a single-label action $a\in Y$ is assessed by $0$-$1$ loss
$L(a,y)=\mathbbm 1[\{a\neq y\}]$, then the worst-case risk of predicting label $1$ is
\[
\sup_{P\in \mathcal M(\overline \Pi)}\mathbb E_P[L(1,Y)]
=1-\inf_{P\in \mathcal M(\overline \Pi)}P(\{1\})
=1-\underline \Pi(\{1\})
=0.12.
\]

\textbf{Case B (Same CP Set Size, but Substantially More Ambiguity).}
Now let
\[
\pi'(1,y^n)=1,\qquad \pi'(2,y^n)=0.55,\qquad \pi'(3,y^n)=0.08,
\]
with the same $\alpha=0.10$. Again $\kappa'_\alpha(y^n,\psi)=\{1,2\}$, so $|\kappa'_\alpha(y^n,\psi)|=2$.
However,
\[
\underline \Pi'(\{1\})
=1-\overline\Pi'(\{2,3\})
=1-\max\{0.55,0.08\}
=0.45,
\qquad
\overline \Pi'(\{1\})=1,
\]
so $P(Y=1)\in[0.45,1]$ for all $P\in \mathcal M(\overline \Pi')$, and the worst-case $0$-$1$ risk of
predicting label $1$ is
\[
\sup_{P\in \mathcal M(\overline \Pi')}\mathbb E_P[L(1,Y)]
=1-\underline \Pi'(\{1\})
=0.55.
\]

Both cases yield conformal prediction sets of the same size,
$|\kappa_\alpha(y^n,\psi)|=|\kappa'_\alpha(y^n,\psi)|=2$, yet the credal layer distinguishes them sharply
through lower/upper probabilities and robust expected losses. This illustrates how
the credal representation captures \emph{uncertainty beyond size}. \hfill $\diamond$
\end{example}

Furthermore, the decomposition through $\mathcal{M}(\overline{\Pi})$ suggests a principled route to separating different sources of predictive
uncertainty \citep{cabezas2025epistemicuncertaintyconformalscores,10.1145/3649329.3663512,hanselle2025conformal,chau2026quantifyingepistemicpredictiveuncertainty}. As a working intuition, analyst-controlled choices---the significance level $\alpha$, which fixes the tolerated error, the score $\psi\in\mathscr F$ and, more generally, the modeling constraints
implicit in $\mathscr F$---tend to affect the \emph{reducible} component of uncertainty (insofar as they shrink or enlarge the set of admissible
predictive laws), whereas irreducible uncertainty is tied to the intrinsic variability of the exchangeable process encoded in $\mathfrak P$.
A precise disentanglement, however, depends on the choice of uncertainty functional and on additional structure, and we therefore treat this point
as interpretative rather than universal.

Thus, Conformal Prediction’s capacity for predictive uncertainty quantification is not an add-on, but a consequence of the categorical
factorization exhibited by Theorem \ref{conf-pred-diagr}. This perspective expands the imprecise probabilistic route to UQ developed in \citet{caprio2025conformalpredictionregionsimprecise} and 
\citet{cella2022validity,cella2023possibility} by showing that the credal layer is already present in CP through the commuting diagram. An in-depth study of some of the properties of $\mathbf{UHCont}$ is carried out in Appendix C in the Supplementary Material.

\textbf{On the meaning of the induced credal set.}
 Since $\mathcal M(\overline\Pi)$ is constructed mechanically from the conformal transducer, one may ask what it means as an uncertainty object.
We emphasize that $\mathcal M(\overline\Pi)$ is generally \emph{not} intended to represent subjective degrees of belief or analyst-specified uncertainty.
Rather, it is an \emph{inferential} (data-driven) imprecise predictive model: It collects the probability measures on $Y$ that are setwise dominated by the
upper probability $\overline\Pi$ induced by the conformal transducer.
In this sense, $\mathcal M(\overline\Pi)$ summarizes the information that CP produces about future outcomes, but in a probabilistic language that supports
decision-relevant functionals.
This semantics is different in spirit from classical credal sets motivated by prior ignorance, even though the mathematical object is the same.

\subsection{A Category Dealing With Measurability}

 We now briefly introduce structure $\mathbf{WMeas}_{\text{uc}}$, which is explored in depth in Appendix A in the Supplementary Material. There, we prove that it is indeed a category, and that the Full CP Diagram commutes in $\mathbf{WMeas}_{\text{uc}}$ as well, under essentially the same conditions of Theorem \ref{conf-pred-diagr}.

\begin{definition}[Structure $\mathbf{WMeas}_{\text{uc}}$]
\label{meas-corr-def}
Define a structure $\mathbf{WMeas}_{\text{uc}}$ as follows,
\begin{itemize}
  \item \textbf{Objects:} Compact Polish spaces $(X,\Sigma_X)$, with $\Sigma_X=\mathcal{B}(X)$, the Borel $\sigma$-algebra.
  \item \textbf{Morphisms:} Weakly measurable (see Def. 14 in Appendix D in the Supplementary Material), uniformly compact-valued (see Def. 15 in Appendix D in the Supplementary Material) 
        correspondences $\Phi: X\rightrightarrows Y$.
\end{itemize}
\end{definition}


The fact that the Full Conformal Prediction Diagram commutes in both $\mathbf{UHCont}$ and $\mathbf{WMeas}_{\text{uc}}$
highlights that, under our assumptions, CP admits a coherent set-valued interpretation that is simultaneously
\emph{stable} and \emph{probabilistically well-posed}. The $\mathbf{UHCont}$ result provides a robustness guarantee for the
prediction correspondence (upper hemicontinuity), while the $\mathbf{WMeas}_{\text{uc}}$ result ensures weak measurability and
compact-valuedness, so that CP defines a random compact set and supports the study of probabilistic functionals of the region.
This joint structure is precisely what makes the credal set factorization operational: It enables uncertainty summaries with
clear probabilistic and decision theoretic meaning, in addition to the usual coverage guarantee.

\section{Unifying Bayesian, Frequentist, and Imprecise Predictive Reasoning}\label{unifying}




In this section, we connect the categorical factorization of conformal prediction to Bayesian prediction.
While Bayesian prediction regions (e.g. posterior predictive density level sets) are not, in general, frequentist-calibrated at finite sample sizes,
they provide a natural reference notion of ``high probability'' predictive content under a model.
Our goal is to show that when the conformal score is chosen to reflect the Bayesian posterior predictive density and the predictive surface is stable,
conformal prediction acts as a finite-sample validity wrapper around this Bayesian target: It preserves distribution-free coverage for every $n$,
while becoming asymptotically equivalent to the Bayesian prediction region. In the limit, the Bayesian density-level route, the direct conformal route, and
the indirect credal-IHDR route therefore become equivalent as predictive
procedures, not merely close as sets.

Formally, we establish an asymptotic commutativity result in $\mathbf{UHCont}$ for a diagram connecting Bayesian, conformal, and imprecise probabilistic prediction procedures.
We then quantify the Hausdorff convergence rate under local empirical process and boundary regularity assumptions, and identify conditions under which upper posterior constructions are related to e-posteriors.
Since the diagram extends the Full CP Diagram of Theorem \ref{conf-pred-diagr}, we retain assumptions \ref{H1}-\ref{H5}.

We mention in passing that connections between Bayesian and Conformal Prediction methods have been discussed previously, see e.g. \citet{Hoff2023BayesOptimal}.
More broadly, links between Bayesian, frequentist, and imprecise probabilistic viewpoints in statistical inference have been developed in a series of papers by
\citet{martin2021impreciseprobabilisticcharacterizationfrequentiststatistical,
martin2023validefficientimpreciseprobabilisticinference,
martin2025nopriorbayesreimaginedprobabilistic}.
Our contribution is complementary: We approach these connections through a categorical lens, expressing them as 
commutative diagrams of set-valued statistical functionals.

\subsection{Setup}
Let $(\Theta,\Sigma_\Theta)$ be a measurable parameter space, and call $\lambda_\Theta$ a fixed $\sigma$-finite Borel measure on $(\Theta,\Sigma_\Theta)$. Call $\Delta_\Theta^{\mathrm{dens}}$
the space of finitely additive probability measures on $\Theta$ admitting a density with respect to $\lambda_\Theta$ \citep{bell1988separation}; 
We impose this requirement---which effectively restricts our attention to countably additive probability measures---to align with Section \ref{main}. There, we take $\Delta_Y$ to be the space of finitely additive probability measures on $Y$, following the convention in IP theory and ensuring maximal generality of our results.


Endow $\Delta_\Theta^{\mathrm{dens}}$ with the $L^{1}(\lambda_\Theta)$ topology. Denote by $\Delta_Y^\text{dens}$ the space of finitely additive probability measures on $Y$ that admit a density with respect to a $\sigma$-finite dominating (atomless) measure $\lambda_Y$.
Consider now the following conditions,

\begin{enumerate}
\renewcommand{\labelenumi}{(\roman{enumi})}
  \renewcommand{\theenumi}{(\roman{enumi})}
  \setcounter{enumi}{4}
  \item\label{H6} \emph{Measurability.}
There exist de Finetti kernels $k_\iota:\Theta\times Y^\iota\to\mathbb{R}_+$, $\iota\in\{n,n+1\}$,
such that, under parameter $\theta$, the joint density of $y^\iota$ is
$k_\iota(\theta,y^\iota)$ (a density on $Y^\iota$ with respect to $\lambda_Y^{\otimes \iota}$).\footnote{Since \(k_n\) and \(k_{n+1}\) are joint densities belonging to the same
dominated predictive system, we implicitly assume the usual consistency relation
$\int_Y k_{n+1}(\theta,(y^n,y)) \lambda_Y(\mathrm dy)
=
k_n(\theta,y^n)$
for \(\lambda_\Theta\)-almost every \(\theta\) and relevant \(y^n\).}
For each $\iota\in\{n,n+1\}$ and each $y^\iota\in Y^\iota$, $\theta\mapsto k_\iota(\theta,y^\iota)$ is
$\Sigma_\Theta$-measurable, and for each $\theta\in\Theta$, $y^\iota\mapsto k_\iota(\theta,y^\iota)$ is Borel.
The posterior predictive is
\[
  p_P(y_{n+1}\mid y^n)=
  \frac{\int_\Theta k_{n+1}(\theta,(y^n,y_{n+1})) P(\mathrm d\theta)}
       {\int_\Theta k_n(\theta,y^n) P(\mathrm d\theta)},
\quad P\in\Delta_\Theta^{\mathrm{dens}}.
\]
  \item\label{H7} \emph{Pointwise $y$-continuity a.e.}
        For $\lambda_\Theta$-almost all $\theta$, the maps
        $y^\iota\mapsto k_\iota(\theta,y^\iota)$ are continuous on $Y^\iota$, for $\iota \in\{n,n+1\}$.
  \item\label{H8} \emph{Integrable envelope.}
        There exists $M(\theta) \equiv M\in L^\infty(\lambda_\Theta)$ with
        \(k_\iota(\theta,y^\iota) \le M(\theta)\) for every $y^\iota\in Y^\iota$, $\iota \in\{n,n+1\}$, and $\lambda_\Theta$-almost all $\theta$.
  \item\label{H10} \emph{Positive marginal likelihood.}
        For every $y^{n}\in Y^{n}$ and every admissible prior probability measure $P\in \Delta_\Theta^{\mathrm{dens}}$,
        \[
          D_P(y^{n})
            \coloneqq \int_\Theta
                 k_n(\theta,y^n)
                 P(\text{d}\theta) = \int_\Theta
                 k_n(\theta,y^n)
                 p(\theta)\lambda_\Theta(\text{d}\theta)
             > 0.
        \]
\end{enumerate}

For $P\in\Delta_\Theta^{\mathrm{dens}}$ having density $p$, and data $y^{n}\in Y^{n}$, set
\[
  p_P\left(y_{n+1}\mid y^{n}\right)
     \coloneqq
     \frac{N_P(y_{n+1},y^{n})}{D_P(y^{n})},
\]
where $N_P(y_{n+1},y^{n})\coloneqq\int_\Theta k_{n+1}(\theta,(y^n,y_{n+1}))
p(\theta)\lambda_\Theta(\text{d}\theta)$. 
By \ref{H10}, it is well-defined.

Because $y\mapsto p_P(y\mid y^{n})$ is continuous (as a result of Lemma \ref{bayes-cont} below) and
$Y^{n+1}$ is compact by \ref{H1}, the natural codomain for such a map is the Banach space $C(Y^n \times Y)$ with norm $\|f\|_\infty\coloneqq\sup_{(y^n,y)\in Y^n \times Y}|f(y^n,y)|$.

Define
\begin{equation}\label{bcp-eq-def}
    \mathsf{BCP}: 
  Y^{n}\times\Delta_\Theta^{\mathrm{dens}}
  \rightarrow
  Y^{n}\times C(Y^n \times Y),
\quad
  \mathsf{BCP}\bigl(y^{n},P\bigr)
     \coloneqq\bigl(y^{n},\psi_{(y^{n},P)}\bigr),
\end{equation}
where
\[
  \psi_{(y^{n},P)}(y_{n+1})
      \coloneqq- p_P(y_{n+1}\mid y^{n}).
\]
Here we abuse notation, and write $\psi_{(y^{n},P)}(y_{n+1}) \equiv \psi_{P}(y^{n},y_{n+1})$ to highlight the dependence of the non-conformity measure on the observed data $y^n$. Formally, for each \(P\), \(\psi_P\) denotes the function on
\(Y^n\times Y\) given by \((y^n,y)\mapsto -p_P(y\mid y^n)\); The notation above
records its evaluation at the observed sample \(y^n\).

For \(\psi_P\) to belong to \(\mathscr F\), the map
$y^n\mapsto p_P(y\mid y^n)$
must be invariant under permutations of \(y^n\), for every fixed \(y\in Y\).
This holds, for instance, when the de Finetti kernels \(k_n\) and \(k_{n+1}\)
are symmetric in the observed coordinates \(y_1,\ldots,y_n\), as is the case
for conditionally i.i.d. models.

We call \(\mathsf{BCP}\) the function in \eqref{bcp-eq-def}, an acronym for
Bayesian Conformal Prediction \citep{fong-bayes,snell2025conformal,wu2026bayesianconformalpredictiondecision}.

In practice, conditions \ref{H6}-\ref{H10} hold in many standard dominated Bayesian models. First, in standard dominated parametric families (Gaussian, exponential families,
finite mixtures with continuous component densities), the likelihood kernel
$f_\theta(y)$ is continuous in $y$, and the joint density
$k_n(\theta,y^n)=\prod_{i=1}^n f_\theta(y_i)$ is Borel in $y^n$ and
$\Sigma_\Theta$-measurable in $\theta$, so \ref{H6} holds.

Condition \ref{H7} ensures that, as the data $y^n$ vary, the integrands in $N_P$
and $D_P$ vary continuously (for $\lambda_\Theta$-almost all $\theta$), so dominated
convergence applies. In practice, standard kernels (Normal, Laplace, Student-$t$,
etc.) are smooth in $y$.

The uniform-in-\(y\) bound in \ref{H8} is a clean sufficient domination condition guaranteeing that limits can be interchanged with integration uniformly over \((y^n,y)\),
which is what yields continuity of \(\mathsf{BCP}\) in the sup-norm.
Because \(Y\) is compact by \ref{H1} and \(k_\iota(\theta,\cdot)\) is continuous, one can define
\[
M_n(\theta)\coloneqq \sup_{y^n\in Y^n} k_n(\theta,y^n),\qquad
M_{n+1}(\theta)\coloneqq \sup_{y^{n+1}\in Y^{n+1}} k_{n+1}(\theta,y^{n+1}).
\]
Assumption \ref{H8} requires these suprema to be dominated by an essentially bounded function of \(\theta\). This is automatic in bounded kernel models, or after restricting the parameter space so that the kernels are uniformly bounded. It may fail in models with unbounded likelihoods or noncompact scale parameters, in which case one should replace \ref{H8} by a prior-dependent integrable-envelope condition and use a compatible weighted topology.\footnote{For example, one may ask that \(M\in L^1\) and topologize \(\Delta_\Theta^{\mathrm{dens}}\) with the weighted topology \(L^1(M \mathrm d\lambda_\Theta)\).}

We remark that Assumption \ref{H8} can be relaxed in standard ways, for instance by requiring an integrable envelope under the relevant priors,
\[
\int_\Theta M(\theta) P(\mathrm d\theta)<\infty,
\]
together with a topology on $\Delta_\Theta^{\mathrm{dens}}$ compatible with this envelope (e.g. a weighted $L^1$ topology).
We adopt \ref{H8} for simplicity and to keep the analytic argument transparent. 

Since $N_P$ involves $k_{n+1}$, the prior integrability condition
$\int M(\theta) P(\mathrm d\theta)<\infty$ guarantees finiteness and continuity
of the predictive. This is usually automatic with common priors when the
envelope $M$ has at most moderate growth in $\theta$ (e.g.  exponential family
models with conjugate priors).

Finally, \ref{H10} is an admissibility condition that rules out division by zero in
$p_P(y_{n+1}\mid y^n)=N_P(y_{n+1},y^n)/D_P(y^n)$.
It requires the prior \(P\) to assign positive mass to parameter values for which
the observed data have positive likelihood; Equivalently,
$P(\{\theta\in\Theta:k_n(\theta,y^n)>0\})>0$.
This condition is automatic when \(k_n(\theta,y^n)>0\) for all \(\theta\) in the
support of every admissible prior under consideration, but in general it should be
understood as part of the definition of the admissible prior class.


In short, \ref{H6}-\ref{H10} are sufficient (and close to necessary in general) to ensure the posterior predictive is well defined
and depends continuously on both the data $y^n$ and the prior $P$; They are met in a wide range of practical models.

Before the main result of this section, we introduce two further conditions. The first concerns the distribution function induced by the posterior predictive density levels.
Recall that we denote by $\Delta_Y^\text{dens}$ the space of finitely additive probability measures on $Y$ that admit a density with respect to a $\sigma$-finite dominating (atomless) measure $\lambda_Y$.

\begin{enumerate}
\renewcommand{\labelenumi}{(\roman{enumi})}
  \renewcommand{\theenumi}{(\roman{enumi})}
  \setcounter{enumi}{8}
  \item\label{H11} \emph{Uniform well-behaved density-level quantile.}
Fix \(\alpha\in(0,1)\). For each \((y^n,P)\in
Y^n\times\Delta_\Theta^{\mathrm{dens}}\), define
\[
F(c;y^n,P)
\coloneqq
\int_{\{y:p_P(y\mid y^n)\le c\}}
p_P(y\mid y^n)\lambda_Y(\mathrm dy),
\qquad
c(\alpha,y^n,P)
\coloneqq
\inf\{c:F(c;y^n,P)\ge\alpha\}.
\]
For every \((y^n,P)\), the map \(c\mapsto F(c;y^n,P)\) is continuous and
strictly increasing in a neighborhood of \(c(\alpha,y^n,P)\), and the
\(\alpha\)-quantile \(c(\alpha,y^n,P)\) is locally unique. Moreover, the map
$(y^n,P)\mapsto c(\alpha,y^n,P)$
is continuous.

Finally, for each fixed \(P\in\Delta_\Theta^{\mathrm{dens}}\) to which the
asymptotic theorem is applied, let \(c_n\coloneqq c(\alpha,y^n,P)\). There exist
constants \(\delta_0>0\), \(L>0\), and events \(E_n\) with
\(\mathfrak P(E_n)\to1\), such that on \(E_n\), for all \(t\in[0,\delta_0]\),
\[
F(c_n+t;y^n,P)\ge \alpha+Lt,
\qquad
F(c_n-t;y^n,P)\le \alpha-Lt.
\]
\end{enumerate}



Let us add a brief discussion on condition \ref{H11}. This condition has two
roles. First, it ensures that the density level quantile
\(c(\alpha,y^n,P)\) is locally well defined and continuous in \((y^n,P)\).
Second, it imposes a local margin condition around the target level, ruling out
cases in which the density level distribution \(F(\cdot;y^n,P)\) becomes nearly
flat near \(c(\alpha,y^n,P)\).

In many continuous Bayesian models, the posterior predictive density
\(y\mapsto p_P(y\mid y^n)\) is continuous and strictly positive on compact \(Y\).
Then, at any level \(c\) such that
\(\lambda_Y(\{y:p_P(y\mid y^n)=c\})=0\), the map 
\[
F(c;y^n,P)
=
\int_{\{p_P(\cdot\mid y^n)\le c\}}
p_P(y\mid y^n)\lambda_Y(\mathrm dy)
\]
is continuous at \(c\). Moreover, \(F\) is strictly increasing locally around
\(c_0\), whenever the predictive density takes values on both sides of \(c_0\)
with positive \(\lambda_Y\)-mass; For instance, when there exists
\(\varepsilon>0\) such that
\[
\lambda_Y(\{c_0-\varepsilon<p_P(y\mid y^n)\le c_0\})>0,
\qquad
\lambda_Y(\{c_0<p_P(y\mid y^n)\le c_0+\varepsilon\})>0.
\]
If this no-plateau/no-gap behavior persists locally for
\((\tilde y^n,\tilde P)\) near \((y^n,P)\), then the local uniqueness and
continuity of \(c(\alpha,y^n,P)\) are natural. The margin part of \ref{H11},
however, is stronger: It requires that, along the random sequence
\(c_n=c(\alpha,y^n,P)\), the increase of \(F\) around \(c_n\) is bounded below
linearly with high probability. This rules out sequences of problems in which
\(F\) remains strictly increasing but its local slope degenerates as
\(n\to\infty\).

Condition \ref{H11} may fail if: (a) \(p_P\) has a plateau of positive
\(\lambda_Y\)-measure at the target level, so \(F\) is locally flat; (b) the posterior predictive law is not absolutely continuous with respect to \(\lambda_Y\), so atoms
create jumps of the density level distribution; (c) the predictive density has a
gap in its range near \(c(\alpha,y^n,P)\), for instance on disconnected outcome
spaces; or (d) the local slope of \(F\) near \(c_n\) tends to zero along the
sequence, violating the uniform margin requirement.



The last condition that we need is the following.

\begin{enumerate}
\renewcommand{\labelenumi}{(\roman{enumi})}
  \renewcommand{\theenumi}{(\roman{enumi})}
  \setcounter{enumi}{9}
  \item\label{H12} \emph{Leave-One-Out (LOO) stability and conditional Glivenko-Cantelli (GC).}
Recall that $\mathfrak{P}$ denotes the sampling law of the exchangeable process $(\mathcal Y_i)_{i\ge1}$. For $(y^n,P)$ and $y\in Y$, define
\[
r_n(y) \coloneqq  p_P(y\mid y^n),\qquad
r_{n,-i}(y) \coloneqq  p_P(y\mid y^{n+1}_{-i}),\quad i \in \{1,\ldots,n\}.
\]

\begin{enumerate}
\item[(a)] \emph{LOO stability.}
\[
\mathsf{err}_n(y^n,P) \coloneqq 
\sup_{y\in Y} \max_{i \in \{1,\ldots,n\}} 
\bigl| r_{n,-i}(y)-r_n(y) \bigr| = o_{\mathfrak{P}}(1).
\]
\item[(b)] \emph{Conditional GC.}
Let
\[
F_n(t;y^n,P) \coloneqq \frac{1}{n}\sum_{i=1}^n
\mathbbm{1}[\{ r_n(y_i)\le t\}]
\quad\text{and}\quad
F(\cdot ;y^n,P) \text{ as in \ref{H11}.}
\]
We have
\[
\sup_{t\in\mathbb{R}}\bigl|F_n(t;y^n,P)-F(t;y^n,P)\bigr|
=o_{\mathfrak{P}}(1).
\]
\end{enumerate}
\end{enumerate}

A natural question is why condition \ref{H12} is needed. The reason is that the
full conformal score involves two sources of randomness that are not controlled
by the analytic assumptions \ref{H1}-\ref{H11}. For a candidate value \(y\), full
CP compares
\[
r_{n,-i}(y_i)
\qquad\text{with}\qquad
r_n(y),
\]
where \(r_{n,-i}\) is the posterior predictive density computed after leaving
out the \(i\)-th observation from the augmented sample \((y^n,y)\). Thus, unlike \(r_n(y_i)\), the leave-one-out term \(r_{n,-i}(y_i)\) depends on
the candidate \(y\) through the augmented sample \(y^{n+1}_{-i}\).

To recover the population density-level threshold defined by
\(F(\cdot;y^n,P)\), we need the approximation
\[
\frac1n\sum_{i=1}^n
\mathbbm 1[\{r_{n,-i}(y_i)\le r_n(y)\}]
\approx
\frac1n\sum_{i=1}^n
\mathbbm 1[\{r_n(y_i)\le r_n(y)\}]
\approx
F(r_n(y);y^n,P),
\]
uniformly in \(y\). The first approximation is precisely the leave-one-out
stability condition \ref{H12}.(a). The second is the conditional
Glivenko-Cantelli condition \ref{H12}.(b).

The second approximation is slightly delicate because the sublevel sets are
random: The function \(r_n\) is itself computed from the same data \(y^n\).
Thus the index class
\[
\{\mathbbm 1[\{r_n(\cdot)\le t\}]:t\in\mathbb R\}
\]
is data-dependent. Assumption \ref{H12}.(b) makes this dependence explicit. In
settings where such dependence is difficult to handle directly, one can enforce
a similar condition by sample splitting, at the cost of some efficiency. We keep
the full sample formulation because it is conceptually closer to full CP.

Conditions like \ref{H12} are natural in regular settings, but they are genuine
asymptotic assumptions, not consequences of the topological hypotheses
\ref{H1}-\ref{H11}. For example, in regular parametric Bayesian models, such as
smooth exponential families on compact \(Y\) with suitably regular priors, the
posterior predictive often depends smoothly on empirical sufficient statistics.
Then deleting one observation changes the predictive surface by
\(O_{\mathfrak P}(1/n)\) uniformly in \(y\), giving \ref{H12}.(a).

For \ref{H12}.(b), one needs the empirical distribution of the observed scores
\(r_n(y_i)\) to be asymptotically compatible with the posterior predictive law
induced by \(r_n\). This is expected in well-specified regular dominated models
when \(r_n\) converges uniformly in probability to a continuous limit \(r\), the
induced predictive laws \(Q_n\) converge in probability to a law \(Q\) with
density \(r\), and a uniform law of large numbers holds over the relevant
sublevel sets. If, in addition, there are no plateaus at the target levels,
equivalently \(Q(\{r=t\})=0\) for the relevant thresholds \(t\), then
\ref{H12}.(b) follows. Strengthening the convergence statements from convergence
in probability to almost sure convergence upgrades the corresponding conclusions
from ``in probability'' to ``almost surely''.

\subsection{Main Result}
Call $\mathsf{QUANT}: [0,1] \times Y^n \times \Delta_\Theta^\text{dens} \rightrightarrows Y$ the set-valued map that extracts the $\alpha$-level set of the posterior predictive distribution. That is,

$$(\alpha,y^n,P) \mapsto \mathsf{QUANT}(\alpha,y^n,P)
    \coloneqq H\bigl(c(\alpha,y^n,P), y^n,P\bigr),$$
where 
$$H:\R\times Y^n\times\Delta_\Theta^{\mathrm{dens}}
     \rightrightarrows Y,
  \qquad
  H\bigl(c,y^n,P\bigr)
    \coloneqq \{ y\in Y: p_P(y\mid y^n)\ge c\}.$$

The following is the main result of this section.

\begin{theorem}[Unifying Bayes, Conformal, and Imprecise Prediction in \(\mathbf{UHCont}\)]
\label{unifying-thm}
Let \(Y\) be compact metric and \(\lambda_Y(Y)<\infty\). Assume the
standing hypotheses \ref{H1}, \ref{H2}, \ref{H5}-\ref{H12}, plus
consonance for the conformal transducers generated by the Bayesian conformal
scores. Fix \(P\in\Delta_\Theta^{\mathrm{dens}}\) and fix
\(\alpha\in(0,1)\). For each \(n\), write
\(r_n(y)\coloneqq p_P(y\mid y^n)\),
\(c_n\coloneqq c(\alpha,y^n,P)\), and
$Q_n\coloneqq \{y\in Y:r_n(y)\ge c_n\}
=\mathsf{QUANT}_\alpha(y^n,P)$.
Assume, in addition, the Hausdorff stability of predictive density level sets. That is, 
for \(0<\delta<\delta_0\), define
\(A_n^+(\delta)\coloneqq\{y\in Y:r_n(y)\ge c_n+\delta\}\) and
\(A_n^-(\delta)\coloneqq\{y\in Y:r_n(y)\ge c_n-\delta\}\). With the convention
that \(d_H(\emptyset,Q_n)=+\infty\), for every \(\varepsilon>0\),
\[
\lim_{\delta\downarrow0}\limsup_{n\to\infty}
\mathfrak P\left(
d_H(A_n^+(\delta),Q_n)
\vee
d_H(A_n^-(\delta),Q_n)
>\varepsilon
\right)=0.
\]
where $d_H(A_n^+(\delta),Q_n)\vee d_H(A_n^-(\delta),Q_n)
=
\max\{
d_H(A_n^+(\delta),Q_n),
d_H(A_n^-(\delta),Q_n)
\}$. Then,
\[
d_H \left(
\kappa_{\alpha}\bigl(y^n,\psi_{(y^n,P)}\bigr),
\mathsf{QUANT}_{\alpha}(y^n,P)
\right)
\xrightarrow{\mathfrak P}0.
\]
Consequently, using the exact conformal-credal factorization under consonance,
\[
d_H \left(
\widehat{\IHDR}_{\alpha}
\bigl(c_{\mathrm{CP}}(\mathsf{BCP}(y^n,P))\bigr),
\mathsf{QUANT}_{\alpha}(y^n,P)
\right)
\xrightarrow{\mathfrak P}0,
\]
where \(c_{\mathrm{CP}}\) denotes the credal-object map from Theorem
\ref{conf-pred-diagr}. Equivalently, the diagram
\[
\begin{tikzcd}[row sep=3em, column sep=4em]
& Y^n\times\mathscr F
    \ar[r,"c_{\mathrm{CP}}"]
    \ar[dr,"\kappa_\alpha"]
&
\mathscr M_{\mathrm{CP}}
    \ar[d,"\widehat{\IHDR}_\alpha"]
\\
Y^n\times\Delta_\Theta^{\mathrm{dens}}
    \ar[ur,"\mathsf{BCP}"]
    \ar[rr,"\mathsf{QUANT}_\alpha"']
&&
 Y
\end{tikzcd}
\]
commutes asymptotically in probability in the Hausdorff metric \(d_H\). The upper route agrees pointwise with the direct conformal route under
consonance; When the no-tie condition \ref{H3} is imposed, this agreement is an exact commutative
triangle in \(\mathbf{UHCont}\). The lower triangle commutes asymptotically in
probability.
\end{theorem}

If \(Y\) is compact metric, as in Theorem \ref{unifying-thm}, the set convergence
is naturally expressed in the Hausdorff distance \(d_H\). For more general
topological spaces, analogous statements can be formulated using Vietoris
convergence of closed subsets of \(Y\).

Let us also comment on the Hausdorff stability assumption imposed in Theorem
\ref{unifying-thm}. This assumption is related to \ref{H11}, but it is not
implied by it. Condition \ref{H11} controls the one-dimensional distribution of
density levels, ensuring that the cutoff \(c_n\) is stable, and that
\(F(\cdot;y^n,P)\) crosses \(\alpha\) non-flatly near \(c_n\). By contrast, the
Hausdorff stability assumption is geometric: It requires that small
perturbations of the cutoff \(c_n\) produce small perturbations, in Hausdorff
distance, of the superlevel sets \(\{r_n\ge c_n+\delta\}\),
\(\{r_n\ge c_n\}\), and \(\{r_n\ge c_n-\delta\}\). Thus it rules out zero-mass
but geometrically visible pathologies, such as isolated threshold points, thin
remote components, flat boundary pieces, or components that appear or disappear
under arbitrarily small changes of the threshold.

A standard sufficient condition, when \(Y\subseteq\mathbb R^d\), is that
\(r_n\) be \(C^1\) in a neighborhood of \(\{r_n=c_n\}\) and that, with
\(\mathfrak P\)-probability tending to one,
\(\inf_{y:r_n(y)=c_n}\|\nabla r_n(y)\|>0\). Then the regular level-set theorem
implies that nearby density superlevel sets move continuously, and locally
linearly, in Hausdorff distance as the threshold varies. This condition is
satisfied in many regular parametric predictive models away from critical
density levels; For instance, in smooth one-dimensional location families when
the derivative of \(r_n\) is nonzero at the boundary points of the density level
set, and in multivariate Gaussian or smooth elliptically contoured models when
the target density level is not the modal level.

At this point, one may ask whether it is possible to avoid this additional
Hausdorff-stability assumption; This is possible only if one changes either the
conclusion or the hypotheses. If the conclusion is weakened from
Hausdorff convergence to a measure-based notion of convergence, for instance
convergence in posterior predictive probability or in symmetric difference mass,
then the scalar regularity in \ref{H11} is often closer to sufficient. However,
Hausdorff convergence is sensitive to geometrically visible but small mass
features of the level set. Thus, to retain a \(d_H\)-conclusion, one needs some
geometric control of how the superlevel sets \(\{r_n\ge c\}\) move as \(c\)
varies. The explicit Hausdorff stability condition in Theorem
\ref{unifying-thm} could be replaced by stronger smoothness assumptions implying
it, such as \(C^1\)-regularity of \(r_n\) near the target boundary together with
a high probability lower bound on \(\|\nabla r_n\|\) along
\(\{r_n=c_n\}\), but some condition of this kind is needed for Hausdorff
convergence.

To prove Theorem \ref{unifying-thm}, we need first to introduce two ancillary lemmas. The first one is the following.

\begin{lemma}\label{bayes-cont}
Under conditions \ref{H1} and \ref{H6}-\ref{H10}, the following hold for every fixed $n\in \mathbb N$,

\begin{itemize}
    \item For every $P\in\Delta_\Theta^{\mathrm{dens}}$, the
        predictive map
        \(
            (y_{n+1},y^{n})\mapsto p_P(y_{n+1}\mid y^{n})
        \)
        is jointly continuous on $Y^n \times Y = Y^{n+1}$;
    \item The map $\mathsf{BCP}$ is continuous when $C(Y^n \times Y) \equiv C(Y^{n+1})$ is equipped with the sup-norm.
\end{itemize}
\end{lemma}

\begin{remark}[Compactness of $Y$]
    The compactness assumption on $Y$ \ref{H1} is primarily a technical device that simplifies the functional-analytic setting
(e.g. uniform norms and domination arguments). Extensions to non-compact spaces can be obtained by standard localization/truncation
arguments (or by working with locally compact $Y$ and appropriate weighted function spaces), but we do not pursue these refinements here. 
\end{remark}

\begin{proof}[Proof of Lemma \ref{bayes-cont}]
    Throughout, let $P\in\Delta_\Theta^{\mathrm{dens}}$ be fixed and write
$\text{d}P/\text{d}\lambda_\Theta=p$.

\textbf{Step 1. Continuity of the numerator $N_P$.}
Define
\[
  g(\theta,y_{n+1},y^{n})
    \coloneqq k_{n+1}(\theta,(y^n,y_{n+1}))p(\theta);
\]
We have $\lvert g\rvert\le p(\theta)M(\theta)$ by \ref{H8}.
Assumption \ref{H7} gives continuity of
$$(y_{n+1},y^{n})\mapsto g(\theta,y_{n+1},y^{n}),$$ 
for
$\lambda_\Theta$-almost all $\theta$.
The envelope $p M\in L^{1}(\lambda_\Theta)$ by \ref{H8}; Hence, 
dominated convergence yields continuity of
$(y_{n+1},y^{n})\mapsto N_P(y_{n+1},y^{n})$.

\textbf{Step 2. Continuity and positivity of the denominator $D_P$.}
The integrand in
\(
  D_P(y^{n})
\)
is $g$, but ``stopped'' at observation $n$, so the same argument as Step 1 shows that 
\(y^{n}\mapsto D_P(y^{n})\) is continuous;  Positivity follows from \ref{H10}.

\textbf{Step 3. Joint (and uniform) continuity of the predictive.}
Because $Y^{n+1}$ is compact by \ref{H1} and $D_P$ never vanishes,
\[
  p_P(y_{n+1}\mid y^{n})
      =\frac{N_P(y_{n+1},y^{n})}{D_P(y^{n})}
\]
is continuous in $(y_{n+1},y^{n})$ and, by compactness, uniformly
continuous.  This proves joint continuity of the map \(
            (y_{n+1},y^{n})\mapsto p_P(y_{n+1}\mid y^{n})
        \).

\textbf{Step 4. Continuity in the prior.}
Let $P_m\to P$ in $L^{1}(\lambda_\Theta)$, with densities
$p_m\coloneqq \text{d}P_m/\text{d}\lambda_\Theta$.
From \ref{H8},
\[
  \bigl\|N_{P_m}-N_{P}\bigr\|_\infty
     \le\int_\Theta \lvert p_m-p\rvert M \text{d}\lambda_\Theta
     \xrightarrow[m\to\infty]{}0,
\]
and similarly
\(
    \|D_{P_m}-D_P\|_\infty\to0.
\)
Because $D_P$ is bounded below on the compact set $Y^{n}$
(by Step 2 there is a constant
\(\underline d_P\coloneqq\min_{y^{n}}D_P(y^{n})>0\))
and $D_{P_m}$ converges uniformly, we have
\[
  \inf_{y^{n}}D_{P_m}(y^{n})\ge\frac12 \underline d_P,
  \quad\text{for all sufficiently large }m.
\]
Consequently
\[
  \bigl\|p_{P_m}( \cdot\mid y^{n})-p_{P}( \cdot\mid y^{n})\bigr\|_\infty
  \xrightarrow[m\to\infty]{}0
  \quad\text{uniformly in }y^{n}\in Y^{n},
\]
and division remains well behaved thanks to the common lower bound
\(\tfrac12 \underline d_P\).

\textbf{Step 5. Continuity of \(\mathsf{BCP}\).}
Fix a point \((y^n,P)\) in the domain. Let
\(y_m^n\to y^n\) in \(Y^n\) and \(P_m\to P\) in \(L^1(\lambda_\Theta)\), with
densities \(p_m=\mathrm dP_m/\mathrm d\lambda_\Theta\).

By Step 4, applied uniformly over \((z^n,z)\in Y^n\times Y\),
$$\sup_{(z^n,z)\in Y^n\times Y}
\left|
p_{P_m}(z\mid z^n)-p_P(z\mid z^n)
\right|
\to 0.$$
Equivalently,
$\|\psi_{P_m}-\psi_P\|_{\infty}
=
\sup_{(z^n,z)\in Y^n\times Y}
\left|
\psi_{P_m}(z^n,z)-\psi_P(z^n,z)
\right|
\to 0$ .
Since also \(y_m^n\to y^n\), it follows that
\[
\mathsf{BCP}(y_m^n,P_m)
=
(y_m^n,\psi_{P_m})
\to
(y^n,\psi_P)
=
\mathsf{BCP}(y^n,P)
\]
in \(Y^n\times C(Y^n\times Y)\). Hence \(\mathsf{BCP}\) is continuous.

We finish by indicating where each assumption enters the proof, and why some
replacement condition would generally be needed if it were removed.

\begin{itemize}
  \item \textbf{Role of \ref{H6} (measurability of $k_\iota$).}
        Without $\Sigma_\Theta$-measurability in $\theta$, the integrals defining
        $N_P$ and $D_P$ may not be Lebesgue integrals. Without Borel measurability
        in $y^\iota$, $(y^{n+1},y^n)\mapsto N_P(y_{n+1},y^n)$ and $y^n\mapsto D_P(y^n)$
        need not even be measurable, so continuity cannot hold.

  \item \textbf{Role of \ref{H7} (pointwise $y$-continuity a.e.).}
        Steps 1-2 in the proof require continuity in the data arguments:
        For $\lambda_\Theta$-almost all  $\theta$, $y^{\iota}\mapsto k_\iota(\theta,y^\iota)$ must be
        continuous. If $k_{n+1}(\theta,\cdot)$ has a jump, then $N_P$ can jump
        in $(y_{n+1},y^n)$; If $k_{n}(\theta,\cdot)$ has a jump, then $D_P$ can
        jump in $y^n$. Either breaks continuity of $p_P=N_P/D_P$.

  \item \textbf{Role of \ref{H8} (integrable envelope).}
        The dominated convergence steps (for $N_P$ and $D_P$ as functions of the
        data) and the uniform $L^1$ control in the prior argument rely on an
        integrable envelope $M\in L^\infty(\lambda_\Theta)$ bounding $k_{n+1}$ \emph{and}
        $k_n$, for almost all $\theta$. Without such $M$, we can have $P_m\to P$ in
        $L^1$ but $N_{P_m}\not\to N_P$ (and similarly for $D_P$), so $\mathsf{BCP}$
        is discontinuous.


  \item \textbf{Role of \ref{H10} (positive marginal likelihood).}
        If $D_P(y^n)=0$ for some $y^n$, then $p_P(\cdot\mid y^n)$ is undefined
        (division by zero), so $\mathsf{BCP}$ is not even well-defined there,
        let alone continuous.
\end{itemize}
\end{proof}

The second ancillary lemma is the following.

\begin{lemma}\label{uhc-quant}
Fix \(\alpha\in(0,1)\), and suppose that \ref{H1}-\ref{H11} hold for this
fixed value of \(\alpha\). Define
\[
\mathsf{QUANT}_\alpha:
Y^n\times \Delta_\Theta^{\mathrm{dens}}\rightrightarrows Y,
\qquad
\mathsf{QUANT}_\alpha(y^n,P)
\coloneqq
H\bigl(c(\alpha,y^n,P),y^n,P\bigr).
\]
Then \(\mathsf{QUANT}_\alpha\) is upper hemicontinuous.
\end{lemma}

\begin{proof}[Proof]
Fix \(\alpha\in(0,1)\). First, the correspondence
\[
H(c,y^n,P)=\{y\in Y:p_P(y\mid y^n)\ge c\}
\]
is upper hemicontinuous. Indeed, fix \((c_0,y_0^n,P_0)\), and let
\(V\subseteq Y\) be open with \(H(c_0,y_0^n,P_0)\subseteq V\). Set
\(C\coloneqq Y\setminus V\). Since \(Y\) is compact, \(C\) is compact. If
\(C=\emptyset\), the claim is immediate. Otherwise define
\[
\Phi(c,y^n,P)
\coloneqq
\sup_{y\in C}\bigl(p_P(y\mid y^n)-c\bigr).
\]
By Lemma \ref{bayes-cont}, the map
\((y,c,y^n,P)\mapsto p_P(y\mid y^n)-c\) is continuous, and hence \(\Phi\) is
upper semicontinuous. Since \(C\cap H(c_0,y_0^n,P_0)=\emptyset\), compactness
of \(C\) gives \(\Phi(c_0,y_0^n,P_0)<0\). Therefore, in a neighborhood of
\((c_0,y_0^n,P_0)\), one has \(\Phi<0\), equivalently
\(H(c,y^n,P)\subseteq V\). Thus \(H\) is upper hemicontinuous.

By \ref{H11}, the map
\[
q_\alpha(y^n,P)
\coloneqq
\bigl(c(\alpha,y^n,P),y^n,P\bigr)
\]
is continuous. Viewed as a singleton-valued correspondence, \(q_\alpha\) is
therefore upper hemicontinuous. Since
$\mathsf{QUANT}_\alpha=H\circ q_\alpha$,
and compositions of upper hemicontinuous correspondences are upper
hemicontinuous, \(\mathsf{QUANT}_\alpha\) is upper hemicontinuous.
\end{proof}


We are finally ready for the proof of our main result.

\begin{proof}[Proof of Theorem \ref{unifying-thm}]
Fix \(P\in\Delta_\Theta^{\mathrm{dens}}\) and \(\alpha\in(0,1)\). By
\citet{fong-bayes}, the negative posterior predictive density is a valid
non-conformity measure. In the notation of this theorem,
\(\psi_{(y^n,P)}(y)=-r_n(y)\), where \(r_n(y)=p_P(y\mid y^n)\).

For \(y\in Y\), define the full conformal transducer
$\pi_n(y)
=
\frac{1}{n+1}
\sum_{i=1}^{n+1}
\mathbbm 1\{r_{n,-i}(y_i)\le r_n(y)\}$,
with the convention \(y_{n+1}=y\) and
\(r_{n,-(n+1)}(y_{n+1})\equiv r_n(y)\). Define
$$\beta_n(\alpha)
\coloneqq
\min\left\{\frac{k}{n+1}:\frac{k}{n+1}>\alpha\right\}.$$
Since \(\pi_n\) takes values in the grid
\(\{0,1/(n+1),\ldots,1\}\), we have
\[
\kappa_\alpha(y^n,\psi_{(y^n,P)})
=
\{y\in Y:\pi_n(y)>\alpha\}
=
\{y\in Y:\pi_n(y)\ge \beta_n(\alpha)\}.
\]
Write \(K_n\coloneqq \kappa_\alpha(y^n,\psi_{(y^n,P)})\) and
\(Q_n\coloneqq \mathsf{QUANT}_\alpha(y^n,P)=\{y:r_n(y)\ge c_n\}\).
Let
\[
\Delta_n
\coloneqq
\sup_{y\in Y}\max_{i \in \{1,\ldots,n\}}|r_{n,-i}(y)-r_n(y)|
=o_{\mathfrak P}(1)
\]
by \ref{H12}.(a), and let
\(\mathscr D_n\coloneqq
\sup_{t\in\mathbb R}|F_n(t;y^n,P)-F(t;y^n,P)|
=o_{\mathfrak P}(1)\) by \ref{H12}.(b). Also, define
\[
G_n(y)
\coloneqq
\frac1n\sum_{i=1}^n
\mathbbm 1\{r_{n,-i}(y_i)\le r_n(y)\}.
\]
For every \(y\in Y\),
\[
F_n(r_n(y)-\Delta_n;y^n,P)
\le
G_n(y)
\le
F_n(r_n(y)+\Delta_n;y^n,P).
\]
Moreover, \(\pi_n(y)=\frac{n}{n+1}G_n(y)+\frac{1}{n+1}\).

We now prove the key sandwich. Fix \(0<\delta<\delta_0\). On the event
\(E_n\cap\{\Delta_n\le \delta/2\}\cap\{\mathscr D_n\le L\delta/4\}\), which has
probability tending to one, the following two inclusions hold for all large \(n\).

First, if \(r_n(y)\ge c_n+\delta\), then
\[
\begin{aligned}
G_n(y)
&\ge F_n(r_n(y)-\Delta_n;y^n,P)\\
&\ge F(r_n(y)-\Delta_n;y^n,P)-\mathscr D_n\\
&\ge F(c_n+\delta/2;y^n,P)-\mathscr D_n\\
&\ge \alpha+L\delta/2-\mathscr D_n\\
&\ge \alpha+L\delta/4.
\end{aligned}
\]
Hence,
\[
\pi_n(y)
=
\frac{n}{n+1}G_n(y)+\frac{1}{n+1}
\ge
\frac{n}{n+1}\left(\alpha+\frac{L\delta}{4}\right)
+\frac{1}{n+1}.
\]
For all sufficiently large \(n\), the right-hand side is at least
\(\beta_n(\alpha)\), because \(\beta_n(\alpha)\le \alpha+1/(n+1)\). Therefore
\(A_n^+(\delta)\subseteq K_n\).

Second, if \(r_n(y)\le c_n-\delta\), then
\[
\begin{aligned}
G_n(y)
&\le F_n(r_n(y)+\Delta_n;y^n,P)\\
&\le F(r_n(y)+\Delta_n;y^n,P)+\mathscr D_n\\
&\le F(c_n-\delta/2;y^n,P)+\mathscr D_n\\
&\le \alpha-L\delta/2+\mathscr D_n\\
&\le \alpha-L\delta/4.
\end{aligned}
\]
Thus,
\[
\pi_n(y)
\le
\frac{n}{n+1}\left(\alpha-\frac{L\delta}{4}\right)
+\frac{1}{n+1}
<
\alpha
<
\beta_n(\alpha),
\]
for all sufficiently large \(n\). Hence \(y\notin K_n\), and therefore
\(K_n\subseteq A_n^-(\delta)\).

Combining the two inclusions, we obtain that, for every fixed
\(0<\delta<\delta_0\), with probability tending to one,
\[
A_n^+(\delta)\subseteq K_n\subseteq A_n^-(\delta).
\]
Since also \(A_n^+(\delta)\subseteq Q_n\subseteq A_n^-(\delta)\), we have, on
this same event,
\[
d_H(K_n,Q_n)
\le
d_H(A_n^+(\delta),Q_n)
\vee
d_H(A_n^-(\delta),Q_n).
\]
Indeed, the inclusion \(K_n\subseteq A_n^-(\delta)\) controls the distance from
\(K_n\) to \(Q_n\), while \(A_n^+(\delta)\subseteq K_n\) controls the distance
from \(Q_n\) to \(K_n\).

Now let \(\varepsilon>0\). By the Hausdorff stability assumption, choose \(\delta>0\) small enough that
\[
\limsup_{n\to\infty}
\mathfrak P\left(
d_H(A_n^+(\delta),Q_n)
\vee
d_H(A_n^-(\delta),Q_n)
>\varepsilon
\right)
\]
is arbitrarily small. For this fixed \(\delta\), the sandwich event above has
probability tending to one. Therefore
\[
\mathfrak P\bigl(d_H(K_n,Q_n)>\varepsilon\bigr)\to0,
\]
which proves
$d_H\left(
\kappa_\alpha(y^n,\psi_{(y^n,P)}),
\mathsf{QUANT}_\alpha(y^n,P)
\right)
\xrightarrow{\mathfrak P}0$.

It remains to interpret this convergence in the diagram. By Lemma
\ref{bayes-cont}, \(\mathsf{BCP}\) is continuous, hence a morphism of
\(\mathbf{UHCont}\) when viewed as a singleton-valued correspondence. By Lemma
\ref{uhc-quant}, \(\mathsf{QUANT}_\alpha\) is upper hemicontinuous. By Proposition \ref{prop-ihdr-cpr}, under consonance the indirect
credal-IHDR construction agrees pointwise with the direct conformal
construction,
\[
\widehat{\IHDR}_\alpha
\bigl(c_{\mathrm{CP}}(\mathsf{BCP}(y^n,P))\bigr)
=
\kappa_\alpha(y^n,\psi_{(y^n,P)}).
\]
When \(\alpha\) also satisfies the no-tie condition of Theorem
\ref{conf-pred-diagr}, this equality is an equality of morphisms in
\(\mathbf{UHCont}\). The second Hausdorff convergence follows immediately from
the first.
\end{proof}

\subsection{Implications}
Theorem \ref{unifying-thm} is the central payoff of this section, and one of the strongest insights of our categorical treatment.
It makes precise an intuition articulated by \citet{martin2022valid} and \citet[Section 5.2]{caprio2025conformalpredictionregionsimprecise}: Conformal Prediction can be viewed as a
bridge between Bayesian, frequentist, and imprecise approaches to predictive reasoning.

Indeed, under regularity and boundary stability conditions, the Bayesian
\emph{predictive density level set} associated with lower-tail density mass
\(\alpha\) (which is not, in general, frequentist-calibrated at finite \(n\)),
the conformal region built from the corresponding Bayesian predictive score,
and the CP-induced imprecise construction (via the credal set and the IHDR)
produce \emph{asymptotically the same} prediction region.
In this sense, their agreement is not merely ex post at the level of sets: Asymptotically, the Bayesian 
density level set, the direct conformal region, and the indirect IHDR route {\em become equivalent as predictive procedures}, in the categorical sense encoded by the asymptotically commuting diagram.
Put differently, when the non-conformity measure is chosen as a (negative) posterior predictive density and the
posterior predictive surface is stable, conformal prediction acts as a \emph{finite-sample validity wrapper} around Bayesian prediction:
It retains the distribution-free coverage guarantee at every $n$, while becoming asymptotically indistinguishable (in $d_H$) from the Bayesian
density level set in regular regimes.

From the categorical viewpoint, this is expressed by the asymptotic commutativity of the diagram in $\mathbf{UHCont}$: The conformal morphism $\kappa_\alpha$ factors through the credal correspondence and IHDR, while the Bayes-to-score map $\mathsf{BCP}$ selects a canonical score in $\mathscr F$. Thus, the asymptotic equivalence between Bayesian, conformal, and imprecise prediction is a structural consequence of their organization as composable set-valued maps with stability properties.

\subsection{Quantitative rate of convergence for Theorem \ref{unifying-thm}}\label{cvg-rate-section}

In this section, we study quantitative rates for the convergence established in
Theorem \ref{unifying-thm}. Under local empirical process control and regularity of the predictive density level boundary, the Hausdorff distance between the
conformal region and the Bayesian predictive density level set is of order
\(O_{\mathfrak P}(n^{-1/2})\), up to the leave-one-out stability rate.

\begin{proposition}[Quantitative Convergence under Local Modulus Assumptions]
\label{prop:quant-weak}
Assume the standing hypotheses of Theorem~\ref{unifying-thm}, with the same fixed
\(P\in\Delta_\Theta^{\mathrm{dens}}\) and \(\alpha\in(0,1)\). Write
$r_n(\cdot)\coloneqq p_P(\cdot\mid y^n)$, 
$r_{n,-i}^{(y)}(\cdot)\coloneqq p_P ( \cdot\mid (y^n,y)_{-i})$, $i \in \{1,\ldots,n\}$, 
where $(y^n,y)_{-i}$ denotes the augmented sample $(y^n,y)$ with the $i$-th entry removed.

Define $\mathrm{err}_n
\coloneqq
\sup_{y\in Y}\max_{i \in \{1,\ldots,n\}}\sup_{z\in Y}
\bigl|r_{n,-i}^{(y)}(z)-r_n(z)\bigr|$,
\[
F_n(t)\coloneqq\frac1n\sum_{i=1}^n  \mathbbm{1}[\{r_n(y_i)\le t\}],
\qquad
F(t)\coloneqq F(t;y^n,P),
\]
and let
\[
c_*\coloneqq c(\alpha,y^n,P),
\qquad
T_n\coloneqq\{y\in Y:r_n(y)\ge c_*\}=\mathrm{QUANT}_\alpha(y^n,P).
\]
For $y\in Y$, define the full conformal transducer
$\pi_n(y)\coloneqq\frac{1}{n+1}\sum_{i=1}^{n+1} \mathbbm{1} [ \{
r_{n,-i}^{(y)}(y_i)\le r_n(y)\} ]$,
with the convention $y_{n+1}=y$ and
$r_{n,-(n+1)}^{(y)}(y_{n+1})\equiv r_n(y)$.
Then,
$\kappa_\alpha(y^n,\psi_{(y^n,P)})
=
\{y\in Y:\pi_n(y)>\alpha\}$.

Assume that there exist $\delta_0>0$, a deterministic sequence $\eta_n\downarrow0$,
continuous nondecreasing functions
\[
\gamma_+,\gamma_-,h_+,h_-:[0,\delta_0]\to[0,\infty)
\]
such that
\[
\gamma_+(0)=\gamma_-(0)=h_+(0)=h_-(0)=0,
\qquad
\gamma_+(u),\gamma_-(u)>0, \text{ for all }u>0,
\]
and events $E_n$ with $\mathfrak P(E_n)\to1$ such that on $E_n$ the following hold; In  (c), we use the convention
\(d_H(\emptyset,T_n)=+\infty\),

\begin{enumerate}
\item[(a)] \textbf{Local empirical control near the target cutoff:}
\[
\sup_{|t-c_*|\le \delta_0}|F_n(t)-F(t)|\le \eta_n.
\]

\item[(b)] \textbf{Local separation of $F$ at the target cutoff:} for every $u\in[0,\delta_0]$,
\[
F(c_*+u)-\alpha\ge \gamma_+(u),
\qquad
\alpha-F(c_*-u)\ge \gamma_-(u).
\]


\item[(c)] \textbf{Local Hausdorff control of level-set motion:} for every $u\in[0,\delta_0]$,
\[
d_H\bigl(\{r_n\ge c_*+u\},T_n\bigr)\le h_+(u),
\qquad
d_H\bigl(\{r_n\ge c_*-u\},T_n\bigr)\le h_-(u).
\]
\end{enumerate}

For any nondecreasing $g:[0,\delta_0]\to[0,\infty)$ with $g(0)=0$ and $g(u)>0$ for $u>0$,
define its generalized inverse by
\[
g^{\leftarrow}(t)\coloneqq\inf\{u\in[0,\delta_0]:g(u)\ge t\},
\]
with the convention $\inf\emptyset\coloneqq\delta_0$.
Set
\[
s_n^+\coloneqq\gamma_+^{\leftarrow}(\eta_n),
\qquad
s_n^-\coloneqq\gamma_-^{\leftarrow} \left(\eta_n+\frac1n\right),
\qquad 
u_n^+\coloneqq\mathrm{err}_n+s_n^+,
\qquad
u_n^-\coloneqq\mathrm{err}_n+s_n^-.
\]

Then, for all sufficiently large \(n\), on the event
$E_n\cap\{u_n^+\le \delta_0, u_n^-\le \delta_0\}$, 
we have the sandwich
\begin{equation}\label{sandwich}
    \{y:r_n(y)\ge c_*+u_n^+\}
\subseteq
\kappa_\alpha(y^n,\psi_{(y^n,P)})
\subseteq
\{y:r_n(y)\ge c_*-u_n^-\},
\end{equation}
and consequently
\begin{equation}\label{first-claim}
    d_H \Bigl(\kappa_\alpha(y^n,\psi_{(y^n,P)}), \mathrm{QUANT}_\alpha(y^n,P)\Bigr)
\le
h_+(u_n^+)\vee h_-(u_n^-).
\end{equation}

In particular, if $u_n^+\to0$ and $u_n^-\to0$ in probability, then
\[
d_H \Bigl(\kappa_\alpha(y^n,\psi_{(y^n,P)}), \mathrm{QUANT}_\alpha(y^n,P)\Bigr)
\overset{\mathfrak P}{\to}0.
\]

Moreover, if for some constants $m_\pm,L_\pm>0$ and some deterministic sequence
$\rho_n\downarrow0$ we have that
\[
\eta_n=O(n^{-1/2}),
\quad
\mathrm{err}_n=O_{\mathfrak P}(\rho_n),
\quad
\gamma_+(u)\ge m_+u,
\]
\[
\gamma_-(u)\ge m_-u,
\quad
h_+(u)\le L_+u,
\quad
h_-(u)\le L_-u,
\]
for all sufficiently small $u$, then
\begin{equation}\label{rate-statement}
    d_H \Bigl(\kappa_\alpha(y^n,\psi_{(y^n,P)}), \mathrm{QUANT}_\alpha(y^n,P)\Bigr)
=
O_{\mathfrak P} \bigl(n^{-1/2}+\rho_n\bigr).
\end{equation}

In particular, if $\rho_n=O(n^{-1/2})$, then the Hausdorff error is
$O_{\mathfrak P}(n^{-1/2})$.
\end{proposition}

\begin{proof}
For each $y\in Y$, define
\[
G_n(y)\coloneqq\frac1n\sum_{i=1}^n
 \mathbbm{1} \left[ \left\{r_{n,-i}^{(y)}(y_i)\le r_n(y)\right\} \right],
\]
so that
$\pi_n(y)=\frac{n}{n+1}G_n(y)+\frac1{n+1}$.
Since \(\eta_n\downarrow0\) and \(\gamma_+(\delta_0)>0\), while
\(\eta_n+1/n\to0\) and \(\gamma_-(\delta_0)>0\), for all sufficiently large
\(n\) the defining sets of \(s_n^+\) and \(s_n^-\) are nonempty. Hence
\[
\gamma_+(s_n^+)\ge \eta_n,
\qquad
\gamma_-(s_n^-)\ge \eta_n+\frac1n.
\]

By definition of $\mathrm{err}_n$, for every $y\in Y$, every $i \in \{1,\ldots,n\}$, and every $z\in Y$,
\[
r_n(z)\le r_{n,-i}^{(y)}(z)+\mathrm{err}_n,
\qquad
r_{n,-i}^{(y)}(z)\le r_n(z)+\mathrm{err}_n.
\]
Taking $z=y_i$ yields 
$r_n(y_i)\le r_{n,-i}^{(y)}(y_i)+\mathrm{err}_n$, 
$r_{n,-i}^{(y)}(y_i)\le r_n(y_i)+\mathrm{err}_n$. 
Hence,
\[
 \mathbbm{1}[\{r_n(y_i)\le r_n(y)-\mathrm{err}_n\}]
\le
 \mathbbm{1} \left[ \left\{r_{n,-i}^{(y)}(y_i)\le r_n(y)\right\} \right]
\le
 \mathbbm{1}[\{r_n(y_i)\le r_n(y)+\mathrm{err}_n\}],
\]
and summing over $i\in\{1,\ldots,n\}$ gives
\begin{equation}
\label{eq:local-bracketing-weak-final}
F_n(r_n(y)-\mathrm{err}_n)\le G_n(y)\le F_n(r_n(y)+\mathrm{err}_n).
\end{equation}

We first prove the left inclusion in \eqref{sandwich}.
Assume that $u_n^+\le \delta_0$ and that
$r_n(y)\ge c_*+u_n^+=c_*+\mathrm{err}_n+s_n^+$.
Then, 
$r_n(y)-\mathrm{err}_n\ge c_*+s_n^+$, 
so by \eqref{eq:local-bracketing-weak-final}, monotonicity of $F_n$, assumption~(a),
and assumption~(b),
\[
G_n(y)
\ge
F_n(c_*+s_n^+)
\ge
F(c_*+s_n^+)-\eta_n
\ge
\alpha+\gamma_+(s_n^+)-\eta_n.
\]
By the observation at the beginning of the proof,
$\gamma_+(s_n^+)\ge \eta_n$, 
hence $G_n(y)\ge \alpha$.
Therefore
\[
\pi_n(y)
=
\frac{n}{n+1}G_n(y)+\frac1{n+1}
\ge
\frac{n}{n+1}\alpha+\frac1{n+1}
=
\alpha+\frac{1-\alpha}{n+1}
>
\alpha,
\]
and so $y\in\kappa_\alpha(y^n,\psi_{(y^n,P)})$.
This proves
$\{y:r_n(y)\ge c_*+u_n^+\}
\subseteq
\kappa_\alpha(y^n,\psi_{(y^n,P)})$.

We next prove the right inclusion in \eqref{sandwich}.
Assume that $u_n^-\le \delta_0$ and that
$r_n(y)\le c_*-u_n^-=c_*-\mathrm{err}_n-s_n^-$. 
Then
$r_n(y)+\mathrm{err}_n\le c_*-s_n^-$, 
so by \eqref{eq:local-bracketing-weak-final}, monotonicity of $F_n$, assumption~(a),
and assumption~(b),
\[
G_n(y)
\le
F_n(c_*-s_n^-)
\le
F(c_*-s_n^-)+\eta_n
\le
\alpha-\gamma_-(s_n^-)+\eta_n.
\]
By the observation at the beginning of the proof,
$\gamma_-(s_n^-)\ge \eta_n+\frac1n$, 
and therefore
$G_n(y)\le \alpha-\frac1n$.
Hence
\[
\pi_n(y)
=
\frac{n}{n+1}G_n(y)+\frac1{n+1}
\le
\frac{n}{n+1}\left(\alpha-\frac1n\right)+\frac1{n+1}
=
\alpha-\frac{\alpha}{n+1}
<
\alpha,
\]
and thus $y\notin\kappa_\alpha(y^n,\psi_{(y^n,P)})$.
This proves
$\kappa_\alpha(y^n,\psi_{(y^n,P)})
\subseteq
\{y:r_n(y)\ge c_*-u_n^-\}$.

We have therefore shown that on
$E_n\cap\{u_n^+\le \delta_0, u_n^-\le \delta_0\}$,
\[
A_n\coloneqq\{r_n\ge c_*+u_n^+\}
\subseteq
\kappa_\alpha(y^n,\psi_{(y^n,P)})
\subseteq
B_n\coloneqq\{r_n\ge c_*-u_n^-\},
\]
while by definition of $T_n$,
$A_n\subseteq T_n\subseteq B_n$.
For any set $K$ such that $A_n\subseteq K\subseteq B_n$, we have
\[
\sup_{x\in K}\mathrm{dist}(x,T_n)
\le
\sup_{x\in B_n}\mathrm{dist}(x,T_n)
\le
d_H(B_n,T_n),
\]
and
\[
\sup_{t\in T_n}\mathrm{dist}(t,K)
\le
\sup_{t\in T_n}\mathrm{dist}(t,A_n)
\le
d_H(A_n,T_n).
\]
Therefore,
\[
d_H(K,T_n)\le d_H(A_n,T_n)\vee d_H(B_n,T_n).
\]
Applying this with $K=\kappa_\alpha(y^n,\psi_{(y^n,P)})$ and using assumption~(c), we get
\[
d_H \Bigl(\kappa_\alpha(y^n,\psi_{(y^n,P)}),T_n\Bigr)
\le
h_+(u_n^+)\vee h_-(u_n^-).
\]
Since $T_n=\mathrm{QUANT}_\alpha(y^n,P)$, the first claim \eqref{first-claim} follows.

If $u_n^+\to0$ and $u_n^-\to0$ in probability, then
$h_+(u_n^+)\vee h_-(u_n^-)\to0$ in probability by continuity of $h_\pm$ at $0$, and
hence the Hausdorff distance converges to $0$ in probability.
Since \(\mathrm{err}_n=O_{\mathfrak P}(\rho_n)\), \(\eta_n\to0\), and
\(\rho_n\to0\), we have \(u_n^+\to0\) and \(u_n^-\to0\) in probability. Hence,
with probability tending to one, \(u_n^\pm\) lie in the neighborhood on which the
linear bounds on \(h_\pm\) and \(\gamma_\pm\) hold.

Finally, under the linear lower bounds on $\gamma_\pm$,
\[
s_n^+\le \frac{\eta_n}{m_+},
\qquad
s_n^-\le \frac{\eta_n+1/n}{m_-},
\]
for all sufficiently large $n$, while under the linear upper bounds on $h_\pm$,
\[
h_+(u_n^+)\vee h_-(u_n^-)
=
O_{\mathfrak P} \bigl(\eta_n+\mathrm{err}_n+1/n\bigr)
=
O_{\mathfrak P} \bigl(n^{-1/2}+\rho_n\bigr),
\]
since $1/n=O(n^{-1/2})$.
This proves the rate statement \eqref{rate-statement}.
\end{proof}

Proposition~\ref{prop:quant-weak} separates the argument into four logically
distinct tasks.

\begin{enumerate}
\item[(1)] \emph{Local empirical control.}
One only needs control of \(F_n-F\) in a neighborhood of the target cutoff
\(c_*\), not on all of \(\mathbb R\). This is weaker than the global form of the
conditional GC assumption \ref{H12}.(b) used in the proof of
Theorem~\ref{unifying-thm}. It is natural when the random sublevel set class
\[
\left\{\mathbbm{1}[\{r_n(\cdot)\le t\}]: |t-c_*|\le \delta_0\right\}
\]
is asymptotically equivalent to a fixed VC-type threshold class, or in split/sample-split
variants where the score function is conditionally fixed.

\item[(2)] \emph{Local separation.}
The functions \(\gamma_\pm\) quantify how sharply \(F\) crosses the target level
\(\alpha\). They rule out the boundary failure modes discussed around
\ref{H11}: Plateaus of positive measure, jumps caused by atoms, gaps in the
range of the predictive density near \(c_*\), and sequences along which the
local slope of \(F\) degenerates. Differentiability with
\[
\partial_c F(c_*;y^n,P)>0
\]
yields the linear case \(\gamma_\pm(u)\asymp u\).

\item[(3)] \emph{Hausdorff motion of level sets.}
The functions \(h_\pm\) encode the geometric passage from threshold error to set
error. A standard sufficient condition is that \(Y\subset\mathbb R^d\), that
\(r_n\) be \(C^1\) in a neighborhood of the boundary \(\{r_n=c_*\}\), and that
\[
\inf_{y: r_n(y)=c_*}\|\nabla r_n(y)\|\ge g_0>0
\]
with high probability. Then the implicit function theorem yields
\(h_\pm(u)\asymp u\) for small \(u\).

\item[(4)] \emph{Small leave-one-out perturbation.}
The term \(\mathrm{err}_n\) is small whenever the posterior predictive depends
smoothly on the empirical distribution or on a finite-dimensional sufficient
statistic. In regular parametric Bayesian models this often gives
\[
\mathrm{err}_n=O_{\mathfrak P}(n^{-1})
\qquad\text{or at worst}\qquad
\mathrm{err}_n=O_{\mathfrak P}(n^{-1/2}).
\]
\end{enumerate}

Thus the proposition applies in regular models whenever one has:
(i) local empirical-process control near the target threshold,
(ii) non-flat crossing of the predictive score distribution at the target level,
(iii) stable movement of the posterior predictive density-level sets under small threshold perturbations, 
and (iv) small leave-one-out perturbations.

\begin{example}[A regular one-dimensional Bayesian location model]
\label{ex:weak-rate-gaussian}
Consider \(Y=[-M,M]\) with its Euclidean metric, where \(M>0\), and let
\(Y_1,\dots,Y_n\mid\theta\) be conditionally i.i.d. from a smooth dominated
location family on \(Y\), for example a Gaussian density truncated to
\([-M,M]\), with a smooth prior on \(\theta\). In regular cases, the posterior
predictive density \(r_n(\cdot)\) is smooth in its argument and depends
smoothly on empirical summaries of the data. In the Gaussian location case,
this dependence is essentially through the sample mean.

Deleting one observation changes the relevant empirical summary by
\(O_{\mathfrak P}(1/n)\). Under smooth dependence of the posterior predictive
on this summary, this gives
$\mathrm{err}_n=O_{\mathfrak P}(1/n)$.
Because \(r_n\) is smooth and \(c_*\) is away from critical levels, meaning that
there exists \(g_0>0\) such that \(|r_n'(y)|\ge g_0\) whenever
\(r_n(y)=c_*\), the boundary \(\{r_n=c_*\}\) consists of regular points, i.e.
points at which the derivative of \(r_n\) does not vanish. Consequently, the
level sets move linearly with the threshold. Thus one may take
\[
h_+(u)\le L_+u,
\qquad
h_-(u)\le L_-u
\]
for small \(u\). If, moreover, the density-level distribution \(F\) crosses
\(\alpha\) with nonzero local slope at \(c_*\), then
\[
\gamma_+(u)\ge m_+u,
\qquad
\gamma_-(u)\ge m_-u
\]
for small \(u\).

Finally, in one dimension, threshold classes generated by sufficiently regular
scores have finite VC complexity, so the local empirical discrepancy is
typically of order
$\eta_n=O_{\mathfrak P}(n^{-1/2})$.
Proposition~\ref{prop:quant-weak} then yields
\[
d_H \Bigl(\kappa_\alpha(y^n,\psi_{(y^n,P)}),
\mathrm{QUANT}_\alpha(y^n,P)\Bigr)
=
O_{\mathfrak P}(n^{-1/2}).
\]

This example should be read as a template rather than a fully worked-out
theorem: the exact verification of the local empirical process bound, the
leave-one-out stability rate, and the boundary regularity depends on the chosen
model and prior. \hfill \(\diamond\)
\end{example}

\subsection{Conformal Prediction and e-Posteriors}\label{e-post-section}
In this section, we expand on how Theorem \ref{unifying-thm} connects CP to the
concept of e-posterior \citep{grunwald2023posterior}, thereby placing CP within
a broader landscape of modern uncertainty quantification tools. Additional
connections between imprecise probability and e-posteriors can be found in
\citet{Martin2024Possibility} and \citet[Section 4]{martin2025regularizedeprocessesanytimevalid}.

Recall that $\Delta_Y^\text{dens}$ denotes the space of finitely additive probability measures on $Y$ that admit a density with respect to a $\sigma$-finite dominating (atomless) measure $\lambda_Y$. Under the de Finetti representation for the exchangeable process discussed in
Section \ref{full-cp}, the likelihood density
\(p(y^n\mid\theta)\equiv \ell(y^n\mid\theta)\) can be written as
\citep{definetti1,definetti2}
\[
\ell(y^n \mid \theta)
=
\int \left[\prod_{i=1}^n h(y_i\mid\varsigma)\right]
G_\theta(\mathrm d\varsigma),
\]
where $h(\cdot \mid \varsigma)$ is the density of an element $L_\varsigma \in \Delta_Y^\text{dens}$ (given a latent parameter $\varsigma$), and $G_\theta$ is a mixing distribution indexed by the parameter $\theta \in \Theta$. 

Consider then a credal prior set $\mathcal{P}_\text{prior}$, that is, 
a nonempty convex and weak$^\star$-closed subset of \(\Delta_\Theta^\text{dens}\)
(and, under our standing topology on \(\Delta_\Theta^\text{dens}\), also closed in
the \(L^1(\lambda_\Theta)\) topology).
Let $\overline{P}$ be its upper probability, and $\overline{p}\coloneqq \sup_{P \in \mathcal{P}_\text{prior}} \frac{\text{d}P}{\text{d}\lambda_\Theta}$ be the upper prior density; The lower prior density $\underline{p}$ is defined similarly, with $\sup$ replaced by $\inf$.\footnote{
Densities are defined only $\lambda_\Theta$-a.e. and, for uncountable credal sets, the pointwise envelopes
$\overline{p}$ and
$\underline{p}$ need not be measurable or version-invariant.
In full generality we should therefore work with the essential envelopes
$\overline p \coloneqq \operatorname*{ess sup}_{P\in\mathcal P_\text{prior}} \frac{\mathrm dP}{\mathrm d\lambda_\Theta}$
and
$\underline p \coloneqq \operatorname*{ess inf}_{P\in\mathcal P_\text{prior}} \frac{\mathrm dP}{\mathrm d\lambda_\Theta}$,
and interpret statements involving $\underline p,\overline p$ ``for all $\theta$'' as holding for $\lambda_\Theta$-a.e. $\theta$. We keep a simpler notation for ease of reading.}

\begin{proposition}[The Upper Posterior is an e-Posterior]\label{post-e-prop}
Suppose that $\underline{p}\in L^{1}(\lambda_\Theta)$, that 
$\underline{\ell}(y^n)\coloneqq\int_\Theta \ell(y^n\mid\theta)\underline p(\theta)\lambda_\Theta(\mathrm d\theta)>0$,
for all $y^n \in Y^n$, and that the upper prior density $\overline p(\theta)\in(0,\infty)$, for all $\theta\in\Theta$.
Then, the upper posterior
\begin{equation}\label{upper-post}
\overline{p}(\theta \mid y^n)
=\frac{\ell(y^n \mid \theta)\overline{p}(\theta)}{\underline{\ell}(y^n)}
\end{equation}
is an e-posterior in the sense of \citet{grunwald2023posterior} if and only if
\begin{equation}\label{iff-cond-corrected}
\int_\Theta \underline{p}(\theta)\lambda_\Theta(\mathrm d\theta)\le\overline{p}(\theta),
\quad\text{for all }\theta\in\Theta.
\end{equation}
\end{proposition}

The upper posterior in \eqref{upper-post} corresponds to the e-variable family $\{S_\theta\}_{\theta\in\Theta}$, where 
\[
S_\theta(y^n) \coloneqq \frac{\underline{\ell}(y^n)}{\ell(y^n\mid\theta) \overline p(\theta)} ,
\]
with an arbitrary finite value assigned on the \(L_\theta\)-null set
\(\{y^n:\ell(y^n\mid\theta)=0\}\), since $\overline p(\theta\mid y^n)\propto 1/S_\theta(y^n)$.
Proposition \ref{post-e-prop} characterizes when each \(S_\theta\) is an
e-variable under \(L_\theta\) \citep{grunwald2023posterior}.


We can interpret the necessary and sufficient condition \eqref{iff-cond-corrected} in Proposition \ref{post-e-prop} as follows:
The quantity $\int_\Theta \underline p(\theta)\lambda_\Theta(\mathrm d\theta)$ is the total \emph{common mass} shared by all priors in $\mathcal P_\text{prior}$ (the part of the prior density that is guaranteed under every admissible $P$).
Condition \eqref{iff-cond-corrected} requires that this common mass is everywhere dominated by the \emph{upper envelope} $\overline p(\theta)$, i.e. the maximal prior density that $\mathcal P_\text{prior}$ is willing to assign at $\theta$.
Equivalently, the credal prior set must be sufficiently ``spread'' so that the overlap shared by all priors does not exceed the pointwise
peak allowed by at least one prior at each parameter value.

\begin{proof}[Proof of Proposition \ref{post-e-prop}]
    First, notice that the upper posterior density in \eqref{upper-post} is well defined. 
    In particular, it is an upper bound for the generalized Bayes' upper posterior density \citep{caprio2026robustpredictiveuncertaintydouble}, \citep[Theorem 6.4.6]{walley1991statistical}.
    
    Then, recall that an e-posterior is a function $p^\text{e}(\theta \mid y^n)=1/Z_\theta(y^n)$, where $Z_\theta(y^n)$ is an e-variable, that is, a statistic for which $\mathbb{E}_{Y^n \sim L_\theta}[Z_\theta(Y^n)] \leq 1$. Call $\lambda_Y^{\otimes n}$ the $n$-fold product of $\lambda_Y$; We have
    \begin{align*}
   \mathbb{E}_{Y^n \sim L_\theta}\left[ \frac{1}{\overline{p}(\theta \mid Y^n)} \right] &= \mathbb{E}_{Y^n \sim L_\theta}\left[ \frac{\underline{\ell}(Y^n)}{\ell(Y^n \mid \theta)\overline{p}(\theta)} \right] \\
   &= \int_{Y^n} \frac{\underline{\ell}(y^n)}{\ell(y^n \mid \theta)\overline{p}(\theta)} \ell(y^n \mid \theta) \lambda_Y^{\otimes n}(\text{d}y^n)\\
   &=\frac{1}{\overline{p}(\theta)} \int_{Y^n} \underline{\ell}(y^n) \lambda_Y^{\otimes n}(\text{d}y^n)\\
   &\leq 1 \iff \int_{Y^n} \underline{\ell}(y^n) \lambda_Y^{\otimes n}(\text{d}y^n) \leq  \overline{p}(\theta), 
\end{align*} 
for all $\theta \in \Theta$. Now, notice that
\begin{align*}
    \int_{Y^n} \underline{\ell}(y^n) \lambda_Y^{\otimes n}(\text{d}y^n) &= \int_{Y^n} \int_\Theta \ell(y^n \mid \theta)\underline{p}(\theta) \lambda_\Theta(\mathrm{d}\theta) \lambda_Y^{\otimes n}(\text{d}y^n)\\
    &=\int_\Theta \underline{p}(\theta) \int_{Y^n} \ell(y^n \mid \theta) \lambda_Y^{\otimes n}(\text{d}y^n) \lambda_\Theta(\mathrm{d}\theta) \\
    &=\int_\Theta \underline{p}(\theta) \lambda_\Theta(\mathrm{d}\theta),
\end{align*}
where the last-but-one equality comes from Fubini-Tonelli, and the last one comes from $\ell(y^n \mid \theta)$ being a proper density. Hence, the upper posterior density $\overline{p}(\theta \mid y^n)$ is an e-posterior if and only if
\begin{align}\label{iff-cond}
    \int_\Theta \underline{p}(\theta) \lambda_\Theta(\mathrm{d}\theta) \leq \overline{p}(\theta), 
\end{align}
for all $\theta \in \Theta$. This concludes the proof.
\end{proof}



A consequence of Proposition \ref{post-e-prop} is that, if \eqref{iff-cond-corrected} holds and 
\begin{align}\label{eq-cond-e-cp}
    \overline{\Pi}(A) \equiv \overline{\Pi}_{(y^n,\psi)}(A)=\int_A \underbrace{\int_\Theta \ell(y_{n+1} \mid \theta) \overline{p}(\theta \mid y^n) \lambda_\Theta(\text{d}\theta)}_{\eqqcolon \overline{p}(y_{n+1} \mid y^n)}\lambda_Y(\text{d}y_{n+1}), \quad \forall A\in \Sigma_Y,
\end{align}
then the $\alpha$-level Imprecise Highest Density Region $\IHDR(\mathcal{M}(\overline{\Pi}))$ in Theorem \ref{unifying-thm}  is a prediction region that is associated with an e-posterior. The equality in \eqref{eq-cond-e-cp} only holds if the upper posterior predictive density $\overline{p}(y_{n+1} \mid y^n)$ collapses to a point mass. We will denote
\[
  \overline\Pi_{\mathrm B}(A)
    \coloneqq\int_A\overline p(y\mid y^{n}) \lambda_Y(dy),
  \qquad A\in\Sigma_Y,
\]
so that \eqref{eq-cond-e-cp} reads $\overline\Pi=\overline\Pi_{\mathrm B}$. 

\begin{proposition}[Coincidence of Upper Probabilities Implies Degeneracy]
\label{prop-degen}
Assume \ref{H1} and \eqref{eq-cond-e-cp}, and recall that
\(\overline\Pi(Y)=1\). We do not impose the atomlessness assumption on
\(\lambda_Y\) in this proposition.
Let
\[
\mu_{\overline p}(\mathrm d\theta\mid y^n)
\coloneqq
\overline p(\theta\mid y^n)\lambda_\Theta(\mathrm d\theta)
\]
be the additive measure induced by the upper envelope posterior density. Then,
\(\mu_{\overline p}(\cdot\mid y^n)\) is a probability measure, and there exists
\(y^\star\in Y\) such that
$\overline\Pi_{\mathrm B}(\cdot)=\delta_{y^\star}(\cdot)$
and
$\mu_{\overline p}
\bigl(\{\theta:L_\theta=\delta_{y^\star}\}\mid y^n\bigr)=1$.

In particular, if \(\overline\Pi_{\mathrm B}\ll\lambda_Y\), then necessarily
\(\lambda_Y(\{y^\star\})>0\), and its density with respect to \(\lambda_Y\) is
\[
\overline p(y\mid y^n)
=
\frac{\mathbbm 1[\{y=y^\star\}]}{\lambda_Y(\{y^\star\})}.
\]
Moreover, the density \(\overline p(\cdot\mid y^n)\) on \(\Theta\) is supported on
$\Theta^\star\coloneqq\{\theta:L_\theta=\delta_{y^\star}\}$,
that is,
\[
\overline p(\theta\mid y^n)=0,
\quad\text{for \(\lambda_\Theta\)-almost all }\theta\notin\Theta^\star,
\qquad
\int_{\Theta^\star}\overline p(\theta\mid y^n)\lambda_\Theta(\mathrm d\theta)=1.
\]
If, in addition, the model is identifiable, so that \(\theta\mapsto L_\theta\)
is injective, then
$\mu_{\overline p}(\cdot\mid y^n)=\delta_{\theta^\star}$.
In this identifiable case, compatibility with the absolute continuity
\(\mu_{\overline p}\ll\lambda_\Theta\) requires
\(\lambda_\Theta(\{\theta^\star\})>0\).
\end{proposition}

\begin{proof}
We prove the result in four steps.

\textit{Step 1 (Normalization of the upper envelope posterior density).}
By \eqref{eq-cond-e-cp},
\[
\overline\Pi_{\mathrm B}(A)
=
\int_A
\int_\Theta \ell(y\mid\theta)\overline p(\theta\mid y^n)
\lambda_\Theta(\mathrm d\theta)
\lambda_Y(\mathrm dy),
\qquad A\in\Sigma_Y.
\]
Equivalently, by Tonelli's theorem,
\[
\overline\Pi_{\mathrm B}(A)
=
\int_\Theta L_\theta(A)\,
\mu_{\overline p}(\mathrm d\theta\mid y^n),
\]
where
$\mu_{\overline p}(\mathrm d\theta\mid y^n)
=
\overline p(\theta\mid y^n)\lambda_\Theta(\mathrm d\theta)$.
Taking \(A=Y\), using \(L_\theta(Y)=1\), and using
\(\overline\Pi_{\mathrm B}(Y)=\overline\Pi(Y)=1\), we get
\[
1
=
\overline\Pi_{\mathrm B}(Y)
=
\int_\Theta \overline p(\theta\mid y^n)\lambda_\Theta(\mathrm d\theta).
\]
Thus \(\mu_{\overline p}(\cdot\mid y^n)\) is an additive probability measure.

\textit{Step 2 (Maxitivity versus additivity).}
For disjoint \(A,B\in\Sigma_Y\), the possibility upper probability satisfies
\[
\overline\Pi(A\cup B)=\max\{\overline\Pi(A),\overline\Pi(B)\},
\]
whereas the additive probability \(\overline\Pi_{\mathrm B}\) satisfies
\[
\overline\Pi_{\mathrm B}(A\cup B)
=
\overline\Pi_{\mathrm B}(A)+\overline\Pi_{\mathrm B}(B).
\]
Since \(\overline\Pi=\overline\Pi_{\mathrm B}\) by \eqref{eq-cond-e-cp}, we have
\[
\max\{\overline\Pi_{\mathrm B}(A),\overline\Pi_{\mathrm B}(B)\}
=
\overline\Pi_{\mathrm B}(A)+\overline\Pi_{\mathrm B}(B)
\]
for all disjoint \(A,B\in\Sigma_Y\). Taking \(B=A^c\) gives
\[
\max\{\overline\Pi_{\mathrm B}(A),1-\overline\Pi_{\mathrm B}(A)\}=1,
\]
and hence
\[
\overline\Pi_{\mathrm B}(A)\in\{0,1\},
\qquad A\in\Sigma_Y.
\]

\textit{Step 3 (\(0\)-\(1\) regular Borel probabilities on compact Hausdorff spaces are Dirac).}
By Step 2, \(\overline\Pi_{\mathrm B}\) is \(\{0,1\}\)-valued. Let
\[
\mathcal K
\coloneqq
\{K\subseteq Y:K\text{ closed and }\overline\Pi_{\mathrm B}(K)=1\}.
\]
If \(K_1,K_2\in\mathcal K\), then
\[
\overline\Pi_{\mathrm B}(K_1\cap K_2)
\ge
\overline\Pi_{\mathrm B}(K_1)+\overline\Pi_{\mathrm B}(K_2)-1
=
1.
\]
Thus \(\mathcal K\) has the finite intersection property. Since \(Y\) is compact,
\[
K^\prime\coloneqq\bigcap_{K\in\mathcal K}K
\]
is nonempty. Fix \(y^\star\in K^\prime\).
If \(U\) is any open neighborhood of \(y^\star\), then \(Y\setminus U\) is closed
and cannot belong to \(\mathcal K\), because otherwise it would contain
\(y^\star\). Therefore
\[
\overline\Pi_{\mathrm B}(Y\setminus U)=0,
\qquad
\overline\Pi_{\mathrm B}(U)=1.
\]
Since \(\overline\Pi_{\mathrm B}\) is a finite Borel measure on a compact
Hausdorff space, it is regular. Hence, by outer regularity,
\[
\overline\Pi_{\mathrm B}(\{y^\star\})
=
\inf\{\overline\Pi_{\mathrm B}(U):U\text{ open},\,y^\star\in U\}
=
1.
\]
Thus
\[
\overline\Pi_{\mathrm B}=\delta_{y^\star}.
\]

\textit{Step 4 (Collapse of the upper envelope posterior predictive).}
From Steps 1 and 3,
\[
\delta_{y^\star}(\cdot)
=
\overline\Pi_{\mathrm B}(\cdot)
=
\int_\Theta L_\theta(\cdot)\,
\mu_{\overline p}(\mathrm d\theta\mid y^n).
\]
Apply this equality to \(Y\setminus\{y^\star\}\). Since
\(\delta_{y^\star}(Y\setminus\{y^\star\})=0\), we get
\[
0
=
\int_\Theta L_\theta(Y\setminus\{y^\star\})\,
\mu_{\overline p}(\mathrm d\theta\mid y^n).
\]
The integrand is nonnegative, so
\[
L_\theta(Y\setminus\{y^\star\})=0
\]
for \(\mu_{\overline p}(\cdot\mid y^n)\)-almost every \(\theta\). Equivalently,
\(L_\theta=\delta_{y^\star}\) for
\(\mu_{\overline p}(\cdot\mid y^n)\)-almost every \(\theta\). Therefore
\[
\mu_{\overline p}
\bigl(\{\theta:L_\theta=\delta_{y^\star}\}\mid y^n\bigr)=1.
\]

If \(\overline\Pi_{\mathrm B}\ll\lambda_Y\), then
\(\delta_{y^\star}\ll\lambda_Y\), which forces
\(\lambda_Y(\{y^\star\})>0\). In that case the density of
\(\delta_{y^\star}\) with respect to \(\lambda_Y\) is
\[
\overline p(y\mid y^n)
=
\frac{\mathbbm 1[\{y=y^\star\}]}{\lambda_Y(\{y^\star\})}.
\]
Since
\[
\mu_{\overline p}
\bigl(\Theta^\star\mid y^n\bigr)=1,
\qquad
\Theta^\star=\{\theta:L_\theta=\delta_{y^\star}\},
\]
the density \(\overline p(\theta\mid y^n)\) is supported on \(\Theta^\star\) in
the stated sense.

Finally, if \(\theta\mapsto L_\theta\) is injective, then
\(\Theta^\star=\{\theta^\star\}\), and therefore
\[
\mu_{\overline p}(\cdot\mid y^n)=\delta_{\theta^\star}.
\]
If also \(\mu_{\overline p}\ll\lambda_\Theta\), this is possible only when
\(\lambda_\Theta(\{\theta^\star\})>0\).
\end{proof}

If $\lambda_Y$ is atomless and the upper posterior predictive admits a density $\overline p(\cdot\mid y^n)\in L^1(\lambda_Y)$,
then \eqref{eq-cond-e-cp} cannot hold for all $A\in\Sigma_Y$.
Indeed, \eqref{eq-cond-e-cp} would imply $\overline\Pi=\overline\Pi_{\mathrm B}$, and Proposition \ref{prop-degen} shows that this forces
$\overline\Pi_{\mathrm B}=\delta_{y^\star}$ for some $y^\star\in Y$, which is not absolutely continuous with respect to an atomless $\lambda_Y$.
Thus, equality in \eqref{eq-cond-e-cp} can occur only under degeneracy/atomicity conditions (and is therefore mainly of conceptual interest).

Nevertheless, Proposition \ref{prop-degen} suggests a possible extension of the diagram in Theorem \ref{unifying-thm}:
One may attempt to define an arrow from $Y^n\times\Delta_\Theta^\text{dens}$ (or, more generally, from a class of credal priors) to $\mathscr C$
by composing an e-posterior-based update with a credal-to-predictive construction.
In categorical terms, this would amount to replacing point priors $P$ by credal sets $\mathcal P_\text{prior}$ as inputs and mapping
$(y^n,\mathcal P_\text{prior})$ to the corresponding posterior credal set in $\mathscr C$.%
\footnote{More precisely, for such an arrow to be a morphism in $\mathbf{UHCont}$, we should work on $Y^n\times \mathscr C_\Theta^\text{dens}$,
where $\mathscr C_\Theta^\text{dens}$ denotes the class---endowed with a hyperspace topology---of convex, nonempty, weak$^\star$-closed subsets of $\Delta_\Theta^\text{dens}$.
In that formulation, the maps $\mathsf{BCP}$ and $\mathsf{QUANT}_\alpha$ are recovered by restricting to singleton credal sets $\{P\}$.}
A systematic study of e-posterior-based conformal constructions from this categorical perspective is beyond the scope of the present paper, and will be the subject of future research.

\renewcommand{\IHDR}{\mathsf{IHDR}_{\alpha}}
\newcommand{\M}{\mathcal{M}}
\newcommand{\prob}{\Delta_Y}
\newcommand{\Mone}{\mathcal{M}_{1}}
\newcommand{\Mtwo}{\mathcal{M}_{2}}
\newcommand{\Mthr}{\mathcal{M}_{3}}
\newcommand{\B}[1]{B_{#1}^{\alpha}}

\section{Conformal Prediction Regions as Functor Images}\label{impr-prob-sec}

In this section, we show that,
for each $\alpha\in[0,1]$, the correspondence $\IHDR$ that takes a credal set  and returns an $\alpha$-level
IHDR in $Y$
can be written as a covariant functor between two natural inclusion categories, thus formalizing its monotonicity and compositionality:
Inclusions of credal sets propagate to inclusions of regions.
As a consequence, under consonance, the conformal prediction region appears as a functor image.

Beyond conceptual clarity, this viewpoint yields a practical privacy-compatible
aggregation principle.
In federated settings \citep{LIU2024128019}, agents can share only (possibly privatized) credal set summaries; If privatization is performed via outer approximations
$\widetilde{\mathcal M}\supseteq \mathcal M$, functoriality ensures the resulting regions are conservative (they cannot shrink),
supporting privacy-compatible aggregation at the level of summary objects rather than raw data.

\subsection{Main Result}
The following is the key finding of this section.

\begin{theorem}[Functoriality of the $ \alpha$-IHDR]\label{thm:IHDRfunctor}
	Let $(\Y,\SigY=2^Y)$ be a measurable space, and let $\prob$ be the set of all
	finitely additive probability measures on $\SigY$ endowed with the weak$^\star$ topology.
	Let
	\[
	\Ccal  \coloneqq \bigl\{ 
	    \M\subseteq\prob : \M \text{ is convex, nonempty, and weak$^\star$-closed}
	  \bigr\},
	\]
	and regard both $\Ccal$ and $\SigY$ as (pre-)categories whose 
	morphisms are inclusions ``$\subseteq$''.  
	Fix $\alpha\in[0,1]$.  For $\M\in\Ccal$, define
	\[
	   \B{\M}
	    \coloneqq 
	   \bigcap\Bigl\{
	       A\in\SigY :
	       \underline{P}_\mathcal{M}(A) \ge 1-\alpha
	   \Bigr\}.
	\]
	Set
	\[
	    \IHDR:\Ccal\rightarrow\SigY,
	    \quad
	    \M\mapsto\B{\M},
	    \qquad
	    \IHDR(\Mone\hookrightarrow\Mtwo)\coloneqq\B{\Mone}\hookrightarrow\B{\Mtwo}.
	\]
	Then, $\IHDR$ is a (covariant) functor.
\end{theorem}

\begin{hproof}
Fix $\alpha$ $\in[0,1]$. We verify that $\IHDR$ defines a functor as in Definition \ref{def:functor-back}:
It maps each object and morphism in $\Ccal$ to an object and morphism in $\SigY$, respectively, and it preserves
identities and composition.
\end{hproof}

To the best of our knowledge, 
Theorem \ref{thm:IHDRfunctor} 
is the first time that an imprecise probabilistic notion such as the IHDR is framed as a functor between two categories, thus adding to the field of category-theoretic imprecise probability theory \citep{liellcock2024compositionalimpreciseprobability}. 

We also note in passing that we do not delve into the examination of categories $(\Ccal,\subseteq)$ and $(\Sigma_Y,\subseteq)$ because they are two special instances of the well-studied category $\mathbf{Sub}(X)$ of subsets of a generic set $X$, endowed with inclusion morphisms. Here we abuse notation: We write $(\Ccal,\subseteq)$ to denote the category whose objects are the elements of $\Ccal$, and whose morphisms are the inclusions, and similarly for $(\Sigma_Y,\subseteq)$.

The next corollary, which follows immediately from Theorem \ref{thm:IHDRfunctor} and \citet[Proposition 5]{caprio2025conformalpredictionregionsimprecise}, 
allows us to see the Conformal Prediction Region outputted by a Full Conformal Prediction procedure as an object image of functor $\mathsf{IHDR}_\alpha$.

\begin{corollary}[Conformal Prediction as a Functor Image]\label{cp-image}
    Suppose that the conformal transducer $\pi(\cdot,y^n)$ is consonant, i.e. 
    $\sup_{y_{n+1}\in Y} \pi(y_{n+1},y^n)=1.$ 
    Then, for all $\alpha\in [0,1]$, all $y^n\in Y^n$, and all $\psi\in\mathscr{F}$, we have that 
    $$\mathscr{R}_\alpha^\psi(y^n) = \mathsf{IHDR}_\alpha(\mathcal{M}(\overline{\Pi})).$$
    That is, the $\alpha$-level Conformal Prediction Region $\mathscr{R}_\alpha^\psi(y^n)$ is the image under functor $\mathsf{IHDR}_\alpha$ of a well-defined credal set 
    $$\mathcal{M}(\overline{\Pi})\coloneqq \{P \in \Delta_Y : P(A) \leq \overline{\Pi}(A) \coloneqq \sup_{y_{n+1}\in A} \pi(y_{n+1},y^n) \text{, } \forall A \in \Sigma_Y \},$$
which is an object of category $(\mathscr{C},\subseteq)$.
\end{corollary}

\subsection{Implications}
Viewing $\mathsf{IHDR}_\alpha$ as a covariant functor yields a convenient privacy compatible conservativeness principle.
Suppose that multiple sites/agents produce credal sets $\mathcal M_j$.
Since $\mathsf{IHDR}_\alpha$ is monotone with respect to set inclusion, any privacy-compatible transformation that replaces $\mathcal M_j$
by a \emph{superset} $\widetilde{\mathcal M}_j\supseteq \mathcal M_j$ yields a prediction region that is conservative relative to the original one:
The region can only stay the same or enlarge; Hence, whenever the original region enjoys a coverage guarantee, this guarantee is not lost by passing to an outer approximation.\footnote{This statement is about the monotonicity of the IHDR map.
It does not assert that an arbitrary differentially-private mechanism preserves the underlying credal semantics; Rather, it gives a simple sufficient condition
($\widetilde{\mathcal M}_j\supseteq \mathcal M_j$) under which the resulting region cannot become smaller.}

Second, collaboration can proceed in a strictly federated manner: Each agent transmits only a (possibly privatized) summary object $\widetilde{\mathcal M}_j$---never the raw data---while still enabling the consortium to construct a global prediction region by aggregation at the level of credal sets. In particular, unlike decentralized learning protocols that exchange iterative model updates and can be vulnerable to reconstruction attacks \citep{el-mrini}, the present functorial construction operates at the level of outer-approximated credal summaries. Hence, the privacy claim here is best understood as an architectural reduction in exposed information, together with inclusion-monotone conservativeness of the resulting regions, rather than as a formal attack-resistance guarantee.

\textbf{Caveat.} The guarantee is inclusion-monotone.
Inner approximations $\widetilde{\mathcal M}_j\subset \mathcal M_j$ may shrink $\mathsf{IHDR}_\alpha(\widetilde{\mathcal M}_j)$ and can lead to under-coverage.
To preserve (and possibly enlarge) coverage after privatization, each site should therefore output an outer approximation.
For federation, a natural aggregation is via the closed convex hull
\[
\mathcal M_{\mathrm{agg}}
=\overline{\operatorname{co}}\Bigl(\bigcup_j \widetilde{\mathcal M}_j\Bigr),
\]
which satisfies $\widetilde{\mathcal M}_j\subseteq \mathcal M_{\mathrm{agg}}$ and hence
$\mathsf{IHDR}_\alpha(\widetilde{\mathcal M}_j)\subseteq \mathsf{IHDR}_\alpha(\mathcal M_{\mathrm{agg}})$ for every $j$.
In particular, if the true predictive law lies in some $\widetilde{\mathcal M}_j$, then it lies in $\mathcal M_{\mathrm{agg}}$.

\section{Conclusion}\label{concl}

In this paper, we took a category-theoretic route to show that Full Conformal Prediction carries an intrinsic layer of uncertainty quantification.
By framing CP as a set-valued morphism in two categories, $\mathbf{UHCont}$ (stability via upper hemicontinuity) and
$\mathbf{WMeas}_\text{uc}$ (measurability as a random compact set), we obtain a principled factorization: Under consonance and mild regularity conditions,
the conformal region coincides with the composition of a credal set extraction map and an Imprecise Highest Density Region (IHDR) map.
This factorization clarifies how CP supports uncertainty summaries beyond region size, through uncertainty functionals defined on the induced credal set.

Moreover, under additional regularity assumptions, we proved an asymptotic commutativity result showing that Bayesian predictive density level sets,
classical conformal regions, and the CP-induced imprecise construction can become asymptotically equivalent, thereby strengthening the view that CP bridges
Bayesian, frequentist, and imprecise approaches to predictive reasoning. We also quantified this convergence under local empirical process and boundary regularity assumptions, obtaining \(O_{\mathfrak P}(n^{-1/2})\) Hausdorff rates up to the leave-one-out stability term. In addition, we identified conditions under which upper posterior constructions correspond to e-posteriors, clarifying when e-value-based and conformal-imprecise representations can coincide.

Finally, we highlighted a functorial perspective in which the CPR appears as a functor image,
and discussed its implications for modularity and privacy-compatible/federated workflows at the level of summary objects.

This work opens the study of \emph{Category-Theoretic Conformal Prediction} and suggests several directions.
First, the functional-analytic structure of $\mathbf{UHCont}$ and $\mathbf{WMeas}_\text{uc}$ may yield further stability, approximation, and convergence results
for CP viewed as a set-valued statistical functional. In particular, it would be interesting to understand to what extent additional categorical constructions---e.g. Kleisli categories \citep{Kleisli1962,Kleisli1965}, \citep[Section 5.1]{perrone2024starting} arising from suitable monads on correspondence subcategories---can formalize composition rules for randomized or aggregated conformal procedures.


Second, it remains to be investigated whether the morphism/functorial viewpoint extends cleanly to Split Conformal Prediction and to other
modern variants introduced in the machine learning literature, e.g 
\citep{barber2023conformal}.

Lastly, we plan to place Conformal Prediction within the framework of Category-Theoretic Probability Theory introduced by \citet{sturtz2015categoricalprobabilitytheory}.

\section*{Acknowledgments}
We wish to express our gratitude to Sam Staton and Paolo Perrone for insightful
discussions on the possibility of formulating Conformal Prediction in a
category-theoretic framework. We are also extremely grateful to Sam Staton for
his invitation to the University of Oxford, where the idea for the present paper
germinated. We are indebted to Nicola Gambino for his assistance in improving
the proof of Theorem 4 in the Supplementary Material, and for suggesting that we
investigate the category of algebras for the Vietoris monad.

Finally, we thank Yusuf Sale and Eyke Hüllermeier for stimulating discussions
on the uncertainty quantification capabilities of Conformal Prediction, Sangwoo
Park for his idea that led to Proposition \ref{post-e-prop}, and Alessandro
Zito and Ryan Martin for comments that substantially improved the manuscript.
ChatGPT 5.5 Thinking was used to assist with proofreading.

\bibliographystyle{plainnat}
\bibliography{references}

\begin{thebibliography}{82}
\providecommand{\natexlab}[1]{#1}
\providecommand{\url}[1]{\texttt{#1}}
\expandafter\ifx\csname urlstyle\endcsname\relax
  \providecommand{\doi}[1]{doi: #1}\else
  \providecommand{\doi}{doi: \begingroup \urlstyle{rm}\Url}\fi

\bibitem[Abellán et~al.(2006)Abellán, Klir, and Moral]{abellan}
Joaquín Abellán, George~Jiří Klir, and Serafín Moral.
\newblock Disaggregated total uncertainty measure for credal sets.
\newblock \emph{International Journal of General Systems}, 1\penalty0 (35):\penalty0 29--44, 2006.

\bibitem[Ackerman et~al.(2024)Ackerman, Freer, Kaddar, Karwowski, Moss, Roy, Staton, and Yang]{Ackerman_2024}
Nate Ackerman, Cameron~E. Freer, Younesse Kaddar, Jacek Karwowski, Sean Moss, Daniel Roy, Sam Staton, and Hongseok Yang.
\newblock Probabilistic programming interfaces for random graphs: Markov categories, graphons, and nominal sets.
\newblock \emph{Proceedings of the ACM on Programming Languages}, 8\penalty0 (POPL):\penalty0 1819–1849, 2024.

\bibitem[Aliprantis and Border(2006)]{aliprantis}
Charalambos~D. Aliprantis and Kim~C. Border.
\newblock \emph{{Infinite Dimensional Analysis: a Hitchhiker's Guide}}.
\newblock Berlin : Springer, 3rd edition, 2006.

\bibitem[Angelopoulos et~al.(2024)Angelopoulos, Barber, and Bates]{angelopoulos2024theoreticalfoundationsconformalprediction}
Anastasios~N. Angelopoulos, Rina~Foygel Barber, and Stephen Bates.
\newblock Theoretical foundations of conformal prediction, 2024.

\bibitem[Aubin and Frankowska(1990)]{AubinFrankowska1990}
Jean{-}Pierre Aubin and H{\'e}l{\`e}ne Frankowska.
\newblock \emph{Set-Valued Analysis}.
\newblock Systems \& Control: Foundations \& Applications. Birkh{\"a}user Boston, Boston, MA, 1990.
\newblock ISBN 0-8176-3478-9.

\bibitem[Augustin et~al.(2014)Augustin, Coolen, De~Cooman, and Troffaes]{augustin2014introduction}
Thomas Augustin, Frank~PA Coolen, Gert De~Cooman, and Matthias~CM Troffaes.
\newblock \emph{Introduction to imprecise probabilities}, volume 591.
\newblock John Wiley \& Sons, 2014.

\bibitem[Barber et~al.(2023)Barber, Candes, Ramdas, and Tibshirani]{barber2023conformal}
Rina~Foygel Barber, Emmanuel~J Candes, Aaditya Ramdas, and Ryan~J Tibshirani.
\newblock Conformal prediction beyond exchangeability.
\newblock \emph{The Annals of Statistics}, 51\penalty0 (2):\penalty0 816--845, 2023.

\bibitem[Beer(1993)]{Beer1993}
Gerald Beer.
\newblock \emph{Topologies on Closed and Closed Convex Sets}, volume 268 of \emph{Mathematics and its Applications}.
\newblock Kluwer Academic Publishers, Dordrecht, 1993.
\newblock ISBN 0-7923-2531-1.

\bibitem[Bell and Hagood(1988)]{bell1988separation}
Wayne~C. Bell and John~W. Hagood.
\newblock Separation properties and exact radon-nikod{\'y}m derivatives for bounded finitely additive measures.
\newblock \emph{Pacific Journal of Mathematics}, 131\penalty0 (2):\penalty0 237--248, 1988.

\bibitem[Bohinen and Perrone(2025)]{bohinen2025categoricalalgebraconditionalprobability}
Mika Bohinen and Paolo Perrone.
\newblock Categorical algebra of conditional probability, 2025.

\bibitem[B\"{o}rgers(1991)]{Borgers1991UHC}
Tilman B\"{o}rgers.
\newblock Upper hemicontinuity of the correspondence of subgame-perfect equilibrium outcomes.
\newblock \emph{Journal of Mathematical Economics}, 20\penalty0 (1):\penalty0 89--106, 1991.
\newblock ISSN 0304-4068.
\newblock \doi{10.1016/0304-4068(91)90019-P}.
\newblock URL \url{https://doi.org/10.1016/0304-4068(91)90019-P}.

\bibitem[Cabezas et~al.(2025)Cabezas, Santos, Ramos, and Izbicki]{cabezas2025epistemicuncertaintyconformalscores}
Luben M.~C. Cabezas, Vagner~S. Santos, Thiago~R. Ramos, and Rafael Izbicki.
\newblock Epistemic uncertainty in conformal scores: A unified approach, 2025.

\bibitem[Caprio(2024)]{impr-MSG}
Michele Caprio.
\newblock {Imprecise Markov Semigroups and their Ergodicity }.
\newblock \emph{arXiv preprint \href{2405.00081}{arXiv:2405.00081}}, 2024.

\bibitem[Caprio et~al.(2024{\natexlab{a}})Caprio, Dutta, Jang, Lin, Ivanov, Sokolsky, and Lee]{caprio2024credal}
Michele Caprio, Souradeep Dutta, Kuk~Jin Jang, Vivian Lin, Radoslav Ivanov, Oleg Sokolsky, and Insup Lee.
\newblock {Credal Bayesian Deep Learning}.
\newblock \emph{Transactions on Machine Learning Research}, 2024{\natexlab{a}}.

\bibitem[Caprio et~al.(2024{\natexlab{b}})Caprio, Sale, H{\"u}llermeier, and Lee]{novel_bayes}
Michele Caprio, Yusuf Sale, Eyke H{\"u}llermeier, and Insup Lee.
\newblock {A Novel Bayes' Theorem for Upper Probabilities}.
\newblock In Fabio Cuzzolin and Maryam Sultana, editors, \emph{Epistemic Uncertainty in Artificial Intelligence}, pages 1--12, Cham, 2024{\natexlab{b}}. Springer Nature Switzerland.

\bibitem[Caprio et~al.(2024{\natexlab{c}})Caprio, Sultana, Elia, and Cuzzolin]{credal-learning}
Michele Caprio, Maryam Sultana, Eleni~G. Elia, and Fabio Cuzzolin.
\newblock Credal learning theory.
\newblock In \emph{Proceedings of the 38th International Conference on Neural Information Processing Systems}, NIPS '24, 2024{\natexlab{c}}.

\bibitem[Caprio et~al.(2025{\natexlab{a}})Caprio, Sale, and Hüllermeier]{caprio2025conformalpredictionregionsimprecise}
Michele Caprio, Yusuf Sale, and Eyke Hüllermeier.
\newblock {Conformal Prediction Regions are Imprecise Highest Density Regions}.
\newblock In Sébastien Destercke, Alexander Erreygers, Max Nendel, Frank Riedel, and Matthias Troffaes, editors, \emph{Proceedings of the Fourteenth International Symposium on Imprecise Probability: Theories and Applications}, pages 292--304. PMLR, 2025{\natexlab{a}}.

\bibitem[Caprio et~al.(2025{\natexlab{b}})Caprio, Stutz, Li, and Doucet]{caprio-conformalized}
Michele Caprio, David Stutz, Shuo Li, and Arnaud Doucet.
\newblock Conformalized credal regions for classification with ambiguous ground truth.
\newblock \emph{Transactions on Machine Learning Research}, 2025{\natexlab{b}}.

\bibitem[Caprio et~al.(2026)Caprio, Papagiannouli, Chau, and Mukherjee]{caprio2026robustpredictiveuncertaintydouble}
Michele Caprio, Katerina Papagiannouli, Siu~Lun Chau, and Sayan Mukherjee.
\newblock Robust predictive uncertainty and double descent in contaminated bayesian random features, 2026.
\newblock URL \url{https://arxiv.org/abs/2602.19126}.

\bibitem[Cella and Martin(2021)]{cella2021valid}
Leonardo Cella and Ryan Martin.
\newblock Valid inferential models for prediction in supervised learning problems.
\newblock In \emph{International Symposium on Imprecise Probability: Theories and Applications}, pages 72--82. PMLR, 2021.

\bibitem[Cella and Martin(2022)]{cella2022validity}
Leonardo Cella and Ryan Martin.
\newblock Validity, consonant plausibility measures, and conformal prediction.
\newblock \emph{International Journal of Approximate Reasoning}, 141:\penalty0 110--130, 2022.

\bibitem[Cella and Martin(2023)]{cella2023possibility}
Leonardo Cella and Ryan Martin.
\newblock Possibility-theoretic statistical inference offers performance and probativeness assurances.
\newblock \emph{International Journal of Approximate Reasoning}, 163:\penalty0 109060, 2023.

\bibitem[Chau et~al.(2025)Chau, Caprio, and Muandet]{chau2025integralimpreciseprobabilitymetrics}
Siu~Lun Chau, Michele Caprio, and Krikamol Muandet.
\newblock Integral imprecise probability metrics.
\newblock In \emph{Proceedings of the 38th International Conference on Neural Information Processing Systems}, NIPS '25, 2025.

\bibitem[Chau et~al.(2026)Chau, Zargarbashi, Sale, and Caprio]{chau2026quantifyingepistemicpredictiveuncertainty}
Siu~Lun Chau, Soroush~H. Zargarbashi, Yusuf Sale, and Michele Caprio.
\newblock Quantifying epistemic predictive uncertainty in conformal prediction, 2026.

\bibitem[Coolen(1992)]{coolen1992imprecise}
Franciscus Petrus~Antonius Coolen.
\newblock Imprecise highest density regions related to intervals of measures.
\newblock \emph{Memorandum COSOR; Volume 9254}, 1992.

\bibitem[Cornish(2025)]{cornish2025stochasticneuralnetworksymmetrisation}
Rob Cornish.
\newblock Stochastic neural network symmetrisation in markov categories, 2025.

\bibitem[Cornish and Wang(2026)]{cornish2026categoricalaccountmetropolishastingsalgorithm}
Rob Cornish and Andi~Q. Wang.
\newblock A categorical account of the metropolis-hastings algorithm, 2026.

\bibitem[Cuzzolin(2020)]{cuzzolin2020geometry}
Fabio Cuzzolin.
\newblock \emph{The geometry of uncertainty: the geometry of imprecise probabilities}.
\newblock Springer Nature, 2020.

\bibitem[Day(1975)]{Day1975}
Alan Day.
\newblock Filter monads, continuous lattices and closure systems.
\newblock \emph{Canadian Journal of Mathematics}, 27\penalty0 (1):\penalty0 50--59, 1975.
\newblock \doi{10.4153/CJM-1975-008-8}.

\bibitem[de~Cooman(1997{\natexlab{a}})]{deCooman1997Possibility1}
Gert de~Cooman.
\newblock Possibility theory, 1: the measure- and integral-theoretic groundwork.
\newblock \emph{International Journal of General Systems}, 25\penalty0 (4):\penalty0 291--323, 1997{\natexlab{a}}.

\bibitem[de~Cooman(1997{\natexlab{b}})]{deCooman1997Possibility2}
Gert de~Cooman.
\newblock Possibility theory, 2: conditional possibility.
\newblock \emph{International Journal of General Systems}, 25\penalty0 (4):\penalty0 325--351, 1997{\natexlab{b}}.

\bibitem[de~Cooman(1997{\natexlab{c}})]{deCooman1997Possibility3}
Gert de~Cooman.
\newblock Possibility theory, 3: possibilistic independence.
\newblock \emph{International Journal of General Systems}, 25\penalty0 (4):\penalty0 353--371, 1997{\natexlab{c}}.

\bibitem[de~Finetti(1974)]{definetti1}
Bruno de~Finetti.
\newblock \emph{{Theory of Probability}}, volume~1.
\newblock New York : Wiley, 1974.

\bibitem[de~Finetti(1975)]{definetti2}
Bruno de~Finetti.
\newblock \emph{{Theory of Probability}}, volume~2.
\newblock New York : Wiley, 1975.

\bibitem[Dubois and Prade(1988)]{DuboisPrade1988PossibilityTheory}
Didier Dubois and Henri Prade.
\newblock \emph{Possibility Theory: An Approach to Computerized Processing of Uncertainty}.
\newblock Plenum Press, New York, 1988.
\newblock ISBN 978-0-306-42520-2.

\bibitem[Dutta et~al.(2025)Dutta, Caprio, Lin, Cleaveland, Jang, Ruchkin, Sokolsky, and Lee]{inn}
Souradeep Dutta, Michele Caprio, Vivian Lin, Matthew Cleaveland, Kuk~Jin Jang, Ivan Ruchkin, Oleg Sokolsky, and Insup Lee.
\newblock Distributionally robust statistical verification with imprecise neural networks.
\newblock In \emph{Proceedings of the 28th ACM International Conference on Hybrid Systems: Computation and Control}, HSCC '25, New York, NY, USA, 2025. Association for Computing Machinery.

\bibitem[El~Mrini et~al.(2024)El~Mrini, Cyffers, and Bellet]{el-mrini}
Abdellah El~Mrini, Edwige Cyffers, and Aur\'{e}lien Bellet.
\newblock Privacy attacks in decentralized learning.
\newblock In \emph{Proceedings of the 41st International Conference on Machine Learning}, ICML'24, 2024.

\bibitem[Ensarguet and Perrone(2024)]{ensarguet2024categoricalprobabilityspacesergodic}
Noé Ensarguet and Paolo Perrone.
\newblock Categorical probability spaces, ergodic decompositions, and transitions to equilibrium, 2024.

\bibitem[Fang et~al.(2025)Fang, Li, Mulchandani, and Kim]{fang2025trustworthyaisafetybias}
Xingli Fang, Jianwei Li, Varun Mulchandani, and Jung-Eun Kim.
\newblock Trustworthy ai: Safety, bias, and privacy -- a survey, 2025.

\bibitem[Fong and Holmes(2021)]{fong-bayes}
Edwin Fong and Chris~C. Holmes.
\newblock Conformal bayesian computation.
\newblock In \emph{Proceedings of the 35th International Conference on Neural Information Processing Systems}, volume~34 of \emph{NIPS '21}, pages 18268--18279, 2021.

\bibitem[Fritz et~al.(2025)Fritz, Gonda, Lorenzin, Perrone, and Mohammed]{fritz2025empiricalmeasuresstronglaws}
Tobias Fritz, Tomáš Gonda, Antonio Lorenzin, Paolo Perrone, and Areeb~Shah Mohammed.
\newblock Empirical measures and strong laws of large numbers in categorical probability, 2025.

\bibitem[Garner(2020)]{Garner2020_Vietoris}
Richard Garner.
\newblock The {Vietoris} monad and weak distributive laws.
\newblock \emph{Applied Categorical Structures}, 28\penalty0 (2):\penalty0 339--354, April 2020.
\newblock \doi{10.1007/s10485-019-09582-w}.

\bibitem[Gibbs et~al.(2025)Gibbs, Cherian, and Candès]{gibbs2023conformal}
Isaac Gibbs, John~J Cherian, and Emmanuel~J Candès.
\newblock Conformal prediction with conditional guarantees.
\newblock \emph{Journal of the Royal Statistical Society Series B: Statistical Methodology}, 87\penalty0 (4):\penalty0 1100--1126, 03 2025.

\bibitem[Gr{\"u}nwald(2023)]{grunwald2023posterior}
Peter~D Gr{\"u}nwald.
\newblock The e-posterior.
\newblock \emph{Philosophical Transactions of the Royal Society A}, 381\penalty0 (2247):\penalty0 20220146, 2023.

\bibitem[Hanselle et~al.(2025)Hanselle, Javanmardi, Oberkofler, Sale, and H{\"u}llermeier]{hanselle2025conformal}
Jonas Hanselle, Alireza Javanmardi, Tobias~Florin Oberkofler, Yusuf Sale, and Eyke H{\"u}llermeier.
\newblock Conformal prediction without nonconformity scores.
\newblock In \emph{The 41st Conference on Uncertainty in Artificial Intelligence}, 2025.

\bibitem[Hoff(2023)]{Hoff2023BayesOptimal}
Peter Hoff.
\newblock Bayes-optimal prediction with frequentist coverage control.
\newblock \emph{Bernoulli}, 29\penalty0 (2):\penalty0 901--928, May 2023.

\bibitem[Hofman et~al.(2024)Hofman, Sale, and Hüllermeier]{hofman2024quantifyingaleatoricepistemicuncertainty}
Paul Hofman, Yusuf Sale, and Eyke Hüllermeier.
\newblock Quantifying aleatoric and epistemic uncertainty with proper scoring rules, 2024.

\bibitem[H{\"u}llermeier and Waegeman(2021)]{Hullermeier2021Aleatoric}
Eyke H{\"u}llermeier and Willem Waegeman.
\newblock Aleatoric and epistemic uncertainty in machine learning: an introduction to concepts and methods.
\newblock \emph{Machine Learning}, 110\penalty0 (3):\penalty0 457--506, March 2021.

\bibitem[Javanmardi et~al.(2023)Javanmardi, Sale, Hofman, and H\"ullermeier]{alireza2}
Alireza Javanmardi, Yusuf Sale, Paul Hofman, and Eyke H\"ullermeier.
\newblock Conformal prediction with partially labeled data.
\newblock In Harris Papadopoulos, Khuong~An Nguyen, Henrik Boström, and Lars Carlsson, editors, \emph{Proceedings of the Twelfth Symposium on Conformal and Probabilistic Prediction with Applications}, volume 204 of \emph{Proceedings of Machine Learning Research}, pages 251--266. PMLR, 9 2023.
\newblock URL \url{https://proceedings.mlr.press/v204/javanmardi23a.html}.

\bibitem[Javanmardi et~al.(2024)Javanmardi, Stutz, and Hüllermeier]{alireza}
Alireza Javanmardi, David Stutz, and Eyke Hüllermeier.
\newblock Conformalized credal set predictors.
\newblock In \emph{Proceedings of the 38th International Conference on Neural Information Processing Systems}. NeurIPS, 2024.

\bibitem[Kleisli(1962)]{Kleisli1962}
Heinrich Kleisli.
\newblock Homotopy theory in abelian categories.
\newblock \emph{Canadian Journal of Mathematics}, 14:\penalty0 139--169, 1962.

\bibitem[Kleisli(1965)]{Kleisli1965}
Heinrich Kleisli.
\newblock Every standard construction is induced by a pair of adjoint functors.
\newblock \emph{Proceedings of the American Mathematical Society}, 16\penalty0 (3):\penalty0 544--546, 1965.

\bibitem[Korolev(2022)]{Korolev}
Yury Korolev.
\newblock Two-layer neural networks with values in a banach space.
\newblock \emph{SIAM Journal on Mathematical Analysis}, 54\penalty0 (6):\penalty0 6358--6389, 2022.

\bibitem[Levi(1980)]{levi2}
Isaac Levi.
\newblock \emph{The Enterprise of Knowledge}.
\newblock London, UK : MIT Press, 1980.

\bibitem[Liell-Cock and Staton(2025)]{liellcock2024compositionalimpreciseprobability}
Jack Liell-Cock and Sam Staton.
\newblock Compositional imprecise probability: A solution from graded monads and markov categories.
\newblock \emph{Proc. ACM Program. Lang.}, 9\penalty0 (POPL), 2025.

\bibitem[Liu et~al.(2024)Liu, Lv, Guo, and Li]{LIU2024128019}
Bingyan Liu, Nuoyan Lv, Yuanchun Guo, and Yawen Li.
\newblock Recent advances on federated learning: A systematic survey.
\newblock \emph{Neurocomputing}, 597:\penalty0 128019, 2024.

\bibitem[Martin(2021)]{martin2021impreciseprobabilisticcharacterizationfrequentiststatistical}
Ryan Martin.
\newblock An imprecise-probabilistic characterization of frequentist statistical inference, 2021.

\bibitem[Martin(2023{\natexlab{a}})]{martin2022valid}
Ryan Martin.
\newblock Valid and efficient imprecise-probabilistic inference with partial priors, ii. general framework, 2023{\natexlab{a}}.

\bibitem[Martin(2023{\natexlab{b}})]{martin2023validefficientimpreciseprobabilisticinference}
Ryan Martin.
\newblock Valid and efficient imprecise-probabilistic inference with partial priors, iii. marginalization, 2023{\natexlab{b}}.

\bibitem[Martin(2024)]{Martin2024Possibility}
Ryan Martin.
\newblock A possibility-theoretic solution to {Basu}'s bayesian--frequentist via media.
\newblock \emph{Sankhya A}, 86\penalty0 (Suppl 1):\penalty0 43--70, 2024.

\bibitem[Martin(2025{\natexlab{a}})]{martin2025efficientmontecarlomethod}
Ryan Martin.
\newblock An efficient monte carlo method for valid prior-free possibilistic statistical inference, 2025{\natexlab{a}}.

\bibitem[Martin(2025{\natexlab{b}})]{martin2025nopriorbayesreimaginedprobabilistic}
Ryan Martin.
\newblock No-prior bayes reimagined: probabilistic approximations of inferential models, 2025{\natexlab{b}}.

\bibitem[Martin(2025{\natexlab{c}})]{martin2025regularizedeprocessesanytimevalid}
Ryan Martin.
\newblock Regularized e-processes: anytime valid inference with knowledge-based efficiency gains.
\newblock \emph{Accepted at Bernoulli}, 2025{\natexlab{c}}.
\newblock URL \url{https://arxiv.org/abs/2410.01427}.

\bibitem[Martin et~al.(2025)Martin, Prim, and Williams]{martin2025decisionmakingpossibilisticinferentialmodels}
Ryan Martin, Shih-Ni Prim, and Jonathan Williams.
\newblock Decision-making with possibilistic inferential models, 2025.

\bibitem[Matache et~al.(2022)Matache, Moss, and Staton]{Matache_2022}
Cristina Matache, Sean Moss, and Sam Staton.
\newblock Concrete categories and higher-order recursion: With applications including probability, differentiability, and full abstraction.
\newblock In \emph{Proceedings of the 37th Annual ACM/IEEE Symposium on Logic in Computer Science}, LICS ’22, page 1–14. ACM, August 2022.

\bibitem[Papadopoulos(2008)]{papadopoulos2008inductive}
Harris Papadopoulos.
\newblock Inductive conformal prediction: Theory and application to neural networks.
\newblock In \emph{Tools in artificial intelligence}. Citeseer, 2008.

\bibitem[Perrone(2024)]{perrone2024starting}
Paolo Perrone.
\newblock \emph{Starting Category Theory}.
\newblock World Scientific, 2024.

\bibitem[Sale et~al.(2023)Sale, Caprio, and H\"{u}llermeier]{pmlr-v216-sale23a}
Yusuf Sale, Michele Caprio, and Eyke H\"{u}llermeier.
\newblock Is the volume of a credal set a good measure for epistemic uncertainty?
\newblock In Robin~J. Evans and Ilya Shpitser, editors, \emph{Proceedings of the Thirty-Ninth Conference on Uncertainty in Artificial Intelligence}, volume 216 of \emph{Proceedings of Machine Learning Research}, pages 1795--1804. PMLR, 8 2023.
\newblock URL \url{https://proceedings.mlr.press/v216/sale23a.html}.

\bibitem[Sale et~al.(2024)Sale, Bengs, Caprio, and H\"{u}llermeier]{second-order}
Yusuf Sale, Viktor Bengs, Michele Caprio, and Eyke H\"{u}llermeier.
\newblock Second-order uncertainty quantification: A distance-based approach.
\newblock In Ruslan Salakhutdinov, Zico Kolter, Katherine Heller, Adrian Weller, Nuria Oliver, Jonathan Scarlett, and Felix Berkenkamp, editors, \emph{Proceedings of the 41st International Conference on Machine Learning}, volume 235 of \emph{Proceedings of Machine Learning Research}, pages 43060--43076. PMLR, 7 2024.
\newblock URL \url{https://proceedings.mlr.press/v235/sale24a.html}.

\bibitem[Shafer and Vovk(2008)]{shafer2008tutorial}
Glenn Shafer and Vladimir Vovk.
\newblock A tutorial on conformal prediction.
\newblock \emph{Journal of Machine Learning Research}, 9\penalty0 (3), 2008.

\bibitem[Shiebler et~al.(2021)Shiebler, Gavranović, and Wilson]{shiebler2021categorytheorymachinelearning}
Dan Shiebler, Bruno Gavranović, and Paul Wilson.
\newblock Category theory in machine learning, 2021.

\bibitem[Snell and Griffiths(2025)]{snell2025conformal}
Jake~C. Snell and Thomas~L. Griffiths.
\newblock Conformal prediction as bayesian quadrature.
\newblock In \emph{Forty-second International Conference on Machine Learning}, 2025.

\bibitem[Song(2024)]{song2024describe}
Congwei Song.
\newblock To describe or construct statistical learning models using the category-theoretical language.
\newblock \emph{BIMSA Technical Report}, 2024.

\bibitem[Stehlík(2012)]{stehlik2012category}
Milan Stehlík.
\newblock Category theory in statistical learning?
\newblock \emph{Proceedings of the 33rd Linz Seminar of Fuzzy Set Theory}, page~56, 2012.

\bibitem[Sturtz(2015)]{sturtz2015categoricalprobabilitytheory}
Kirk Sturtz.
\newblock Categorical probability theory, 2015.

\bibitem[Stutts et~al.(2024)Stutts, Kumar, Tulabandhula, and Trivedi]{10.1145/3649329.3663512}
Alex~Christopher Stutts, Divake Kumar, Theja Tulabandhula, and Amit Trivedi.
\newblock Invited: Conformal inference meets evidential learning: Distribution-free uncertainty quantification with epistemic and aleatoric separability.
\newblock In \emph{Proceedings of the 61st ACM/IEEE Design Automation Conference}, DAC '24, New York, NY, USA, 2024. Association for Computing Machinery.
\newblock ISBN 9798400706011.

\bibitem[Troffaes and de~Cooman(2014)]{TroffaesDeCooman2014}
Matthias C.~M. Troffaes and Gert de~Cooman.
\newblock \emph{Lower Previsions}.
\newblock Wiley Series in Probability and Statistics. John Wiley \& Sons, Chichester, United Kingdom, 2014.

\bibitem[Vovk(2013)]{vovk-trans}
Vladimir Vovk.
\newblock Transductive conformal predictors.
\newblock In Harris Papadopoulos, Andreas~S. Andreou, Lazaros Iliadis, and Ilias Maglogiannis, editors, \emph{Artificial Intelligence Applications and Innovations}, pages 348--360, Berlin, Heidelberg, 2013. Springer Berlin Heidelberg.
\newblock ISBN 978-3-642-41142-7.

\bibitem[Vovk et~al.(2005)Vovk, Gammerman, and Shafer]{vovk2005algorithmic}
Vladimir Vovk, Alexander Gammerman, and Glenn Shafer.
\newblock \emph{Algorithmic learning in a random world}, volume~29.
\newblock Springer, 2005.

\bibitem[Walley(1991)]{walley1991statistical}
Peter Walley.
\newblock \emph{Statistical Reasoning with Imprecise Probabilities}, volume~42 of \emph{Monographs on Statistics and Applied Probability}.
\newblock Chapman and Hall, London, 1991.

\bibitem[Wu et~al.(2026)Wu, Lohmanova, Kaski, and Caprio]{wu2026bayesianconformalpredictiondecision}
Fanyi Wu, Veronika Lohmanova, Samuel Kaski, and Michele Caprio.
\newblock Bayesian conformal prediction as a decision risk problem, 2026.

\bibitem[Wyler(1981)]{Wyler1981}
Oswald Wyler.
\newblock Algebraic theories of continuous lattices.
\newblock In Gerhard Gierz, Karl~H. Hofmann, Klaus Keimel, Jimmie~D. Lawson, Michael~W. Mislove, and Dana~S. Scott, editors, \emph{Continuous Lattices}, volume 871 of \emph{Lecture Notes in Mathematics}, pages 390--413. Springer, Berlin, 1981.
\newblock \doi{10.1007/BFb0089921}.

\end{thebibliography}

\newpage
\appendix

\setcounter{theorem}{0}
\setcounter{corollary}{0}
\setcounter{remark}{0}
\setcounter{example}{0}

\section*{{\bf SUPPLEMENTARY MATERIAL}}

\section{Conformal Prediction as a Weakly measurable Morphism}
This section shows that $\mathbf{WMeas}_{\text{uc}}$ is a category, and proves that, under minimal assumptions, the (Full) Conformal Prediction correspondence $\kappa_\alpha$, $\alpha \in [0,1]$, is a morphism of $\mathbf{WMeas}_{\text{uc}}$ as well, embedded in a commuting diagram. 

Recall that $\mathbf{WMeas}_{\text{uc}}$ is a structure with

\begin{itemize}
  \item \textbf{Objects:} Compact Polish spaces $(X,\Sigma_X)$, with $\Sigma_X=\mathcal{B}(X)$, the Borel $\sigma$-algebra.
  \item \textbf{Morphisms:} Weakly measurable (see Definition \ref{w-mble-definition} in Appendix \ref{funct-anal}), uniformly compact-valued (see Definition \ref{unif-cpct-val-definition} in Appendix \ref{funct-anal}) 
        correspondences $\Phi: X\rightrightarrows Y$.\footnote{\label{foot1}For simplicity, here we implicitly assume that $|\Phi(x)|<\infty$, otherwise we would also need upper hemicontinuity of the morphisms, and we would end up working with $\mathbf{WMeas}_{\text{uc,uhc}}$, introduced in Appendix \ref{functors-sec}.}
\end{itemize}


\begin{theorem}[$\mathbf{WMeas}_{\text{uc}}$ is a category]
\label{meas-corr-cat}
The structure $\mathbf{WMeas}_{\text{uc}}$ defines a category.
\end{theorem}

\begin{proof}
We verify the four axioms.

\textbf{(1) and (4) Identity morphisms and Identity laws.}
For any object $X$, $\operatorname{id}_X(x)=\{x\}$ is nonempty compact.
Put $K_{\operatorname{id}_X}=X$ (compact by hypothesis); The uniformly compact-valued condition holds.
For every open $O\subseteq X$ we have
$\operatorname{id}_X^{-1}(O)=O\in\Sigma_X$, so $\operatorname{id}_X$
is a morphism.  The usual set-theoretic equalities
$\Phi\circ\operatorname{id}_X=\Phi$ and
$\operatorname{id}_Y\circ\Phi=\Phi$ prove the identity laws.

\textbf{(2) Associativity.}
For morphisms
$\Phi\colon X\rightrightarrows Y$,
$\Psi\colon Y\rightrightarrows Z$ and
$\Theta\colon Z\rightrightarrows W$,
the associativity of union gives
$\bigl((\Theta \circ \Psi)\circ\Phi\bigr)(x)
      =(\Theta \circ (\Psi \circ \Phi))(x)$,
for every $x\in X$.

\textbf{(3) Closure under composition.}
Let $\Phi\colon X\rightrightarrows Y$ and
$\Psi\colon Y\rightrightarrows Z$ be morphisms.
Write $K_\Phi\subseteq Y$ and $K_\Psi\subseteq Z$ for the uniform compact bounds, see Definition \ref{def:ucv}.(ii) in Appendix \ref{funct-anal}.

\emph{Compactness and finite-valuedness of fibers.}
For each \(x\in X\),
\[
  (\Psi\circ\Phi)(x)
     =\bigcup_{y\in\Phi(x)}\Psi(y).
\]
By the finite-valued convention, \(\Phi(x)\) is finite. Since each
\(\Psi(y)\) is compact, \((\Psi\circ\Phi)(x)\) is a finite union of compact
sets, hence compact. Moreover, since each \(\Psi(y)\) is finite-valued, the
composite is finite-valued. Finally,
$(\Psi\circ\Phi)(x)\subseteq K_\Psi$
for all \(x\), so the uniformly compact-valued condition holds with the same
compact bound \(K_\Psi\).

\emph{Weak measurability.} Fix an open $O\subseteq Z$ and set $H\coloneqq\Psi\circ\Phi$.
Since $X$ is standard Borel and $Y$ is Polish, $\Phi$ being weakly
measurable with closed values implies that its graph
$\operatorname{Gr}(\Phi)\subset X\times Y$ is Borel. By the
Lusin-Novikov Uniformization Theorem, there exist Borel maps
\(f_j:X\to Y\), \(j\in\mathbb N\), such that, for each \(x\),
$\Phi(x)=\{f_j(x):j\in\mathbb N\}$.

Then, for every $x$,
\[
H(x)\cap O\neq\emptyset
\iff \exists j \text{ such that }\Psi\bigl(f_j(x)\bigr)\cap O\neq\emptyset,
\]
and therefore
\[
H^{-1}(O)
=\bigcup_{j=1}^{\infty} f_j^{-1}\bigl(\Psi^{-1}(O)\bigr).
\]
Since each $f_j$ is Borel and $\Psi$ is weakly measurable,
$\Psi^{-1}(O)\in\Sigma_Y$ and hence
$f_j^{-1}\bigl(\Psi^{-1}(O)\bigr)$ $\in$ $\Sigma_X$.
A countable union of Borel sets is Borel, so $H^{-1}(O)\in\Sigma_X$.
Thus $H$ is weakly measurable, completing closure. Therefore, $\mathbf{WMeas}_{\text{uc}}$ satisfies the axioms of a category.
\end{proof}


Similarly to what we did in Section 3.2 in the Main Text, we now give simple sufficient conditions under which the correspondences $\kappa$, $\CRED$, and $\IHDR$ are uniformly compact-valued and weakly measurable. As we shall see in Theorem \ref{conf-pred-diagr2}, these conditions parallel those in Section 3.2 in the Main Text, with (i), (ii), and (iv) strengthened to

\begin{enumerate}
\renewcommand{\labelenumi}{(\roman{enumi}')}
  \renewcommand{\theenumi}{(\roman{enumi}')}
  \item\label{H1-prime}
    $\Y$ is compact metric and $\SigY$ is its Borel $\sigma$-field.
\item\label{H2-prime}
    Each $\psi\in\F$ is jointly continuous on $\Y^{n}\times\Y$;
    $\F$ is compact Polish in the uniform ($\sup$-norm) topology.
\end{enumerate}
\begin{enumerate}
\renewcommand{\labelenumi}{(\roman{enumi}')}
  \renewcommand{\theenumi}{(\roman{enumi}')}
  \setcounter{enumi}{3}
\item\label{H5-prime} When working with $\mathbf{WMeas}_\text{uc}$,
$\Delta_Y$ denotes the set of {\em countably additive} probability
measures on $Y$, endowed with the weak$^\star$ topology inherited from $C(Y)^\star$,\footnote{This is needed to make $\Delta_Y$ metrizable: see proof of weak measurability of $\CRED$ in the next page.}
\end{enumerate}
respectively. Since $Y$ is compact metric, $C(Y)$ is separable; Hence $\Delta_Y$ is compact metrizable. As a consequence, $\mathscr C$ is compact
metrizable (the Vietoris topology coincides with the Hausdorff metric topology on
nonempty closed subsets of a compact metric space).

In Remark \ref{rk:minimality2}, we show that \ref{H1-prime} is needed to make $\IHDR$ weakly measurable, and \ref{H2-prime} to make $Y^n \times \F$ Polish. We are now ready for the main result of this section.

\begin{theorem}[The Conformal Prediction Diagram Commutes in \(\mathbf{WMeas}_{\mathrm{uc}}\)]
\label{conf-pred-diagr2}
Work under the finite-valued convention stated in footnote \ref{foot1}. Fix
\(\alpha\in[0,1]\setminus S_{n+1}\). Assume \ref{H1-prime}, \ref{H2-prime},
and \ref{H5-prime}. Assume also that the conformal transducer is consonant.

Let \(X\coloneqq Y^n\times\mathscr F\), and define
\[
\kappa_\alpha:X\rightrightarrows Y,
\qquad
\kappa_\alpha(y^n,\psi)
\coloneqq
\{y\in Y:\pi_\psi(y,y^n)>\alpha\}.
\]
Let
\[
c_{\mathrm{CP}}:X\to\mathscr C,
\qquad
c_{\mathrm{CP}}(y^n,\psi)\coloneqq \CRED(y^n,\psi),
\]
be the credal-object map, viewed as a singleton-valued correspondence
\(X\rightrightarrows\mathscr C\). Finally, let
\[
\IHDR:\mathscr C\rightrightarrows Y,
\qquad
\IHDR(\mathcal M)
\coloneqq
\bigcap\{A\in\Sigma_Y:\underline P_{\mathcal M}(A)\ge 1-\alpha\}.
\]
Then \(\kappa_\alpha\), \(c_{\mathrm{CP}}\), and \(\IHDR\) are morphisms in
\(\mathbf{WMeas}_{\mathrm{uc}}\), and the diagram
\[
\begin{tikzcd}
X \ar{r}{c_{\mathrm{CP}}} \ar{dr}[swap]{\kappa_\alpha}
&
\mathscr C \ar{d}{\IHDR}
\\
&  Y
\end{tikzcd}
\]
commutes in \(\mathbf{WMeas}_{\mathrm{uc}}\), that is,
$\IHDR\circ c_{\mathrm{CP}}=\kappa_\alpha$ .
\end{theorem}

Unlike in the \(\mathbf{UHCont}\) factorization, where we restrict the middle
object to \(\mathscr M_{\mathrm{CP}}=\operatorname{Im}(c)\) endowed with the
final topology induced by \(c\), the weak-measurability argument above applies
directly to the standalone correspondence \(\IHDR:\mathscr C\rightrightarrows Y\).
Thus the diagram in \(\mathbf{WMeas}_{\mathrm{uc}}\) may be written with the full
credal set object \(\mathscr C\).

\begin{proof}[Proof of Theorem \ref{conf-pred-diagr2}]
We first note that the relevant spaces are objects of
\(\mathbf{WMeas}_{\mathrm{uc}}\). By \ref{H1-prime}, \(Y\) is compact metric,
hence compact Polish, and so \(Y^n\) is compact Polish. By \ref{H2-prime},
\(\mathscr F\) is compact Polish, and therefore
\(X=Y^n\times\mathscr F\) is compact Polish. By \ref{H5-prime}, \(\Delta_Y\)
is compact metrizable. The hyperspace of nonempty closed subsets of
\(\Delta_Y\), endowed with the Vietoris topology, is compact metrizable; the
subspace \(\mathscr C\) of nonempty closed convex subsets is closed in this
hyperspace, hence compact Polish.

We now prove that the arrows are morphisms.

\textbf{1. The conformal correspondence \(\kappa_\alpha\).}
Since \(\alpha\notin S_{n+1}\), put
\[
\beta(\alpha)
\coloneqq
\min\left\{\frac{k}{n+1}\in S_{n+1}:\frac{k}{n+1}>\alpha\right\}.
\]
Because \(\pi_\psi(y,y^n)\) takes values in \(S_{n+1}\),
\[
\kappa_\alpha(y^n,\psi)
=
\{y\in Y:\pi_\psi(y,y^n)\ge \beta(\alpha)\}.
\]
By the same argument as in Lemma 9 in the Main Text, the map
\((y,y^n,\psi)\mapsto \pi_\psi(y,y^n)\) is upper semicontinuous. Hence
\(\kappa_\alpha(y^n,\psi)\) is closed in the compact space \(Y\), and therefore
compact. Uniform compact-valuedness holds with compact bound \(Y\). Under the
finite-valued convention, the values are also finite.

It remains to prove weak measurability. Let \(G\subseteq Y\) be open. Let \(\{\overline B_k\}_{k\in\mathbb N}\) be a countable family of closed balls
with centers in a fixed countable dense subset of \(Y\) and rational radii such
that every point of every open set is contained in some
\(\overline B_k\subseteq G\). Then
\(\kappa_\alpha(y^n,\psi)\cap G\neq\emptyset\) if and only if there exists
\(k\) with \(\overline B_k\subseteq G\) such that
\[
\sup_{y\in\overline B_k}\pi_\psi(y,y^n)\ge \beta(\alpha).
\]
For each \(k\), the map
\(u_k(y^n,\psi)\coloneqq \sup_{y\in\overline B_k}\pi_\psi(y,y^n)\)
is upper semicontinuous by Lemma 9 in the Main Text, hence Borel. Therefore
\[
\{(y^n,\psi):\kappa_\alpha(y^n,\psi)\cap G\neq\emptyset\}
=
\bigcup_{\overline B_k\subseteq G}
\{(y^n,\psi):u_k(y^n,\psi)\ge \beta(\alpha)\}
\]
is Borel. Thus \(\kappa_\alpha\) is weakly measurable.

\textbf{2. The credal-object map \(c_{\mathrm{CP}}\).}
It is enough to show that the compact-valued correspondence
\[
\Gamma:X\rightrightarrows \Delta_Y,
\qquad
\Gamma(y^n,\psi)=\CRED(y^n,\psi),
\]
is weakly measurable. First, \(\Gamma(x)\in\mathscr C\) for every
\(x=(y^n,\psi)\in X\). Convexity is immediate. Nonemptiness follows from
consonance: since \(Y\) is compact and \(y\mapsto\pi_\psi(y,y^n)\) is upper
semicontinuous, there exists \(y^\star\in Y\) with
\(\pi_\psi(y^\star,y^n)=1\), and then \(\delta_{y^\star}\in\CRED(x)\).

We next show that \(\Gamma(x)\) is weakly closed. Let \(P_m\in\Gamma(x)\) and
\(P_m\Rightarrow P\) weakly. Fix a closed set \(F\subseteq Y\). For every open
\(G\supseteq F\), Portmanteau gives \(P(F)\le P(G)\le \liminf_m P_m(G)\). By
inner regularity of the countably additive probabilities \(P_m\), and since
\(P_m\) is dominated on closed sets,
\[
P_m(G)
=
\sup_{K\subseteq G,\ K\text{ compact}}P_m(K)
\le
\sup_{K\subseteq G,\ K\text{ compact}}\sup_{y\in K}\pi_\psi(y,y^n)
\le
\sup_{y\in G}\pi_\psi(y,y^n).
\]
Therefore \(P(F)\le \sup_{y\in G}\pi_\psi(y,y^n)\) for every open
\(G\supseteq F\). Taking the infimum over such \(G\), and using upper
semicontinuity of \(y\mapsto\pi_\psi(y,y^n)\), yields
\[
P(F)\le \sup_{y\in F}\pi_\psi(y,y^n).
\]
Thus \(P\in\Gamma(x)\), so \(\Gamma(x)\) is weakly closed.

Since \(\Gamma(x)\) is a closed subset of the compact space \(\Delta_Y\), it is
compact. Thus \(\Gamma\) is compact-valued. Moreover, for compact-valued
correspondences into a compact metric space, weak measurability is equivalent
to Borel measurability of the associated hyperspace-valued map. Thus weak
measurability of \(\Gamma\) implies that
\(c_{\mathrm{CP}}:X\to\mathscr C\) is Borel. Since \(c_{\mathrm{CP}}\) is
singleton-valued and \(\mathscr C\) is compact, it is uniformly compact-valued;
under the finite-valued convention it is a morphism of
\(\mathbf{WMeas}_{\mathrm{uc}}\).

We prove weak measurability of \(\Gamma\). Let \(G\subseteq\Delta_Y\) be open.
Since \(\Delta_Y\) is compact metric, there exist compact sets \(K_m\subseteq G\)
with \(K_m\uparrow G\). Hence
\[
\{\Gamma\cap G\neq\emptyset\}
=
\bigcup_{m=1}^\infty
\{\Gamma\cap K_m\neq\emptyset\},
\]
so it suffices to show that each set on the right is Borel.

Let \(\mathscr K=(K^j)_{j\in\mathbb N}\) be the countable family of all finite
unions of closed balls in \(Y\) with centers in a fixed countable dense subset
of \(Y\) and rational radii. For
each \(j\), set
\[
u_j(y^n,\psi)\coloneqq \sup_{y\in K^j}\pi_\psi(y,y^n).
\]
For \(\ell\in\mathbb N\), define the closed thickening
\(K^{j,\ell}\coloneqq\{y\in Y:\operatorname{dist}(y,K^j)\le 1/\ell\}\), and set
\[
v_{j,\ell}(y^n,\psi)
\coloneqq
\sup_{y\in K^{j,\ell}}\pi_\psi(y,y^n).
\]
Each \(v_{j,\ell}\) is upper semicontinuous by Lemma 9 in the Main Text. Choose
continuous functions \(\phi_{j,\ell}:Y\to[0,1]\) such that
\(\phi_{j,\ell}\downarrow\mathbbm 1_{K^j}\) pointwise and
\(\operatorname{supp}(\phi_{j,\ell})\subseteq K^{j,\ell}\); for instance,
\[
\phi_{j,\ell}(y)=\max\{0,1-\ell\operatorname{dist}(y,K^j)\}.
\]

For fixed \(m,r,\ell\), define
\[
C_{m,r,\ell}
\coloneqq
\Bigl\{
((y^n,\psi),P)\in X\times K_m:
\int\phi_{j,\ell}\,dP\le v_{j,\ell}(y^n,\psi),
\ 1\le j\le r
\Bigr\}.
\]
Each \(C_{m,r,\ell}\) is closed. Indeed, if
\(((y_q^n,\psi_q),P_q)\to((y^n,\psi),P)\) and
\(\int\phi_{j,\ell}\,dP_q\le v_{j,\ell}(y_q^n,\psi_q)\), then continuity of
\(P\mapsto\int\phi_{j,\ell}\,dP\) and upper semicontinuity of \(v_{j,\ell}\)
give
\[
\int\phi_{j,\ell}\,dP
\le
\limsup_q v_{j,\ell}(y_q^n,\psi_q)
\le
v_{j,\ell}(y^n,\psi).
\]
Set
\[
C_{m,r}\coloneqq\bigcap_{\ell=1}^\infty C_{m,r,\ell},
\qquad
\Pi_{m,r}\coloneqq \operatorname{pr}_X(C_{m,r}).
\]
Since \(X\times K_m\) is compact and \(C_{m,r}\) is closed,
\(\Pi_{m,r}\) is compact, hence Borel.

For \(x=(y^n,\psi)\), define
\[
S_{m,r}(x)\coloneqq\{P\in K_m:(x,P)\in C_{m,r}\}.
\]
The sets \(S_{m,r}(x)\) are decreasing in \(r\), and
\[
\bigcap_{r=1}^\infty S_{m,r}(x)=\CRED(x)\cap K_m.
\]
The inclusion \(\supseteq\) follows because, if \(P\in\CRED(x)\), then
\(\phi_{j,\ell}\le\mathbbm 1_{K^{j,\ell}}\), so
\[
\int\phi_{j,\ell}\,dP
\le
P(K^{j,\ell})
\le
v_{j,\ell}(x).
\]
Conversely, if \(P\in\bigcap_r S_{m,r}(x)\), then for every \(j,\ell\),
\[
\int\phi_{j,\ell}\,dP\le v_{j,\ell}(x).
\]
Letting \(\ell\to\infty\), bounded convergence gives
\(P(K^j)\le u_j(x)\) for all \(j\). Now fix any closed \(F\subseteq Y\), and
let \(G\supseteq F\) be open. By compactness of \(F\) and the construction of
\(\mathscr K\), there exists \(K^j\in\mathscr K\) such that
\(F\subseteq K^j\subseteq G\). Hence
\[
P(F)\le P(K^j)\le u_j(x)\le \sup_{y\in G}\pi_\psi(y,y^n).
\]
Taking the infimum over open \(G\supseteq F\), and using upper semicontinuity
of \(y\mapsto\pi_\psi(y,y^n)\), yields
\[
P(F)\le \sup_{y\in F}\pi_\psi(y,y^n).
\]
Thus \(P\in\CRED(x)\); Indeed, if \(x\in\bigcap_{r=1}^\infty\Pi_{m,r}\), then the compact sets
\(S_{m,r}(x)\subseteq K_m\) are nonempty and decreasing in \(r\). Hence their
intersection is nonempty, and by the identity above this intersection is
\(\CRED(x)\cap K_m\). Therefore
\[
\{x:\CRED(x)\cap K_m\neq\emptyset\}
=
\bigcap_{r=1}^\infty \Pi_{m,r},
\]
which is Borel. Hence \(\Gamma\) is weakly measurable, and consequently
\(c_{\mathrm{CP}}\) is a morphism in \(\mathbf{WMeas}_{\mathrm{uc}}\).

\textbf{3. The IHDR correspondence.}
For \((\mathcal M,y)\in\mathscr C\times Y\), define
\[
T(\mathcal M,y)\coloneqq \sup_{P\in\mathcal M}P(\{y\}).
\]
We first claim that
\[
y\in\IHDR(\mathcal M)
\quad\Longleftrightarrow\quad
T(\mathcal M,y)>\alpha.
\]
Indeed, by definition,
\[
\IHDR(\mathcal M)
=
\bigcap\{A\in\Sigma_Y:\underline P_{\mathcal M}(A)\ge1-\alpha\}.
\]
If \(y\in\IHDR(\mathcal M)\), then \(Y\setminus\{y\}\) cannot satisfy
\(\underline P_{\mathcal M}(Y\setminus\{y\})\ge1-\alpha\). Hence
\(\overline P_{\mathcal M}(\{y\})>\alpha\), i.e. \(T(\mathcal M,y)>\alpha\).
The converse follows because, if \(T(\mathcal M,y)>\alpha\) and \(A\) does not
contain \(y\), then \(\{y\}\subseteq A^c\), so
\[
\underline P_{\mathcal M}(A)
=
1-\overline P_{\mathcal M}(A^c)
<
1-\alpha.
\]
Thus every \(A\) with \(\underline P_{\mathcal M}(A)\ge1-\alpha\) must contain
\(y\).

Next, \(T\) is upper semicontinuous. The map
\[
f:\Delta_Y\times Y\to[0,1],
\qquad
f(P,y)=P(\{y\}),
\]
is upper semicontinuous: if \((P_m,y_m)\to(P,y)\), then
\(P_m\otimes\delta_{y_m}\Rightarrow P\otimes\delta_y\), and the diagonal in
\(Y\times Y\) is closed, so the Portmanteau theorem gives
\[
\limsup_m P_m(\{y_m\})\le P(\{y\}).
\]
Since \(T(\mathcal M,y)=\sup_{P\in\mathcal M}f(P,y)\), compactness of
\(\mathcal M\) and the Vietoris topology imply that \(T\) is upper
semicontinuous on \(\mathscr C\times Y\).

Let \(G\subseteq Y\) be open. Choose compact sets \(L_m\subseteq G\) with
\(L_m\uparrow G\), and define
\[
R_m(\mathcal M)\coloneqq \sup_{y\in L_m}T(\mathcal M,y).
\]
Each \(R_m\) is upper semicontinuous, hence Borel. Therefore
\[
\{\mathcal M:\IHDR(\mathcal M)\cap G\neq\emptyset\}
=
\bigcup_{m=1}^\infty
\{\mathcal M:R_m(\mathcal M)>\alpha\}
\]
is Borel. Thus \(\IHDR\) is weakly measurable. Under the finite-valued
convention, its values are finite subsets of the compact metric space \(Y\),
hence compact. Uniform compact-valuedness holds with bound \(Y\). Therefore
\(\IHDR\) is a morphism of \(\mathbf{WMeas}_{\mathrm{uc}}\).

Finally, by Proposition 4 in the Main Text, under consonance,
\[
\IHDR(c_{\mathrm{CP}}(y^n,\psi))
=
\kappa_\alpha(y^n,\psi)
\]
for every \((y^n,\psi)\in X\). Hence the displayed diagram commutes in
\(\mathbf{WMeas}_{\mathrm{uc}}\).
\end{proof}

The same considerations that we put forth in Section 3.2 in the Main Text on the compactness of $Y$, and in Section 3.3 in the Main Text on the importance of Theorem 7 in the Main Text, still hold here for Theorem \ref{conf-pred-diagr2}. 

The following remark explains the role of the assumptions in Theorem
\ref{conf-pred-diagr2}, and shows that removing them can destroy either the
object structure, uniform compact-valuedness, or weak measurability needed in
\(\mathbf{WMeas}_{\mathrm{uc}}\).

\begin{remark}[Role of Hypotheses \ref{H1-prime}, \ref{H2-prime}, (iii) in Section 3.2 in the Main Text, and \ref{H5-prime}]
\label{rk:minimality2}
We discuss why each hypothesis is structurally needed for the simultaneous
uniform compact-valuedness and weak measurability of the correspondences
\(\kappa\), \(\CRED\), and \(\IHDR\).

\begin{itemize}
\item\label{noH1-prime}
\textbf{Dropping \ref{H1-prime}.}  
Let \(Y=\mathbb R\), \(n=1\), \(\psi(y^{1},y)=-y\), and
\(\alpha=\frac34\). Then
\[
\kappa(\alpha,y^{1},\psi)=[y^{1},\infty),
\]
which is not compact. Hence uniform compact-valuedness fails.

\item\label{noH2-prime}
\textbf{Dropping \ref{H2-prime}.}
\emph{(a) Completeness/Polishness.}
Let \(\mathscr F\) be the set of polynomials on \([0,1]\), endowed with the
sup-norm subspace topology inherited from \(C([0,1])\). Write
\[
\mathscr F=\bigcup_{n\ge0}V_n,
\]
where \(V_n\) is the finite-dimensional space of degree-\(\le n\) polynomials.
Each \(V_n\) is closed and nowhere dense in \(C([0,1])\), and hence
\(\mathscr F\) is meagre in itself. By the Baire theorem, no nonempty Polish
space is meagre in itself. Thus \(\mathscr F\) is not Polish, so
\(Y^n\times\mathscr F\) is not Polish and is not an object of
\(\mathbf{WMeas}_{\mathrm{uc}}\).

\emph{(b) Continuity.}
Let \(Y=[0,1]\), \(n=1\), and fix \(\alpha=\tfrac34\), so that
\(\beta(\alpha)=1\). Choose an analytic non-Borel set
\(B\subset[0,1]\) whose complement contains a nonempty open interval \(O\), and
set \(A\coloneqq B\cup O\). Then \(A\) is analytic and non-Borel. Define the
score
\[
\psi(y^{1},y)\coloneqq \mathbbm{1}[y\in A].
\]
For \(x=y^{1}\), a direct computation of the full conformal transducer gives
\[
\kappa \left(\tfrac34,x,\psi\right)=
\begin{cases}
[0,1], & x\in A,\\[2pt]
[0,1]\setminus A, & x\notin A.
\end{cases}
\]
Taking \(G\coloneqq O\subset A\), which is open, we obtain
\[
\{x:\kappa(\tfrac34,x,\psi)\cap G\neq\emptyset\}=A,
\]
which is analytic non-Borel. Hence \(\kappa\) fails to be weakly measurable.
This illustrates the role of the continuity requirement in \ref{H2-prime}.
Notice that this example is not finite-valued; it is intended to show that the
measurability conclusion can fail once the regularity of \(\psi\) is removed.

\item\label{noH3-bis}
\textbf{Dropping (iii).}
Let \(\psi\equiv0\) and take \(\alpha=1\in S_{n+1}\). Then all nonconformity
scores are equal, so \(\pi\equiv1\). Since the conformal region is defined by
the strict inequality \(\pi>\alpha\), we get
\[
\kappa_1(y^n,\psi)=\emptyset.
\]
Thus, if morphisms are required to be nonempty-valued correspondences,
\(\kappa_1\) is not a morphism. The no-tie rule avoids this boundary pathology
and also allows the strict superlevel set \(\{\pi>\alpha\}\) to be rewritten as
the closed superlevel set \(\{\pi\ge\beta(\alpha)\}\).

\item
\textbf{Dropping \ref{H5-prime}.}
\emph{(i) Replacing the weak topology on \(\Delta_Y\) by total variation.}
Equip \(\Delta_Y\) with the total variation metric instead of the weak topology.
Let \(\psi\equiv0\). Then all nonconformity scores are equal, hence
\(\pi\equiv1\), and
\[
\CRED(y^n,\psi)=\Delta_Y
\qquad\text{for every }y^n.
\]
But \((\Delta_Y,\mathrm{TV})\) is not compact in general. Thus
\(\CRED(y^n,\psi)\) need not be compact, and \(\CRED\) fails to be
compact-valued, hence uniformly compact-valued. This illustrates why the weak
topology in \ref{H5-prime} is used.

\emph{(ii) Keeping the weak topology on \(\Delta_Y\), but replacing the Vietoris
topology on \(\mathscr C\) by the lower Vietoris hit topology.}
Let \(Y=[0,1]\), \(A=\{0\}\), \(P_0=\delta_0\), \(P_k=\delta_{1/k}\), and
\[
\mathcal M_k=\operatorname{co}(\{P_0,P_k\}),
\qquad
\mathcal M_\infty=\{P_0\}.
\]
Then \(\mathcal M_k\to\mathcal M_\infty\) in the lower Vietoris topology, but
\[
\underline P_{\mathcal M_k}(A)=0,
\qquad
\underline P_{\mathcal M_\infty}(A)=1.
\]
Thus \(\mathcal M\mapsto \underline P_{\mathcal M}(A)\) is not lower
semicontinuous. Consequently, the continuity-based proof of weak measurability
of \(\IHDR\) used above does not go through under the lower Vietoris topology;
additional arguments or stronger hypotheses would be needed.
\end{itemize}

Of course, consonance is needed to apply Proposition 4 in the Main Text. Thus
the assumptions in Theorem \ref{conf-pred-diagr2} play distinct structural
roles: They ensure that the relevant spaces are objects of
\(\mathbf{WMeas}_{\mathrm{uc}}\), and that the correspondences are uniformly
compact-valued and weakly measurable. The examples above show that these roles
cannot in general be omitted without replacement.
\end{remark}

\section{Further Proofs From Sections 3 and 5 in the Main Text}\label{extra:proofs}

In this section, we provide the full rigorous proofs of our category-theoretic results from Sections 3 and 5 in the Main Text.


\subsection{Proof of Theorem 6 in the Main Text}
We verify the axioms of a category. Recall that the composition $\Psi \circ \Phi$ of two correspondences $\Phi : X \rightrightarrows Y$ and $\Psi : Y \rightrightarrows Z$ is given by
\(
(\Psi \circ \Phi)(x) \coloneqq \bigcup_{y \in \Phi(x)} \Psi(y)
\).

\textbf{(1) Identity morphisms.} For each object $X$, define $\mathrm{id}_X(x) \coloneqq \{x\}$. Let $V \subseteq X$ be open and suppose $x_0 \in X$ satisfies $\mathrm{id}_X(x_0) = \{x_0\} \subseteq V$. Then, since $x_0 \in V$ and $V$ is open, there exists an open neighborhood $U = V$ of $x_0$ such that for all $x \in U$, $\mathrm{id}_X(x) = \{x\} \subseteq V$. Hence, $\mathrm{id}_X$ is upper hemicontinuous.

\textbf{(2) Associativity of composition.} Let $\Phi : X \rightrightarrows Y$, $\Psi : Y \rightrightarrows Z$, and $\Theta : Z \rightrightarrows W$. Then for all $x \in X$,
\[
((\Theta \circ \Psi) \circ \Phi)(x)
= \bigcup_{y \in \Phi(x)} \bigcup_{z \in \Psi(y)} \Theta(z)
= \bigcup_{z \in (\Psi \circ \Phi)(x)} \Theta(z)
= (\Theta \circ (\Psi \circ \Phi))(x)
\]
by the associativity of union. Thus, composition is associative.

\textbf{(3) Closure under composition.} Let $\Phi : X \rightrightarrows Y$ and $\Psi : Y \rightrightarrows Z$ be upper hemicontinuous correspondences. By \citet[Theorem 17.23]{aliprantis}, the composition $\Psi \circ \Phi$ is also upper hemicontinuous. Thus, the composition of morphisms is again a morphism.


\textbf{(4) Identity laws.} For any morphism $\Phi : X \rightrightarrows Y$ and any $x\in X$, 
\[
(\Phi \circ \mathrm{id}_X)(x) = \bigcup_{z \in \mathrm{id}_X(x)} \Phi(z) = \Phi(x), \quad
(1_Y \circ \Phi)(x) = \bigcup_{y \in \Phi(x)} \{y\} = \Phi(x).
\]
Hence, $\mathrm{id}_X$ acts as a left and right identity.

Therefore, $\mathbf{UHCont}$ satisfies the axioms of a category. \hfill \qed





\subsection{Proof of Theorem 17 in the Main Text}
We verify that $\IHDR$ is a functor.

	\textbf{(1) Objects.}
	Every $A$ in the defining intersection for $\B{\M}$ is measurable; Arbitrary
	intersections of sets in $\SigY$ remain in $\SigY$, hence
	$\B{\M}\in\SigY$.  Thus $\IHDR$ maps objects of $\Ccal$ to objects of $\SigY$.

	\textbf{(2) Morphisms (monotonicity).}
	Suppose $\Mone\subseteq\Mtwo$.  
	Let $y\in\B{\Mone}$.  If $y\notin\B{\Mtwo}$, then by definition of
	$\B{\Mtwo}$, there exists $A\in\SigY$ such that
	\[
	  \underline{P}_{\mathcal{M}_2}(A) \ge 1-\alpha, 
	  \qquad y\notin A.
	\]
	Because $\Mone\subseteq\Mtwo$, the same inequality holds with $\Mtwo$
	replaced by $\Mone$, so $A$ belongs to the family of sets whose intersection
	defines $\B{\Mone}$.  Consequently $y\in A$, contradicting $y\notin A$.
	Therefore $y\in\B{\Mtwo}$, and
	\(
	      \B{\Mone}\subseteq\B{\Mtwo}.
	\)
	
    Since $\mathcal M_1\subseteq\mathcal M_2$ implies
$\overline P_{\mathcal M_1}(A)\le \overline P_{\mathcal M_2}(A)$ for all $A\in\Sigma_Y$,
we have $\underline P_{\mathcal M_1}(A)=1-\overline P_{\mathcal M_1}(A^c)\ge
1-\overline P_{\mathcal M_2}(A^c)=\underline P_{\mathcal M_2}(A)$.
Hence
\[
\{A\in\Sigma_Y: \underline P_{\mathcal M_2}(A)\ge 1-\alpha\}
 \subseteq
\{A\in\Sigma_Y: \underline P_{\mathcal M_1}(A)\ge 1-\alpha\},
\]
and therefore $ B_{\mathcal M_1}^{\alpha}\subseteq B_{\mathcal M_2}^{\alpha}$.
    
    Intuitively, enlarging a credal set can only decrease the lower probability (by the monotonicity of lower probabilities), so fewer events clear the $1-\alpha$ bar; Intersecting fewer sets yields a larger IHDR. Hence $\IHDR$ maps the arrow $\Mone\hookrightarrow\Mtwo$ to the arrow
	$\B{\Mone}\hookrightarrow\B{\Mtwo}$.

	\textbf{(3) Identities.}
	For each $\M\in\Ccal$, the identity morphism is $\M\hookrightarrow\M$.
	Since inclusion is reflexive, $\B{\M}\hookrightarrow\B{\M}$ is the identity
	in $\SigY$, so $\IHDR(\mathrm{id}_{\M})=\mathrm{id}_{\B{\M}}$.

	\textbf{(4) Composition.}
	Given $\Mone\subseteq\Mtwo\subseteq\Mthr$, the composite morphism in $\Ccal$
	is $\Mone\hookrightarrow\Mthr$.  By monotonicity,
	\(
	   \B{\Mone}\subseteq\B{\Mtwo}\subseteq\B{\Mthr},
	\)
	so
    \begin{align*}
        \IHDR(\Mone\hookrightarrow\Mthr)
	        =\B{\Mone}\hookrightarrow\B{\Mthr}
	        &=(\B{\Mtwo}\hookrightarrow\B{\Mthr})
	              \circ(\B{\Mone}\hookrightarrow\B{\Mtwo})\\
	        &=\IHDR(\Mtwo\hookrightarrow\Mthr)
	              \circ\IHDR(\Mone\hookrightarrow\Mtwo).
    \end{align*}

Since $\IHDR$ respects objects, morphisms, identities and
	composition, it is a covariant functor. \hfill \qed

\begin{remark}[Contravariant Variants]
Replacing $B_\mathcal{M}^\alpha$ by its complement
$Y\setminus B_\mathcal{M}^\alpha$ yields a contravariant functor \citep[Section 1.3.7]{perrone2024starting}.
Allowing $\alpha$ itself to vary introduces a functor
\(
  \mathsf{IHDR}\colon
  [0,1] \times \Ccal \rightarrow \Sigma_Y
\)
that is antitone in the first argument and covariant in the second. 
\end{remark}

\renewcommand{\IHDR}{\mathsf{IHDR}}

\section{Further Properties of our Categories}\label{further-props}

In this section, we ask ourselves whether we can further characterize the categories that we introduced in Section 3 in the Main Text. Throughout, if restrictions have to be imposed on the objects or morphisms of a category, we will report them as superscript and subscripts of the ensuing subcategory, respectively.

\subsection{Monoidal Categories}

The first, natural question is whether they are monoidal \citep[Chapter 6]{perrone2024starting}.
Alas, $\mathbf{UHCont}$ is not, as upper hemicontinuity is not preserved under product in general unless the correspondences are compact-valued. We can define a slight strengthening of $\mathbf{UHCont}$, though, which is monoidal.  

\begin{definition}[Structure \textbf{UHCont}$_\text{c}$]\label{structure-uhcont-comp}
Define a structure $\mathbf{UHCont}_\text{c}$ as follows,
\begin{itemize}
  \item \textbf{Objects:} Topological spaces $(X, \tau_X)$.
  \item \textbf{Morphisms:} {\em Compact-valued}, upper hemicontinuous correspondences $\Phi : X \rightrightarrows Y$.
\end{itemize}
\end{definition}

\begin{theorem}[$\mathbf{UHCont}_\text{c}$ is a Monoidal Category]\label{monoidal1}
The structure \(\mathbf{UHCont}_\text{c}\) is a monoidal category under the Cartesian product of topological spaces.
\end{theorem}

\begin{proof}
We verify the monoidal category structure.

\textbf{(1) Tensor Product of Objects and Morphisms.} For objects \(X, Y\), define \(X \otimes Y \coloneqq X \times Y\), the Cartesian product with the product topology.

Given morphisms \(\Phi_1: X \rightrightarrows X^\prime\) and \(\Phi_2: Y \rightrightarrows Y^\prime\), define their tensor product as
\[
(\Phi_1 \otimes \Phi_2)(x,y) \coloneqq \Phi_1(x) \times \Phi_2(y).
\]

Since compact sets are preserved under finite products, \(\Phi_1(x) \times \Phi_2(y)\) is compact for all \((x,y) \in X \times Y\). Also, the product of upper hemicontinuous correspondences with compact values is upper hemicontinuous with compact values \citep[Theorem 17.28]{aliprantis}. Thus, \(\Phi_1 \otimes \Phi_2\) is a morphism in \(\mathbf{UHCont}_\text{c}\).

\textbf{(2) Unit Object.} Let the unit object be the singleton space \(I \coloneqq \{*\}\), with the trivial topology. For any object \(X\), the canonical identifications
\[
\lambda_X: I \times X \rightarrow X, \quad \rho_X: X \times I \rightarrow X
\]
are homeomorphisms. They define morphisms in \(\mathbf{UHCont}_\text{c}\) via single-valued, continuous correspondences (e.g., \( (*, x) \mapsto x \)) with compact singleton values.

\textbf{(3) Associativity.} Define the associator as the canonical homeomorphism,
\[
\alpha_{X,Y,Z} : (X \times Y) \times Z \rightarrow X \times (Y \times Z).
\]
Its graph defines a single-valued, continuous correspondence with compact-valued singleton fibers. Hence, it is a morphism in the category.

\textbf{(4) Bifunctoriality of $\otimes$.}
For morphisms $\Phi_1:X \rightrightarrows X'$, $\Psi_1:X' \rightrightarrows X''$
and $\Phi_2:Y \rightrightarrows Y'$, $\Psi_2:Y' \rightrightarrows Y''$,
for each $(x,y)\in X\times Y$,
\begin{align*}
    \bigl((\Psi_1 \circ \Phi_1)\otimes(\Psi_2 \circ \Phi_2)\bigr)(x,y)
&= \Bigl( \bigcup_{u\in\Phi_1(x)}\Psi_1(u)\Bigr)\times
  \Bigl( \bigcup_{v\in\Phi_2(y)}\Psi_2(v)\Bigr)\\
&= \bigcup_{(u,v)\in \Phi_1(x)\times \Phi_2(y)} \Psi_1(u)\times\Psi_2(v),
\end{align*}
which equals $\bigl((\Psi_1\otimes\Psi_2)\circ(\Phi_1\otimes\Phi_2)\bigr)(x,y)$
by the definition of composition for correspondences. Identities are preserved
since $(\mathrm{id}_X\otimes \mathrm{id}_Y)(x,y)=\{x\}\times\{y\}$.

\textbf{(5) Naturality of the associator and unitors.}
For morphisms $\Phi_1:X \rightrightarrows X'$, $\Phi_2:Y \rightrightarrows Y'$,
$\Phi_3:Z \rightrightarrows Z'$, we have
\[
\alpha_{X',Y',Z'}\circ\bigl((\Phi_1\otimes\Phi_2)\otimes\Phi_3\bigr)
= \bigl(\Phi_1\otimes(\Phi_2\otimes\Phi_3)\bigr)\circ \alpha_{X,Y,Z},
\]
since both sides send $(x,y,z)$ to
$\Phi_1(x)\times\bigl(\Phi_2(y)\times\Phi_3(z)\bigr)$ (after re-bracketing).
For unitors, naturality reads
\[
\lambda_{X'}\circ(\mathrm{id}_I\otimes \Phi_1)=\Phi_1\circ \lambda_X,
\qquad
\rho_{Y'}\circ(\Phi_2\otimes \mathrm{id}_I)=\Phi_2\circ \rho_Y,
\]
because $\{*\}\times \Phi_1(x)=\Phi_1(x)$ and $\Phi_2(y)\times\{*\}=\Phi_2(y)$.

\textbf{(6) Coherence Conditions.}
The pentagon and triangle commute, as in \textbf{Top} with the Cartesian
product; Here $\alpha,\lambda,\rho$ are homeomorphisms (single-valued morphisms),
so the standard coherence diagrams commute strictly.

\textbf{(7) Symmetry.}
The flip map $\sigma_{X,Y}:X\times Y\to Y\times X$, $\sigma(x,y)=(y,x)$,
is a homeomorphism, hence a morphism; The usual hexagon coherence holds.
Thus $\mathbf{UHCont}_\mathrm{c}$ is symmetric monoidal. Moreover, the braiding is natural:
$(\Phi_2\otimes \Phi_1)\circ \sigma_{X,Y}
= \sigma_{X',Y'}\circ (\Phi_1\otimes \Phi_2)$.
\end{proof}

Given our previous proof, it is easy to see that, under conditions (i)-(iv) in Section 3.2 in the Main Text, the Full Conformal Prediction correspondence $\kappa$ is a morphism of $\mathbf{UHCont}_\text{c}$. This is because $Y$ is compact Hausdorff, and the graph of $\kappa$ is closed by the Proof of Theorem 7 in the Main Text. By the Closed Graph Theorem (Lemma 8 in the Main Text), this means that $\kappa$ is upper hemicontinuous and closed-valued; But $Y$ is compact Hausdorff, which implies that $\kappa$ is compact-valued.

Recall from Definition 10 in the Main Text that the structure $\mathbf{WMeas}_\text{uc}$ has compact Polish spaces with the Borel $\sigma$-algebra for objects, and weakly measurable, uniformly compact-valued correspondences for morphisms. $\mathbf{WMeas}_\text{uc}$ is a monoidal category.

\begin{theorem}[$\mathbf{WMeas}_{\text{uc}}$ is a Monoidal Category]\label{monoid-cat-wmeasuc}
The structure \(\mathbf{WMeas}_{\text{uc}}\) forms a monoidal category under the Cartesian product.
\end{theorem}


\begin{proof}

Let the unit object be the singleton space \(I \coloneqq \{*\}\), with the trivial topology, and
\[
(X,\Sigma_X) \otimes (X',\Sigma_{X'})
     \coloneqq (X\times X', \Sigma_X \otimes \Sigma_{X'}),
\qquad
(F\otimes G)(x,x') \coloneqq F(x)\times G(x').
\]
Since $X,X'$ are Polish, $\Sigma_X\otimes\Sigma_{X'}=\mathcal{B}(X\times X')$.

\textbf{(1) Tensor is a bifunctor.}
For morphisms $F:X \Rightarrow Y$ and $G:X' \Rightarrow Y'$ the map
$(F\otimes G)(x,x')=F(x)\times G(x')$ is

\emph{Uniformly compact.}  If $K_F\subset Y$ and $K_G\subset Y'$
witness compactness, $K_F\times K_G$ is compact by Tychonoff, and
$(F\otimes G)(x,x')\subset K_F\times K_G$ for all $(x,x')$.

\emph{Weakly measurable.}  Let
$O=\bigcup_{k\in\mathbb N}(B_k\times C_k)$ be an open subset of
$Y\times Y'$ written as a countable union of basic rectangles
(recall that $Y,Y'$ are second-countable, so any open set in $Y\times Y'$ is a \emph{countable} union of basic rectangles).
Then,
\[
  (F\otimes G)^{-1}(O)
    = 
   \bigcup_{k\in\mathbb N}
     \bigl(\{x:F(x)\cap B_k \neq\emptyset\}
            \times
            \{x':G(x')\cap C_k \neq\emptyset\}\bigr),
\]
a countable union of measurable rectangles, hence measurable in
$\Sigma_X \otimes \Sigma_{X'}$.

Functoriality is immediate:
\((G_1 \circ F_1)\otimes(G_2 \circ F_2)
  =(G_1\otimes G_2) \circ (F_1\otimes F_2)\) because product
distributes over union, and $1_X\otimes1_{X'}=1_{X\times X'}$.

\textbf{(2) Associativity, unit, symmetry.}
Because $\otimes$ is the cartesian product on objects and the
point‑wise cartesian product on fibers, we have the canonical
\emph{homeomorphisms}
\[
  \alpha_{X,Y,Z} : (X\times Y)\times Z  \longrightarrow  X\times (Y\times Z),
  \qquad
  \lambda_X : \{\ast\}\times X  \longrightarrow  X,
  \qquad
  \rho_X : X\times\{\ast\}  \longrightarrow  X,
\]
and the symmetry (flip) map
\[
  \sigma_{X,Y} : X\times Y  \longrightarrow  Y\times X, 
                 (x,y)\mapsto(y,x).
\]
Each of these maps is single‑valued and continuous, hence a morphism in
$\mathbf{WMeas}_{\text{uc}}$ (singleton fibers are compact).  They play
the role of associator, left- and right-unitors, and symmetry,
respectively.



\textbf{(3) Coherence.}  The associator, unitors, and symmetry just introduced are the usual
homeomorphisms that witness the cartesian product as a symmetric
monoidal structure in \textbf{Set}.  
Consequently the pentagon, triangle, and hexagon coherence diagrams
commute exactly as they do in \textbf{Top} (hence also in \textbf{Meas}), and the same proofs apply
verbatim in $\mathbf{WMeas}_{\text{uc}}$.

Hence
\((\mathbf{WMeas}_{\mathrm{uc}},\otimes,\{\ast\})\) is a 
symmetric monoidal category.
\end{proof}

\subsection{Functors Between Categories}\label{functors-sec}

Another question that we ask ourselves is whether there is a canonical functor between $\mathbf{WMeas}_\text{uc}$ and $\mathbf{UHCont}$ (or vice versa). The answer, as the attentive reader may expect, is no. The reason is that there exists a non-functorial gap between measurability and continuity, two different concepts that also motivated us to consider categories $\mathbf{WMeas}_\text{uc}$ and $\mathbf{UHCont}$ to begin with. 

Functors do exist if we focus on restrictions of $\mathbf{WMeas}_\text{uc}$ and $\mathbf{UHCont}$.\footnote{We defer the study of monoidal functors to future research.} In particular, call
$\mathbf{WMeas}_\text{uc,uhc}\subset \mathbf{WMeas}_\text{uc}$ the category whose objects are compact Polish spaces with the Borel $\sigma$-algebra, and whose morphisms are uniformly compact-valued, weakly measurable, {\em and upper hemicontinuous} correspondences. Then, the following holds.

\begin{theorem}[Faithful Functor between $\mathbf{WMeas}_\text{uc,uhc}$ and $\mathbf{UHCont}$]\label{thm-funct1}
There exists a faithful functor 
\[
\mathcal{F} : \mathbf{WMeas}_\text{uc,uhc} \rightarrow \mathbf{UHCont}
\]
defined by
\begin{itemize}
  \item \(\mathcal{F}(X, \tau_X, \mathcal{B}(\tau_X)) \coloneqq (X, \tau_X)\),
  \item \(\mathcal{F}(\Phi) \coloneqq \Phi\) (viewed as a topological correspondence).
\end{itemize}
\end{theorem}

\begin{proof}
We verify that \(\mathcal{F}\) is a functor.

{\bf (1) Identity.} The identity correspondence \(\mathrm{id}_X(x) = \{x\}\) is trivially compact-valued, upper hemicontinuous, and weakly measurable. Thus, \(\mathcal{F}(\mathrm{id}_X) = \mathrm{id}_X\) is a valid morphism in \(\mathbf{UHCont}\).


{\bf (2) Composition.} Let $\Phi : X \rightrightarrows Y$ and $\Psi : Y \rightrightarrows Z$ be morphisms in $\mathbf{WMeas}_{\text{uc,uhc}}$.

\emph{Uniform compact bound.} If $K_\Psi\subset Z$ is a uniform compact bound for $\Psi$, then
\[
(\Psi\circ\Phi)(x)=\bigcup_{y\in\Phi(x)}\Psi(y)\subset K_\Psi
\quad\text{for all }x\in X,
\]
so $\Psi\circ\Phi$ is uniformly bounded.

\emph{Compact fibers.} Since $\Phi(x)$ is compact, $\Psi$ has compact values, and $Z$ is compact Hausdorff, upper hemicontinuity of $\Psi$ implies $\operatorname{Gr}(\Psi)\subset Y\times Z$ is closed (by Lemma 8 in the Main Text). Hence
\[
G_x \coloneqq \operatorname{Gr}(\Psi)\cap\bigl(\Phi(x)\times K_\Psi\bigr)
\]
is compact; Its projection onto $Z$ equals $(\Psi\circ\Phi)(x)$, so each $(\Psi\circ\Phi)(x)$ is compact.

\emph{Upper hemicontinuity.} By \cite[Theorem 17.23]{aliprantis}, with $\Phi$ and $\Psi$ both u.h.c., the composition $\Psi\circ\Phi$ is u.h.c.

\emph{Weak measurability.} Preserved under composition as in the proof of Theorem \ref{meas-corr-cat} (via Castaing selectors).

Thus $\Psi\circ\Phi$ is again a morphism in $\mathbf{WMeas}_{\text{uc,uhc}}$, and
$\mathcal F(\Psi\circ\Phi)=\mathcal F(\Psi)\circ\mathcal F(\Phi)$.

{\bf(3) Faithfulness.} If \(\Phi \neq \Psi\) as correspondences in \(\mathbf{WMeas}_\text{uc,uhc}\), then \(\mathcal{F}(\Phi) \neq \mathcal{F}(\Psi)\) in \(\mathbf{UHCont}\). So \(\mathcal{F}\) is injective on morphisms. Therefore, \(\mathcal{F}\) is a well-defined faithful functor.
\end{proof}

Functor \(\mathcal{F}\) in Theorem \ref{thm-funct1} ``forgets'' the measurable structure but preserves the topological one. The condition of weak measurability ensures that the image \(\Phi\) remains a valid morphism in both categories. 

It is also worth mentioning that, in general, $\mathcal{F}$ does not have a limit in $\mathbf{UHCont}$ \citep[Definition 3.1.5]{perrone2024starting}. This is because the representability conditions fail: Universal lifting of upper hemicontinuous cones is obstructed by the loss of measurability information via $\mathcal{F}$. This is a consequence of the incompatibility between continuity-based and measurability-based structures, exactly the same structural gap that made defining functor $\mathcal{F}$ nontrivial in the first place.

Let now $\mathbf{UHCont}_\text{c}^\text{P} \subset \mathbf{UHCont}$ be the category whose objects are compact {\em Polish} topological spaces, and whose morphisms are upper hemicontinuous, 
{uniformly compact-valued} correspondences. Then, the following holds.

\begin{theorem}[Faithful Functor between $\mathbf{UHCont}_{\mathrm{c}}^{\mathrm{P}}$ and $\mathbf{WMeas}_\text{uc}$]\label{thm-funct2}
There exists a faithful functor
\[
\mathcal{G} : \mathbf{UHCont}_{\mathrm{c}}^{\mathrm{P}} \rightarrow \mathbf{WMeas}_\text{uc}
\]
defined by
\begin{itemize}
  \item \(\mathcal{G}(X, \tau_X) \coloneqq (X, \tau_X, \mathcal{B}(\tau_X))\),
  \item \(\mathcal{G}(\Phi) \coloneqq \Phi\), viewing the correspondence as a measurable one.
\end{itemize}
\end{theorem}

\begin{proof}
If $X$ is compact Polish, then $(X,\tau_X,\mathcal B(\tau_X))$ is a compact Polish
measurable space, so $\mathcal G$ is well-defined on objects.

Let $\Phi:X\rightrightarrows Y$ be a morphism in $\mathbf{UHCont}_{\mathrm{c}}^{\mathrm{P}}$,
i.e., upper hemicontinuous with uniformly compact values (hence closed values). 
Since $Y$ is compact Polish, $\Phi$ is automatically \emph{uniformly} compact-valued
(take $K_\Phi=Y$). We now show $\Phi$ is weakly measurable.

\emph{Weak measurability.}
Fix open $O\subseteq Y$ and define $\varphi_O(y)\coloneqq d(y,O^c)$, where $d$ is the metric associated with the topology on $Y$.
Set $h_O(x)\coloneqq \sup_{y\in\Phi(x)} \varphi_O(y)$. By Berge’s maximum theorem $h_O$ is upper semicontinuous. Hence
\[
\{x:\Phi(x)\cap O\neq\emptyset\}
=\{x:h_O(x)>0\}
=\bigcup_{m=1}^\infty \{x:h_O(x)\ge 1/m\},
\]
a countable union of closed sets; therefore it is measurable. Hence $\Phi$ is weakly measurable.
Therefore $\mathcal G(\Phi)$ is a morphism in $\mathbf{WMeas}_{\mathrm{uc}}$.

\emph{Functoriality.} Identities are preserved trivially. For composition, if
$\Phi:X\rightrightarrows Y$ and $\Psi:Y\rightrightarrows Z$ are upper hemicontinuous, then $\Psi\circ\Phi$ is upper hemicontinuous \citep[Theorem 17.23]{aliprantis}. Because $\Psi$ is upper hemicontinuous with compact values and $Z$ is
compact Hausdorff, Lemma 8 in the Main Text implies that
$\operatorname{Gr}(\Psi)\subset Y\times Z$ is closed. For each $x\in X$, set
\(
G_x \coloneqq \operatorname{Gr}(\Psi)\cap\bigl(\Phi(x)\times Z\bigr).
\)
Since $\Phi(x)$ is compact and $Z$ is compact Hausdorff, the product
$\Phi(x)\times Z$ is compact; Hence $G_x$ is compact as a closed subset of a
compact set. The projection onto $Z$ satisfies
\[
\pi_Z(G_x)=\bigcup_{y\in \Phi(x)}\Psi(y)=(\Psi\circ\Phi)(x),
\]
and since $\pi_Z:Y\times Z\to Z$ is continuous, $\pi_Z(G_x)$ is compact. Therefore
$(\Psi\circ\Phi)(x)$ is compact,
and the same uniform bound $K_{\Psi}=Z$ works. Weak measurability of $\Psi\circ\Phi$
follows from the previous argument applied to the composite. Thus
$\mathcal G(\Psi\circ\Phi)=\mathcal G(\Psi)\circ\mathcal G(\Phi)$.

Finally, $\mathcal G$ is the identity on morphisms, so it is injective on each
$\mathrm{Hom}$-set and hence faithful.
\end{proof}



We also mention that, just as continuity cannot capture all measurable behavior, measurability cannot reconstruct topological properties like upper hemicontinuity. Therefore, $\mathcal{G}$ typically fails to admit a limit object in \(\mathbf{WMeas}_\text{uc}\) \citep[Definition 3.1.5]{perrone2024starting} representing $\text{Cone}(-,\mathcal{G})$, for precisely the same reason $\mathcal{F}$ does: The two categories encode fundamentally distinct structures.

\subsection{Monads}

We now ask ourselves whether we can find monads on the categories $\mathbf{UHCont}$ and \(\mathbf{WMeas}_\text{uc}\). The general answer is no, but we can find interesting subcategories for which a monad can be found.

\begin{definition}[Structure \(\mathbf{UHCont}^{\mathrm{cH}}_\text{c}\)]\label{def-chcorr}
Define the structure \(\mathbf{UHCont}^{\mathrm{cH}}_\text{c}\) as follows,
\begin{itemize}
  \item \textbf{Objects:} {\em Compact Hausdorff} spaces \(X\).
  \item \textbf{Morphisms:} {\em Compact-valued}, upper hemicontinuous \(\Phi : X \rightrightarrows Y\).
\end{itemize}
\end{definition}

It is immediate to see that structure \(\mathbf{UHCont}^{\mathrm{cH}}_\text{c}\) is a subcategory of \(\mathbf{UHCont}\). It admits a monad.

\begin{theorem}[Vietoris \emph{lax} monad on $\mathbf{UHCont}^{\mathrm{cH}}_{\mathrm{c}}$]
\label{vietoris-monad}
Define an endofunctor $T$ and morphisms $\eta,\nu$ as follows:
\begin{itemize}
  \item On objects: $T(X)\coloneqq\mathcal K(X)$, the space of nonempty compact subsets of $X$
        with the Vietoris topology.
  \item On morphisms: for $\Phi:X\rightrightarrows Y$ and $K\in\mathcal K(X)$, set
        \[
           T(\Phi)(K) \coloneqq \bigl\{ L\in\mathcal K(Y) :  L\subset \Phi[K] \bigr\},
           \qquad
           \Phi[K]\coloneqq\bigcup_{x\in K}\Phi(x).
        \]
  \item Unit: $\eta_X:X\to\mathcal K(X)$, $\eta_X(x)=\{x\}$ (single-valued morphism).
  \item Multiplication: $\nu_X:\mathcal K(\mathcal K(X))\rightrightarrows \mathcal K(X)$,
        \[
          \nu_X(\mathcal A) = \Bigl\{ \bigcup_{A\in\mathcal A'}A : \emptyset\neq \mathcal A'\subseteq\mathcal A
           \text{ compact in }\mathcal K(X)\Bigr\},
        \]
\end{itemize}
where $\mathcal K(\mathcal K(X))$ denotes the hyperspace whose elements are nonempty compact subsets of $\mathcal K(X)$. Then, $T$ is an endofunctor on $\mathbf{UHCont}^{\mathrm{cH}}_{\mathrm{c}}$,
$\eta$ and $\nu$ are natural (single-valued, resp.  set-valued) morphisms,
and $(T,\eta,\nu)$ satisfies the \emph{lax} monad laws (with respect to inclusion of graphs),
\[
\mathrm{id} \subseteq \nu\circ T\eta,\qquad
\nu\circ \eta T=\mathrm{id},\qquad
\nu\circ T\nu \subseteq \nu\circ \nu T.
\]
Moreover, on the subcategory of single-valued continuous maps (\textbf{CHaus}),
these inclusions become equalities, recovering the classical Vietoris monad.
\end{theorem}

\begin{proof}

\textbf{(1) $T$ is a well-defined endofunctor.}
For a morphism $\Phi:X\rightrightarrows Y$, values of $T(\Phi)$ are nonempty compact subsets of $\mathcal K(Y)$.
Indeed, $K\mapsto\Phi[K]$ is compact for each compact $K\subset X$
(since $\Phi$ has a closed graph and compact values, and $K$ is compact),
so $T(\Phi)(K)=\mathcal K(\Phi[K])$ is compact in $\mathcal K(Y)$.
To see upper hemicontinuity, it suffices to show the graph of $T(\Phi)$ is closed in
$\mathcal K(X)\times\mathcal K(Y)$.
Let $K_m\to K$ in $\mathcal K(X)$ and $L_m\to L$ in $\mathcal K(Y)$ with $L_m\subset \Phi[K_m]$.
Pick any $y\in L$; Choose $y_m\in L_m$ with $y_m\to y$.
For each $m$, pick $x_m\in K_m$ with $y_m\in\Phi(x_m)$.
By compactness of $X$ and the Vietoris convergence $K_m\to K$, the net $(x_m)$ has a subnet
$(x_{m_\ell})$ converging to some $x\in K$.
Upper hemicontinuity of $\Phi$ then gives $y\in\Phi(x)$, hence $y\in\Phi[K]$.
As $y\in L$ was arbitrary, $L\subset \Phi[K]$; Thus $(K,L)\in\operatorname{Gr}(T(\Phi))$.
Therefore $T(\Phi)$ is compact-valued, upper hemicontinuous.
Functoriality is immediate,
\begin{align*}
    T(\Psi\circ\Phi)(K)
=\{ L: L\subset (\Psi\circ\Phi)[K]\}
&=\{ L: L\subset \Psi[\Phi[K]]\}\\
&=\bigcup_{ M\in \mathcal{K}(\Phi[K])} \{ L \in \mathcal{K}(Y): L\subset \Psi[M]\}\\
&=(T\Psi\circ T\Phi)(K),
\end{align*}
and $T(\mathrm{id})=\mathrm{id}$.

\textbf{(2) $\eta$ and $\nu$ are morphisms (natural in $X$).}
For $\eta_X$, singletons are compact and the map $x\mapsto\{x\}$
is continuous into $\mathcal K(X)$, hence a morphism.
For $\nu_X$, values are compact: If $\mathcal A'\subset\mathcal A$ is compact in $\mathcal K(X)$,
then $\bigcup_{A\in \mathcal A'}A$ is a compact subset of $X$.
Upper hemicontinuity of $\nu_X$ follows from the Vietoris calculus: For a basic open set
$\langle U_1,\dots,U_m\rangle\subset\mathcal K(X)$,
\[
\nu_X^{-1}\bigl(\langle U_1,\dots,U_m\rangle\bigr)
\supseteq\Bigl\langle [U_1],\dots,[U_m], \langle U_1\cup\cdots\cup U_m\rangle \Bigr\rangle,
\]
and the right hand side is open in $\mathcal K(\mathcal K(X))$ (see Definition \ref{vietoris} in Appendix \ref{funct-anal}); Hence $\nu_X$ is upper hemicontinuous.

Naturality (in the lax sense) holds by direct computation:
For any $\Phi:X\rightrightarrows Y$,
\[
(T\Phi)\circ \eta_X  \supseteq  \eta_Y\circ \Phi
\qquad\text{and}\qquad
\nu_Y\circ T(T\Phi) \supseteq  T(\Phi)\circ \nu_X,
\]
since $\{y\}\subset \Phi(x)$ iff $y\in\Phi(x)$, and unions commute with direct images.

\textbf{(3) Lax monad laws.}
All inclusions are understood as inclusions of correspondences (i.e., of their graphs).

\emph{Left unit (lax):}
For $K\in\mathcal K(X)$,
\[
(T\eta_X)(K)=\bigl\{ \mathcal C\in\mathcal K(\mathcal K(X)):  \mathcal C\subset \eta_X[K]\bigr\}.
\]
Thus
\[
(\nu_X\circ T\eta_X)(K)
=\bigl\{ \textstyle\bigcup_{C\in\mathcal C'} C:  \emptyset\neq \mathcal C'\subset \eta_X[K],  \mathcal C'\in\mathcal K(\mathcal K(X))\bigr\}
=\{ C\in\mathcal K(X): C\subset K \},
\]
so $\mathrm{id}_{\mathcal K(X)}\subseteq \nu_X\circ T\eta_X$ (take $\mathcal C'=\eta_X[K]$ to recover $K$), but equality
need not hold, hence laxity.

\emph{Right unit (strict):}
$\eta_{\mathcal K(X)}(K)=\{K\}$ and therefore $(\nu_X\circ \eta_{\mathcal K(X)})(K)=K$,
so $\nu\circ \eta T=\mathrm{id}$ holds as equality.

\emph{Associativity (lax):}
For $\mathscr B\in\mathcal K(\mathcal K(\mathcal K(X)))$, elements of $(T\nu_X)(\mathscr B)$
are compact families $\mathcal A\subset\{\bigcup_{A\in\mathcal C}A:\mathcal C\in\mathscr B\}$, and
\[
(\nu_X\circ T\nu_X)(\mathscr B)
=\Bigl\{ \bigcup_{A\in\mathcal A'}A:  \mathcal A'\subset (T\nu_X)(\mathscr B),  \mathcal A'\in\mathcal K(\mathcal K(X))\Bigr\}.
\]
Each such union is a union of unions of sets from members of $\mathscr B$; Regrouping gives a union
of sets from $\bigcup_{\mathcal C\in\mathscr B}\mathcal C$, so
\[
\nu_X\circ T\nu_X  \subseteq  \nu_X\circ \nu_{\mathcal K(X)},
\]
establishing the lax associativity.

\textbf{(4) Restriction to \textbf{CHaus}.}
If we restrict morphisms to single-valued continuous maps $f:X\to Y$,
then $T(f)(K)=f[K]$, $\eta$ and $\nu$ are the classical unit and multiplication,
and the three lax inclusions above become equalities, recovering the usual Vietoris monad.
\end{proof}


The construction in Theorem \ref{vietoris-monad} adapts the Vietoris
hyperspace mechanism to the correspondence category
$\mathbf{UHCont}^{\mathrm{cH}}_{\mathrm{c}}$.
Because morphisms are upper hemicontinuous, compact-valued \emph{correspondences},
the hyperspace assignment $T$ is a \emph{strict} endofunctor:
For morphisms $\Phi,\Psi$ we have
\[
T(\Psi\circ\Phi)=T(\Psi)\circ T(\Phi).
\]
The key point is that for compact $K\subset X$, $\Phi[K]$ is compact, so
$T(\Psi\circ\Phi)(K)=\{L:L\subset \Psi[\Phi[K]]\}=(T\Psi\circ T\Phi)(K)$.
The resulting structure is \emph{lax} only at the level of monad laws, where
\[
\mathrm{id}\subseteq \nu\circ T\eta,\qquad
\nu\circ \eta T=\mathrm{id},\qquad
\nu\circ T\nu \subseteq \nu\circ \nu T.
\]

By contrast, on the classical category $\mathbf{CHaus}$ (compact
Hausdorff spaces with continuous maps), the Vietoris hyperspace
$X\mapsto\mathcal K(X)$ together with $T(f)(K)=f[K]$, $\eta(x)=\{x\}$, and
$\mu(\mathcal A)=\bigcup_{A\in\mathcal A}A$ defines a \emph{strict} monad
 \citep{Garner2020_Vietoris}. In that classical setting, the
Eilenberg-Moore (EM) category of algebras for the Vietoris monad is
equivalent to the category $\mathbf{ContLat}$ of continuous lattices with
Scott-continuous lattice homomorphisms \cite{Day1975,Wyler1981}.

Hence, the EM-$\mathbf{ContLat}$ equivalence applies to the \emph{classical}
Vietoris monad on $\mathbf{CHaus}$. For our correspondence-based (lax)
variant on $\mathbf{UHCont}^{\mathrm{cH}}_{\mathrm{c}}$, we would need a
theory of \emph{lax algebras}; We do not pursue this here.


We now show that, for any $n\in\mathbb{N}$, the (minimal) sufficient conditions for the Full Conformal Prediction diagram to commute in category \(\mathbf{UHCont}^{\mathrm{cH}}_\text{c}\) are the same that make it commute in category \(\mathbf{UHCont}\).

\begin{theorem}[The Pulled-Back Conformal Prediction Diagram Commutes in
\(\mathbf{UHCont}^{\mathrm{cH}}_{\mathrm c}\)]
\label{cont-cp}
Assume that (i), (ii), and (iv) in Section 3.2 in the Main Text hold, and that the
conformal transducer \(\pi\) is consonant. Fix
\(\alpha\in[0,1]\setminus S_{n+1}\).
Let \(\mathscr F_0\subset\mathscr F\) be nonempty, closed, uniformly bounded,
and equicontinuous on the compact domain \(Y^n\times Y\). Set
$X_0\coloneqq Y^n\times\mathscr F_0$.
For \((y^n,\psi)\in X_0\), define
\[
\kappa_\alpha(y^n,\psi)
\coloneqq
\{y\in Y:\pi_\psi(y,y^n)>\alpha\},
\]
and define the pulled-back indirect correspondence
\[
\mathcal H_\alpha(y^n,\psi)
\coloneqq
\IHDR_\alpha\bigl(\CRED(y^n,\psi)\bigr).
\]
Then \(\kappa_\alpha:X_0\rightrightarrows Y\) and
\(\mathcal H_\alpha:X_0\rightrightarrows Y\) are morphisms in
\(\mathbf{UHCont}^{\mathrm{cH}}_{\mathrm c}\), and
$\mathcal H_\alpha=\kappa_\alpha$.
Equivalently, the direct full conformal construction and the indirect
credal-IHDR construction agree as compact-valued upper hemicontinuous
correspondences from \(X_0\) to \(Y\).
Moreover, if one equips the image
\[
\mathscr M_{\mathrm{CP},0}
\coloneqq
\{\,\CRED(y^n,\psi):(y^n,\psi)\in X_0\,\}
\]
with a compact Hausdorff topology for which
\[
c_{\mathrm{CP}}:X_0\to\mathscr M_{\mathrm{CP},0},
\qquad
c_{\mathrm{CP}}(y^n,\psi)=\CRED(y^n,\psi),
\]
is continuous and the restricted correspondence
\[
\widehat{\IHDR}_\alpha:\mathscr M_{\mathrm{CP},0}\rightrightarrows Y,
\qquad
\widehat{\IHDR}_\alpha(\mathcal M)=\IHDR_\alpha(\mathcal M),
\]
is compact-valued and upper hemicontinuous, then the triangular diagram
\[
\begin{tikzcd}
X_0 \ar{r}{c_{\mathrm{CP}}} \ar{dr}[swap]{\kappa_\alpha}
&
\mathscr M_{\mathrm{CP},0} \ar{d}{\widehat{\IHDR}_\alpha}
\\
& Y
\end{tikzcd}
\]
commutes in \(\mathbf{UHCont}^{\mathrm{cH}}_{\mathrm c}\).
\end{theorem}

\begin{proof}
By Arzelà-Ascoli, \(\mathscr F_0\) is compact in the sup-norm topology.
Since \(Y\) is compact Hausdorff by (i), \(Y^n\) is compact Hausdorff, and hence
$X_0=Y^n\times\mathscr F_0$
is compact Hausdorff. Thus \(X_0\) and \(Y\) are objects of
\(\mathbf{UHCont}^{\mathrm{cH}}_{\mathrm c}\).

We first show that \(\kappa_\alpha\) is compact-valued and upper
hemicontinuous. Since \(\alpha\notin S_{n+1}\), define
\[
\beta(\alpha)
\coloneqq
\min\left\{\frac{k}{n+1}\in S_{n+1}:\frac{k}{n+1}>\alpha\right\}.
\]
Because \(\pi_\psi(y,y^n)\) takes values in \(S_{n+1}\),
\[
\kappa_\alpha(y^n,\psi)
=
\{y\in Y:\pi_\psi(y,y^n)\ge\beta(\alpha)\}.
\]
By the argument used in Lemma 9 in the Main Text, the map
\[
(y,y^n,\psi)\mapsto \pi_\psi(y,y^n)
\]
is upper semicontinuous on \(Y\times X_0\). Therefore the graph of
\(\kappa_\alpha\) is
\[
\operatorname{Gr}(\kappa_\alpha)
=
\{((y^n,\psi),y)\in X_0\times Y:
\pi_\psi(y,y^n)\ge\beta(\alpha)\},
\]
which is closed in \(X_0\times Y\). For each \((y^n,\psi)\), the value
\(\kappa_\alpha(y^n,\psi)\) is a closed subset of compact \(Y\), hence compact.
Since \(Y\) is compact Hausdorff, the closed-graph theorem for compact-valued
correspondences implies that \(\kappa_\alpha\) is upper hemicontinuous. Thus
\(\kappa_\alpha\) is a morphism in
\(\mathbf{UHCont}^{\mathrm{cH}}_{\mathrm c}\).
Next, by Proposition 4 in the Main Text, under consonance,
\[
\IHDR_\alpha\bigl(\CRED(y^n,\psi)\bigr)
=
\kappa_\alpha(y^n,\psi)
\]
for every \((y^n,\psi)\in X_0\). Hence
$\mathcal H_\alpha=\kappa_\alpha$.
Since \(\kappa_\alpha\) is compact-valued and upper hemicontinuous,
\(\mathcal H_\alpha\) is also compact-valued and upper hemicontinuous. Therefore
\(\mathcal H_\alpha\) is a morphism in
\(\mathbf{UHCont}^{\mathrm{cH}}_{\mathrm c}\), and the direct and indirect
constructions agree as morphisms \(X_0\rightrightarrows Y\).

Finally, suppose that \(\mathscr M_{\mathrm{CP},0}\) is endowed with a compact
Hausdorff topology for which \(c_{\mathrm{CP}}\) is continuous and
\(\widehat{\IHDR}_\alpha\) is compact-valued and upper hemicontinuous. Then
\(c_{\mathrm{CP}}\), viewed as a singleton-valued correspondence, is a morphism,
and \(\widehat{\IHDR}_\alpha\) is a morphism by assumption. For every
\((y^n,\psi)\in X_0\),
\[
(\widehat{\IHDR}_\alpha\circ c_{\mathrm{CP}})(y^n,\psi)
=
\widehat{\IHDR}_\alpha(\CRED(y^n,\psi))
=
\IHDR_\alpha(\CRED(y^n,\psi))
=
\kappa_\alpha(y^n,\psi).
\]
Therefore the displayed triangular diagram commutes in
\(\mathbf{UHCont}^{\mathrm{cH}}_{\mathrm c}\).
\end{proof}

Theorem \ref{cont-cp} gives minimal sufficient conditions under which
(i) Full Conformal Prediction is a morphism in the category
\(\mathbf{UHCont}^{\mathrm{cH}}_\text{c}\), which admits a monad, and (ii) the Full Conformal Prediction Diagram commutes in $\mathbf{UHCont}^{\mathrm{cH}}_\text{c}$. We now show that a subcategory of 
$\mathbf{WMeas}_\text{uc}$ admits a monad.

\begin{theorem}[Vietoris monad on $\mathbf{WMeas}_{\mathrm{uc}}$ (with hemicontinuity)]
\label{thm:vietoris-monad-wmeas-uc}
For this theorem only, strengthen the morphism condition in $\mathbf{WMeas}_{\mathrm{uc}}$ by additionally requiring \emph{both upper and lower hemicontinuity} (still uniformly compact-valued and weakly measurable).\footnote{It is easy to see that this class is still closed under composition.} Define the endofunctor $\mathcal{K}$, unit $\eta$, and multiplication $\nu$ by
\[
\begin{aligned}
\mathcal{K}(X)            &\coloneqq  \{  C\subseteq X : C\neq\emptyset, C \text{ compact} \} \text{with the Vietoris topology (Borel $\sigma$-algebra),}\\
\mathcal{K}(\Phi)(C)      &\coloneqq  \{ L\in\mathcal K(Y): L\subset \overline{\Phi[C]} \}, \qquad C\in \mathcal{K}(X),\\
\eta_X(x)                 &\coloneqq  \bigl\{ \{x\}\bigr\}, \\
\nu_X(\mathcal C)         &\coloneqq  \bigl\{\textstyle\bigcup_{C\in\mathcal C} C\bigr\}, \qquad \mathcal C\in \mathcal{K}(\mathcal{K}(X)).
\end{aligned}
\]
Then, $(\mathcal{K},\eta,\nu)$ is a \emph{lax} monad on $\mathbf{WMeas}_{\mathrm{uc}}$ (under this added hemicontinuity standing assumption). Moreover, if $|X|>1$, the unit $\eta_X:X\to\mathcal K(X)$ is not an isomorphism 
(it is not surjective), so the monad is nontrivial.
\end{theorem}

\begin{proof}
\textbf{(1) $\mathcal{K}$ is an endofunctor.}

\emph{(Objects).} If $X$ is compact Polish, then $\mathcal{K}(X)$ with the Vietoris topology is compact Polish; Its Borel $\sigma$-algebra makes $\mathcal{K}(X)$ an object of $\mathbf{WMeas}_{\mathrm{uc}}$.

\emph{(Morphisms).} Let $\Phi:X\rightrightarrows Y$ be a morphism (uniformly compact-valued with common compact bound $K_\Phi\subset Y$, weakly measurable, and both upper and lower hemicontinuous).

Each fiber of $\mathcal{K}(\Phi)$ is the singleton $\{\overline{\Phi[C]}\}$, which lies in $\mathcal{K}(K_\Phi)$. Hence, $\mathcal{K}(\Phi)$ is uniformly compact valued with uniform bound $\mathcal{K}(K_\Phi) \subset \mathcal{K}(Y)$. By a standard hyperspace continuity result 
\citep{AubinFrankowska1990,Beer1993}, 
if $\Phi$ has compact values and is both upper and lower hemicontinuous, 
then the map
\[
H:\mathcal K(X)\longrightarrow \mathcal K(Y), 
\qquad 
H(C)\coloneqq \overline{\Phi[C]},
\]
is continuous for the Vietoris topology. Hence $\mathcal K(\Phi)$ is continuous 
(and therefore weakly measurable) with singleton compact fibers uniformly bounded by $K_\Phi$.

\emph{Functoriality.} $\mathcal{K}(\mathrm{id}_X)=\mathrm{id}_{\mathcal{K}(X)}$. For $\Psi:Y\rightrightarrows Z$ also hemicontinuous with compact values,
\[
\bigl(\mathcal{K}(\Psi)\circ\mathcal{K}(\Phi)\bigr)(C)
= \bigl\{\overline{\Psi[\overline{\Phi[C]}]}\bigr\},
\qquad
\mathcal{K}(\Psi\circ\Phi)(C)
= \bigl\{\overline{(\Psi\circ\Phi)[C]}\bigr\}.
\]
We claim $\overline{\Psi[\overline{\Phi[C]}]}=\overline{(\Psi\circ\Phi)[C]}$. One inclusion is monotonicity:
$\Psi[\Phi[C]]\subset \Psi[\overline{\Phi[C]}]$, hence $\overline{(\Psi\circ\Phi)[C]}\subset \overline{\Psi[\overline{\Phi[C]}]}$.
For the reverse, let $z\in \overline{\Psi[\overline{\Phi[C]}]}$. Then, there are sequences $y_m\in \overline{\Phi[C]}$ and $z_m\in \Psi(y_m)$ with $z_m\to z$. By compactness, pass to subsequences with $y_m\to y\in \overline{\Phi[C]}$. Since $\Psi$ is upper hemicontinuous with compact values, its graph is closed, so $(y,z)\in \mathrm{Gr}(\Psi)$, i.e. $z\in \Psi(y)$. Because $y\in \overline{\Phi[C]}$, there are $u_k\in \Phi[C]$ with $u_k\to y$. By \emph{lower} hemicontinuity of $\Psi$ at $y$, for every neighbourhood $W$ of $z$ there is a neighbourhood $N$ of $y$ with $\Psi(x)\cap W\neq\emptyset$ for all $x\in N$; Hence for large $k$ we can pick $v_k\in \Psi(u_k)\cap W$. Thus $W$ meets $\Psi[\Phi[C]]$ for every neighbourhood $W$ of $z$, proving $z\in \overline{\Psi[\Phi[C]]}=\overline{(\Psi\circ\Phi)[C]}$. Therefore $\mathcal{K}(\Psi\circ\Phi)=\mathcal{K}(\Psi)\circ\mathcal{K}(\Phi)$.

\textbf{(2) The unit $\eta$ and multiplication $\nu$.}
Both $\eta_X$ and $\nu_X$ have singleton fibers, hence uniformly compact-valued. For weak measurability (indeed continuity), compute on subbasic opens
\[
\eta_X^{-1}\langle U\rangle = U,\qquad
\nu_X^{-1}\langle U\rangle
= \bigl\{\mathcal C\in\mathcal{K}(\mathcal{K}(X)):\exists C\in\mathcal C, C\cap U\neq\emptyset\bigr\},
\]
and similarly for miss-type sets $[U]$; these are Vietoris-open. Naturality of $\eta$ and $\nu$ is immediate from $\mathcal{K}(\Phi)[\{x\}]=\{\overline{\Phi[x]}\}$ and $\mathcal{K}(\Phi)\bigl[\{\overline{\cup \mathcal C}\}\bigr]=\{\overline{\cup \mathcal{K}(\Phi)[\mathcal C]}\}$.

\textbf{(3) Monad laws.}
The unit and associativity hold in the lax sense:
\[
\mathrm{id}_{\mathcal K(X)}  \subseteq  \nu_X\circ \mathcal K(\eta_X),
\qquad
\nu_X\circ \eta_{\mathcal K(X)}  =  \mathrm{id}_{\mathcal K(X)},
\qquad
\nu_X\circ \mathcal K(\nu_X)  \subseteq  \nu_X\circ \nu_{\mathcal K(X)}.
\]
Indeed, $\mathcal K(\eta_X)(C)$ consists of all compact families of singletons drawn
from $C$, so $\nu_X\circ \mathcal K(\eta_X)(C)$ contains $C$ (taking the full family of singletons).
For the right unit, $\eta_{\mathcal K(X)}(C)=\{\{C\}\}$ and $\nu_X$ returns $\{C\}$.
For associativity, unions of (compact families of) unions yield no more than the union
of all underlying members, giving the inclusion.

\textbf{(4) Non-triviality.}
If $|X|>1$, then $\eta_X:X\to \mathcal{K}(X)$ is not surjective (non-singleton compact sets are not in its image), hence no natural isomorphism $\mathcal{K}\cong \mathrm{Id}$ exists.

Thus $(\mathcal{K},\eta,\nu)$ is a lax monad on $\mathbf{WMeas}_{\mathrm{uc}}$ under the added hemicontinuity assumption.
\end{proof}

The construction in Theorem \ref{thm:vietoris-monad-wmeas-uc} is a version of classical Vietoris
monad, adapted to our measurable setting: Objects are compact Polish spaces
(with their Borel $\sigma$-algebras), and morphisms are uniformly
compact-valued, weakly measurable correspondences. 
For the hyperspace construction and functoriality of $\mathcal K$, we strengthened
morphisms to be both upper and lower hemicontinuous (a standard hypothesis
ensuring continuity of the induced map on hyperspaces). 
This added assumption is only used to establish the (lax) monad; All results proved
earlier for $\mathbf{WMeas}_{\mathrm{uc}}$---including Theorem \ref{conf-pred-diagr2}---do not rely on the monad and remain valid verbatim.

Regarding Eilenberg-Moore algebras: On the category \(\mathbf{CHaus}\) of compact
Hausdorff spaces, the Vietoris monad \citep{Garner2020_Vietoris} has Eilenberg-Moore category equivalent to
the category \(\mathbf{ContLat}\) of continuous lattices with Scott-continuous
homomorphisms \citep{Day1975,Wyler1981}. 
Since compact Polish spaces are precisely the compact \emph{metrizable} Hausdorff spaces, restricting this equivalence to metrizable carriers
yields the full subcategory of \(\mathbf{ContLat}\) consisting of those
algebras whose Lawson topology is metrizable.

We have a nice byproduct of focusing on the (sub)categories that we considered in this section.

\begin{proposition}[Forgetful Embedding between $\mathbf{UHCont}^{\mathrm{cH}}_\text{c}$ and $\mathbf{WMeas}_\text{uc}$]
\label{prop:uhcont-to-WMeas}
There is a faithful functor
\[
F  :  \mathbf{UHCont}^{\mathrm{cH}}_\text{c}
    \rightarrow 
   \mathbf{WMeas}_\text{uc},
\]
defined as follows.
\begin{enumerate}
  \item[\textup{(Obj)}] For every metrizable compact Hausdorff space $X$, pick any compatible
        (complete separable) metric $d_X$; The resulting topological
        space is compact Polish.
        Set $F(X)\coloneqq(X,d_X)$ equipped with its Borel $\sigma$--algebra.
  \item[\textup{(Mor)}] For a morphism
        $\Phi : X \rightrightarrows Y$ in
        $\mathbf{UHCont}^{\mathrm{cH}}_\text{c}$
        (i.e. compact--valued and upper hemicontinuous),
        let
        \[
          F(\Phi) \coloneqq \Phi
          \qquad
          (\text{same graph and fibers}).
        \]
\end{enumerate}
With these assignments, $F$ preserves identities and composition,
and on every hom-set it is injective. Hence, $F$ is faithful.
\end{proposition}

The functor $F$ in Proposition \ref{prop:uhcont-to-WMeas} is defined on the full subcategory of $\mathbf{UHCont}^{\mathrm{cH}}_\mathrm{c}$ with metrizable objects, since $(X,\mathcal B(X))$ must be compact Polish to lie in $\mathbf{WMeas}_{\mathrm{uc}}$. When working within that metrizable subcategory (as we do here), the statement holds verbatim. The choice of compatible metric on $X$ does not affect $\mathcal B(X)$.

\begin{proof}[Proof of Proposition \ref{prop:uhcont-to-WMeas}]
\textbf{(1) $F(\Phi)$ is a morphism in $\mathbf{WMeas}_\text{uc}$.}
Because $Y$ is compact, $\Phi(x)\subseteq Y$ for all $x$, so
$K_\Phi\coloneqq Y$ is a common compact bound; Thus $F(\Phi)$ is uniformly
compact-valued.  
To show weak measurability, fix an open $O\subseteq Y$ and define
$\varphi_O(y)\coloneqq d_Y \bigl(y,(Y\setminus O)\bigr)$, where $d_Y$ is any
compatible (complete) metric on $Y$. Set
\[
h_O(x)\coloneqq \sup_{y\in \Phi(x)} \varphi_O(y).
\]
Since $\Phi$ is upper hemicontinuous with compact values, Berge’s maximum
theorem implies $h_O$ is upper semicontinuous; Hence
\[
\{x: \Phi(x)\cap O\neq\emptyset\}
=\{x: h_O(x)>0\}
\]
is open (in particular, Borel). This is precisely the weak measurability
condition for $F(\Phi)$.

\textbf{(2) Functoriality.}
$F$ leaves the underlying correspondence unchanged, so
$F(\mathrm{id}_X)=\mathrm{id}_{F(X)}$
and
$F(\Psi\circ\Phi)=F(\Psi)\circ F(\Phi)$,
for all composable $\Phi,\Psi$.

\textbf{(3) Faithfulness.}
On each hom-set, $F$ acts as the identity map
\(
\Phi\mapsto\Phi,
\)
hence it is injective.

\textbf{(4) Domain of $F$.}
If every object of
$\mathbf{UHCont}^{\mathrm{cH}}_\text{c}$
is assumed to be \emph{metrizable}
(equivalently, second-countable), then the object assignment is
defined on the whole category.  Otherwise, $F$ can still be applied to
the metrizable subcategory, for which the conclusion holds verbatim.
\end{proof}

Alas, there is no faithful functor
\(E:\mathbf{WMeas}_{\mathrm{uc}}\to\mathbf{UHCont}\)
that (i) is the identity on objects and (ii) fixes all Dirac morphisms
\(\delta_y:\{*\}\to Y\), \(\delta_y(*)=\{y\}\). To see this, take \(X=Y=[0,1]\) (usual topology and Borel \(\sigma\)-algebra) and define
\begin{align*}
    F(x)&=\{0\} \text{ if }x\in\mathbb{Q},\quad F(x)=\{1\} \text{ if }x\notin\mathbb{Q},\\
G(x)&=\{1\} \text{ if }x\in\mathbb{Q},\quad G(x)=\{0\} \text{ if }x\notin\mathbb{Q}.
\end{align*}

Both \(F\) and \(G\) are morphisms in \(\mathbf{WMeas}_{\mathrm{uc}}\) (singleton-valued and weakly measurable), with \(F\neq G\).
For each \(x\in X\), let \(\delta_x:\{*\}\to X\) be the Dirac morphism. Then
\(F\circ\delta_x=\delta_{F(x)}\) and \(G\circ\delta_x=\delta_{G(x)}\).
By the hypotheses and functoriality,
\[
E(F)\circ\delta_x  =  E(F\circ\delta_x)  =  E(\delta_{F(x)})  =  \delta_{F(x)},
\]
so \(E(F)(x)=\{F(x)\}\) for all \(x\). Thus, \(E(F)\) coincides pointwise with \(F\),
but \(F\) is not upper hemicontinuous (its graph is not closed since \(\mathbb{Q}\) and \(\mathbb{R}\setminus\mathbb{Q}\) are both dense). This contradicts that \(E(F)\) must be a morphism in \(\mathbf{UHCont}\), so no such faithful functor \(E\) exists.

\section{Background in Functional Analysis}\label{funct-anal}
This section presents background functional-analytic notions that are used throughout the Main Text and the supplementary material. The concepts we present in this section are studied in detail in \citet[Chapters 17 and 18]{aliprantis}.

A correspondence $\Phi: X \rightrightarrows Y$ from a set $X$ to a set $Y$ assigns to each $x\in X$ a subset $\Phi(x) \in 2^Y$. The image of a set $A \subseteq X$ under $\Phi$ is the set $\Phi(A)\coloneqq \cup_{x\in A} \Phi(x)$. Given two correspondences $\Phi: X \rightrightarrows Y$ and $\Psi: Y \rightrightarrows Z$, their composition is defined naturally as
$$\Psi \circ \Phi: X \rightrightarrows Z, \quad (\Psi \circ \Phi)(x) \coloneqq \bigcup_{y \in \Phi(x)} \Psi(y).$$
The graph of $\Phi$ is given by $\text{Gr}\Phi \coloneqq \{(x,y)\in X \times Y : y \in \Phi(x)\}$.

We now turn our attention to continuity and measurability notions for correspondences.



\begin{definition}[Upper hemicontinuity]\label{uhc-definition}
Let $X$ and $Y$ be topological spaces. A correspondence $\Phi : X \rightrightarrows Y$ is
\emph{upper hemicontinuous at} $x_0 \in X$ if for every open set $V \subseteq Y$ such that $\Phi(x_0) \subseteq V$, there exists an open neighborhood $U$ of $x_0$ such that for all $x \in U$, $\Phi(x) \subseteq V$. 

We say that $\Phi$ is \emph{upper hemicontinuous (u.h.c.)} if it is upper hemicontinuous at every point $x \in X$.
\end{definition}

As the name suggests, an alternative notion of continuity for correspondences, called lower hemicontinuity, can be defined \citep[Definition 17.2]{aliprantis}. In this work, we focus primarily on u.h.c. for its historical importance---especially in fields such as economics and game theory \citep{Borgers1991UHC}---and because, as pointed out in the Main Text, upper hemicontinuity formalizes a
stability property of the CP correspondence: Small perturbations in the inputs
(data and non-conformity measure) cannot produce pathological outward ``jumps'' of the output region.

The following is the definition of how a correspondence can be considered measurable in the weakest possible sense.

\begin{definition}[Weakly Measurable Correspondence]\label{w-mble-definition}
Let $(X, \Sigma_X)$ be a measurable space and $(Y, \tau_Y)$ a topological space. A correspondence $\Phi : X \rightrightarrows Y$ is said to be \emph{weakly measurable} if for every open set $O \subseteq Y$, the set
\[
\Phi^u(O) \coloneqq \{ x \in X : \Phi(x) \cap O \neq \emptyset \} \in \Sigma_X.
\]
\end{definition}

Notice that weak measurability has nothing to do with weak topologies, and that other (stronger) measurability notions can be found in \citet[Definition 18.1]{aliprantis}.

\begin{definition}[Uniformly Compact-Valued Correspondences]\label{unif-cpct-val-definition}
\label{def:ucv}
Let $(X,\Sigma_X)$ and $(Y,\Sigma_Y)$ be measurable Polish spaces.
A correspondence $\Phi\colon X\rightrightarrows Y$ is called
\emph{uniformly compact-valued} if
\begin{enumerate}
  \item[(i)] (compact fibers) $\Phi(x)\subseteq Y$ is nonempty compact, for every $x\in X$;
  \item[(ii)] (uniform compact bound) there exists a \emph{one and the same}
        compact set $K_\Phi\subseteq Y$ such that $\Phi(x)\subseteq K_\Phi$,
        for all $x\in X$.
\end{enumerate}
\end{definition}

Finally, we also recall the definition of the Vietoris topology.

\begin{definition}[Vietoris Topology]\label{vietoris}
Let \(X\) be a compact Hausdorff space. The Vietoris topology on the hyperspace
\(\mathcal{K}(X)\) of nonempty compact subsets of \(X\) is the smallest topology
making, for every open \(U\subseteq X\),
\[
\langle U\rangle \coloneqq \{K\in\mathcal{K}(X): K\cap U\neq\emptyset\}
\quad\text{and}\quad
[U]\coloneqq \{K\in\mathcal{K}(X): K\subseteq U\}
\]
open. Equivalently, it is generated by finite intersections
\([U_0]\cap \langle U_1\rangle \cap \cdots \cap \langle U_m\rangle\)
and, in particular, by the basic “hit-and-cover” sets
\[
\bigl\langle U_1,\dots,U_n\bigr\rangle
\coloneqq \{K\in\mathcal{K}(X): K\subseteq \textstyle\bigcup_{i=1}^n U_i
 \text{and} K\cap U_i\neq\emptyset \text{, } \forall i\},
\]
since \(\langle U_1,\dots,U_n\rangle=[U_1\cup\cdots\cup U_n]\cap\bigcap_{i=1}^n\langle U_i\rangle\).
\end{definition}

\section{Background in Category Theory}\label{cat-th-intro}
This section presents background category-theoretic notions that are used throughout the Main Text and the supplementary material.

Recall from Section 2.2 in the Main Text the definitions of a category and a functor between categories.

\begin{definition}[Category]%
\label{def:category}
A \emph{category} $\mathbf C$ consists of
\begin{itemize}
  \item a collection $\mathbf C_0$ of \emph{objects} (written $X,Y,Z,\dots$);
  \item a collection $\mathbf C_1$ of \emph{morphisms} (written $f,g,h,\dots$);
\end{itemize}
together with
\begin{enumerate}
  \item for every morphism $f$ there are two distinguished objects called its \emph{source} and \emph{target}, written $s(f)$ and $t(f)$ (we abbreviate $f\colon X\rightarrow Y$ when $s(f)=X$, $t(f)=Y$);
  \item for every object $X$ an \emph{identity} morphism $\mathrm{id}_X\colon X\rightarrow X$;
  \item for every composable pair $X \xrightarrow{f} Y \xrightarrow{g} Z$ (that is, for which $t(f)=s(g)$) a \emph{composite} morphism $g\circ f\colon X\rightarrow Z$;
\end{enumerate}
satisfying the \emph{unitality} laws $f\circ\mathrm{id}_X = f = \mathrm{id}_Y\circ f$ and the \emph{associativity} law
$(h\circ g)\circ f = h\circ(g\circ f)$, whenever the expressions are well defined.
\end{definition}


\begin{definition}[Functor]%
\label{def:functor}
Let $\mathbf C,\mathbf D$ be categories.  
A \emph{functor} $F\colon\mathbf C\rightarrow\mathbf D$ assigns
\begin{itemize}
  \item to every object $X$ of $\mathbf C$ an object $FX$ of $\mathbf D$;
  \item to every morphism $f\colon X\rightarrow Y$ in $\mathbf C$ a morphism $Ff\colon FX\rightarrow FY$ in $\mathbf D$,
\end{itemize}
such that
\[
F(\mathrm{id}_X)=\mathrm{id}_{FX}, \qquad 
F(g\circ f)=Fg\circ Ff
\]
for all composable $f,g$ in $\mathbf C$. 

A functor $F\colon\mathbf C\rightarrow\mathbf D$ is called {\em faithful} if, for every objects $X,X^\prime$ of $\mathbf C$ and every morphism $f,f^\prime: X \rightarrow X^\prime$ of $\mathbf C$, we have that $Ff=Ff^\prime$ in $\mathbf{D}$ implies $f=f^\prime$ in $\mathbf C$.
\end{definition}

A special type of functor, called {\em faithful functor}, is presented in \citet[Section 1.5.3]{perrone2024starting}.
Following up on the question that we asked ourselves before, we can query whether there is a convenient choice of relating two functors.

\begin{definition}[Natural transformation]%
\label{def:nattrans}
Let $\mathbf C,\mathbf D$ be categories and let
\[F,G : \mathbf C \rightarrow \mathbf D\]
be functors.  
A {\em natural transformation} $\alpha : F \Rightarrow G$ assigns to every object
$X$ of $\mathbf C$ a morphism
\[
  \alpha_X : FX \rightarrow GX
\]
such that for every morphism $f : X \rightarrow Y$ in $\mathbf C$
\[
  Gf \circ \alpha_X  =  \alpha_Y \circ Ff .
\]

Natural transformation $\alpha$ is a {\em natural isomorphism} if for each object $X$ of $\mathbb{C}$, the component $\alpha_X:FX \rightarrow GX$ is an isomorphism. In that case, functors $F$ and $G$ are said to be {\em (naturally) isomorphic}.
\end{definition}

Then, we turn our attention to a special structure on a category, called a monad, that can be thought of as a consistent way of extending spaces to include generalized elements and generalized functions of a specific type. Monads are of special interest in Computer Science, where they provide a modular language for composing effectful computations, and they are relevant here because the Vietoris monad formalizes the passage from points to compact sets.

\begin{definition}[Monad on a category]%
\label{def:monad}
For a category $\mathbf C$, a {\em monad} is a triple
\[
  (T,\eta,\mu)
\]
consisting of
\begin{itemize}
  \item a functor $T : \mathbf C \rightarrow \mathbf C$;
  \item a natural transformation (the \emph{unit})
        \[
           \eta : \operatorname{id}_{\mathbf C}  \Rightarrow  T;
        \]
  \item a natural transformation (the \emph{multiplication})
        \[
           \mu : T  \circ T  \Rightarrow  T;
        \]
\end{itemize}
subject to the {\em associativity} and {\em unit} laws
\[
\mu \circ T\mu  =  \mu \circ \mu T,
\qquad
\mu \circ T\eta  =  \mu \circ \eta T  =  \operatorname{id}_T.
\]
\end{definition}

Finally, we are interested in finding a category whose objects and morphisms can be multiplied together in a way that is associative and unital, similarly to what we have for monoids (but only up to isomorphisms). Such categories, called monoidal categories, have wide applications, ranging from information theory to probability.

\begin{definition}[Monoidal category]%
\label{def:monoidal}
A {\em monoidal category} is a tuple
\[
  (\mathbf C,\otimes,I,\alpha,\lambda,\rho)
\]
where
\begin{itemize}
  \item $\mathbf C$ is a category;
  \item $\otimes : \mathbf C \times \mathbf C \rightarrow \mathbf C$ is a product, the \emph{tensor product} \citep[Definition 1.3.49]{perrone2024starting};\footnote{We do not introduce the tensor product formally in this background section because its definition follows the natural intuition of how a (cartesian) product behaves.}
  \item $I$ is an object of $\mathbf C$ (the \emph{unit});
  \item $\alpha_{A,B,C} : A\otimes(B\otimes C) \rightarrow (A\otimes B)\otimes C$
        (\emph{associator}),
        $\lambda_A :I\otimes A \rightarrow A$,
        $\rho_A :A\otimes I \rightarrow A$
        are natural isomorphisms,
\end{itemize}
satisfying Mac Lane’s {\em pentagon} and {\em triangle} equalities, for all objects $A,B,C,D$ of $\mathbf C$
\[
\alpha_{A,B,C\otimes D} \circ \alpha_{A\otimes B,C,D}
   = 
(\alpha_{A,B,C}\otimes\operatorname{id}_D) \circ 
\alpha_{A,B\otimes C,D} \circ 
(\operatorname{id}_A\otimes\alpha_{B,C,D}),
\]
\[
(\operatorname{id}_A\otimes\lambda_B) \circ \alpha_{A,I,B}
  = 
\rho_A\otimes\operatorname{id}_B .
\]
\end{definition}

\end{document}